\documentclass[nohyperref]{article}

\usepackage{microtype}
\usepackage{graphicx}
\usepackage{subfigure}
\usepackage{booktabs} 

\usepackage[hypertexnames=false]{hyperref}

\usepackage[accepted]{icml2022}

\usepackage{amsmath}
\usepackage{amssymb}
\usepackage{mathtools}
\usepackage{amsthm}

\usepackage[capitalize,noabbrev]{cleveref}

\theoremstyle{plain}
\newtheorem{theorem}{Theorem}[section]

\newtheorem{lemma}[theorem]{Lemma}
\newtheorem{corollary}[theorem]{Corollary}
\theoremstyle{definition}
\newtheorem{definition}[theorem]{Definition}

\theoremstyle{remark}

\usepackage[textsize=tiny]{todonotes}

\global\long\def\RR{\mathbb{R}}
\newcommand{\PolyS}{{\sc PolySketch} }
\newcommand{\SRHT}{{\sc SRHT} }
\newcommand{\algoname}[1]{\textnormal{\textsc{#1}}}
\usepackage{bbm}
\newcommand{\EE}{\mathbb{E}}
\newenvironment{proofof}[1]{\noindent{\bf Proof of #1:}}{\hfill$\qed$\par}
\usepackage{tikz}
\usetikzlibrary{decorations.pathmorphing}
\usetikzlibrary{decorations.markings}
\usetikzlibrary{shapes.gates.logic.US,trees,positioning,arrows}

\tikzset{
	arn/.style = {circle, white, draw=black, fill=gray!30, inner sep = 10.5},
	arn_t/.style = {circle, white, draw=black, very thick, fill=gray!30, inner sep = 11.0},
	arn_l/.style = {circle, white, draw=black, very thick, fill=black, inner sep = 2},
	photon/.style={draw=black, very thick, dashed},
	electron/.style={draw=black, very thick},
	tr/.style={buffer gate US,thick,draw,fill=gray!60,rotate=90,	anchor=east,minimum width=2.25cm},
	br/.style={buffer gate US,thick,draw,fill=gray!60,rotate=90,	anchor=east,minimum width=4.5cm},
	brr/.style={buffer gate US,draw,fill=gray!60,rotate=90,	anchor=east,minimum width=4.5cm, opacity = 0.6},
	trr/.style={buffer gate US,thick,draw,fill=gray!60,rotate=90,	anchor=east,minimum width=2.25cm, opacity = 0.6},
	trrr/.style={buffer gate US,draw,fill=white!60,rotate=90,	anchor=east,minimum width=2.25cm, opacity = 0.5}
}
\newcommand{\jj}{\mathbf{j}}
\newcommand{\ii}{\mathbf{i}}
\usepackage{makecell}

\usepackage{enumitem}
\setlist[enumerate]{itemsep=0.5pt, wide=\parindent}
\setlist[itemize]{topsep=0pt,wide=\parindent}

\icmltitlerunning{Leverage Score Sampling for Tensor Product Matrices in Input Sparsity Time}

\newcommand{\alglinelabel}{%
	\addtocounter{ALC@line}{-1}
	\refstepcounter{ALC@line}
	\label
}

\usepackage{cleveref}

\makeatletter
\providecommand\theHALG@line{\thealgorithm.\arabic{ALG@line}}
\newcommand{\ALG@lineautorefname}{Line}
\def\BState{\State\hskip-\ALG@thistlm}
\makeatother

\begin{document}
	
	\twocolumn[
	\icmltitle{Leverage Score Sampling for Tensor Product Matrices in Input Sparsity Time}
	
	
	
	\icmlsetsymbol{equal}{*}
	
	\begin{icmlauthorlist}
		\icmlauthor{David P. Woodruff}{equal,yyy}
		\icmlauthor{Amir Zandieh}{equal,comp}
	\end{icmlauthorlist}
	
	\icmlaffiliation{yyy}{Max-Planck-Institut für Informatik}
	\icmlaffiliation{comp}{Carnegie Mellon University}
	
	\icmlcorrespondingauthor{David Woodruff}{dwoodruf@cs.cmu.edu}
	\icmlcorrespondingauthor{Amir Zandieh}{azandieh@mpi-inf.mpg.de}
	
	\icmlkeywords{Machine Learning, ICML}
	
	\vskip 0.3in
	]
	
	
	
	\printAffiliationsAndNotice{\icmlEqualContribution} 
	
	\begin{abstract}
		We propose an input sparsity time sampling algorithm that can spectrally approximate the Gram matrix corresponding to the $q$-fold column-wise tensor product of $q$ matrices using a nearly optimal number of samples, improving upon all previously known methods by poly$(q)$ factors. 
		Furthermore, for the important special case of the $q$-fold self-tensoring of a dataset, which is the feature matrix of the degree-$q$ polynomial kernel, the leading term of our method’s runtime is proportional to the size of the input dataset and has no dependence on $q$.
		Previous techniques either incur poly$(q)$ slowdowns in their runtime or remove the dependence on $q$ at the expense of having sub-optimal target dimension, and depend quadratically on the number of data-points in their runtime. 
		Our sampling technique relies on a collection of $q$ partially correlated random projections which can be simultaneously applied to a dataset $X$ in total time that only depends on the size of $X$, and at the same time their $q$-fold Kronecker product acts as a near-isometry for any fixed vector in the column span of $X^{\otimes q}$. 
		We also show that our sampling methods generalize to other classes of kernels beyond polynomial, such as Gaussian and Neural Tangent kernels.
	\end{abstract}
	
	\section{Introduction}
	\label{sec:intro}
	
	In many learning problems such as regression or PCA, one is given a \emph{feature} (or design) matrix $\Phi \in \RR^{m \times n}$ and needs to compute the inverse or singular value decomposition (SVD) of the Gram matrix $\Phi^\top \Phi$. However, the feature matrices $\Phi$, particularly the features that correspond to kernel functions, often have a massive (sometimes infinite) number of rows, which makes the storage and computations involving $\Phi^\top \Phi$ prohibitively expensive. This has motivated a long line of work on approximating the Gram matrix $\Phi^\top \Phi$ by a low-rank matrix~\cite{williams2001using,rahimi2007random,avron2014subspace,alaoui2014fast,cohen2015uniform,musco2017recursive,avron2017random}.
	
	In this work, we focus on feature matrices whose columns are tensor products of a large number of arbitrary vectors, i.e., $\Phi=X^{(1)} \otimes X^{(2)} \otimes \ldots X^{(q)}$ for datasets $X^{(1)}, \ldots X^{(q)} \in \RR^{d \times n}$ (for tensor product notations see Definitions~\ref{def-tensor-prod} and \ref{def-tensor-prod-matrices}).
	Note that the tensor product matrix $\Phi$ defined this way has $d^q$ rows and $n$ columns. 
	This type of tensor product feature matrix $\Phi$ is of great importance in learning applications, particularly because the special case of $X^{(1)} = \cdots = X^{(q)}$ corresponds to the feature matrix of the degree-$q$ polynomial kernel, i.e., the Gram matrix $\Phi^\top \Phi$ is the degree-$q$ polynomial kernel matrix.
	To tackle scalability challenges, much work has focused on \emph{compressing} the large number of rows of such tensor product feature matrices through linear sketching or sampling techniques~\cite{pham2013fast, avron2014subspace, ahle2019oblivious,meister2019tight, zandieh2021scaling, song2021fast}.
	
	The aim of our work is to devise efficient sampling methods for reducing the dimensionality (number of rows) of tensor product matrices while preserving the spectral structure of the Gram matrix. 
	Formally, for any given $\epsilon, \lambda> 0$ and any $X^{(1)}, \ldots X^{(d)} \in \RR^{d \times n}$, if the feature matrix is defined as $\Phi := X^{(1)} \otimes \ldots X^{(q)}$, we want to find a sampling matrix $\Pi \in \RR^{s \times d^d}$, such that the sub-sampled Gram matrix $\Phi^\top \Pi^\top \Pi \Phi$ is an \emph{$(\epsilon , \lambda)$-spectral approximation} to $\Phi^\top \Phi$, i.e.,
	\begin{equation}\label{spectral-bound}
		\frac{\Phi^\top \Phi+\lambda I}{1+\epsilon} \preceq \Phi^\top \Pi^\top \Pi \Phi+\lambda I \preceq \frac{\Phi^\top \Phi+\lambda I}{1-\epsilon}.
	\end{equation}
	Sampling a small number of rows of any matrix $\Phi$ according to its \emph{leverage scores} is known to yield a spectral approximation to $\Phi^\top \Phi$~\cite{li2013iterative}. 
	Our goal is to generate a sampling matrix $\Pi$ according to the ridge leverage scores of $\Phi$ in input sparsity time, i.e., $O\left( \sum_{j=1}^q {\rm nnz}\left( X^{(j)} \right) \right)$. 
	
	\subsection{Our Main Results}
	\begin{itemize}
		\item It is well-known that for any linear sketch or sampling matrix $\Pi$ to satisfy \eqref{spectral-bound}, its number $s$ of rows needs to be proportional to the \emph{statistical dimension} $s_\lambda :=\sum_{i=1}^n\frac{\lambda_i}{\lambda_i+\lambda}$, where the $\lambda_i$ are the eigenvalues of $\Phi^\top \Phi$, \cite{avron2019universal}. \citet{woodruff2020near} recently showed that it is possible to generate a sampling matrix with $s=O\left( \frac{s_\lambda}{\epsilon^2} \log n \right)$ rows that satisfies \eqref{spectral-bound} in time $\widetilde{O}\left( {\rm poly}(q,\epsilon^{-1}) \cdot s_\lambda^2 n + q^{1.5}  \sum_{j=1}^q {\rm nnz}\left( X^{(j)} \right) \right)$. 
		The significance of this result was showing the possibility of decoupling $\epsilon^{-1}$ factors from the leading term in its runtime, i.e., $\sum_{j=1}^q {\rm nnz}\left( X^{(j)} \right)$. The following fundamental question about whether the factor $q^{1.5}$ in the runtime of \cite{woodruff2020near} is necessary has not been answered yet.
		
		\emph{Can we produce a sampling matrix that satisfies \eqref{spectral-bound} in time $\widetilde{O}\left( {\rm poly}(q,\epsilon^{-1}) \cdot s_\lambda^2 n + \sum_{j=1}^q {\rm nnz}\left( X^{(j)} \right) \right)$?}
		
		We answer the above question positively in Theorem~\ref{main-thrm-tensor-prod}, which shows that input sparsity runtime and small $s = O(\epsilon^{-2} s_\lambda \log n)$ number of samples are achievable.
		One advantage of our method is that after computing the sampling matrix $\Pi$ using Theorem~\ref{main-thrm-tensor-prod}, we can simply store $\Pi \Phi$ using $O(n s) = O(\epsilon^{-2} s_\lambda n \log n )$ words of memory, while the memory needed to store the exact Gram matrix $\Phi^\top \Phi$ is $\Theta(n^2)$.
		Thus, our method reduces the memory from quadratic in the dataset size $n$, to linear.
		
		Additionally, for solving many downstream learning tasks such as ridge regression, low-rank approximation, or PCA with the feature matrix $\Phi$, one typically needs to either compute the inverse or the SVD of the Gram matrix $\Phi^\top \Phi$. If $\Phi^\top \Phi$ is pre-computed exactly and is stored in memory, then computing its SVD requires $\Theta(n^3)$ additional runtime. 
		So the total time to compute $\Phi^\top \Phi$ exactly and then find its SVD, for tensor product feature matrices $\Phi$, is $\Theta\left(n \cdot \sum_{j=1}^q {\rm nnz}\left( X^{(j)} \right) + n^3 \right)$.
		In contrast, given the sub-sampled feature matrix $\Pi \Phi$, we can (spectrally) approximate the SVD of $\Phi^\top \Phi$ by the SVD of $(\Pi \Phi)^\top (\Pi \Phi)$, using only $s^2 n = O(\epsilon^{-4} s_\lambda^2 n \log^2 n)$ operations. Thus, using our Theorem~\ref{main-thrm-tensor-prod}, the SVD of $(\Pi \Phi)^\top (\Pi \Phi)$ can be computed in total time $\widetilde{O}\left( {\rm poly}(q,\epsilon^{-1}) \cdot s_\lambda^2 n + \sum_{j=1}^q {\rm nnz}\left( X^{(j)} \right) \right)$. Hence, our method improves the runtime of solving downstream applications, such as ridge regression or PCA from cubic in $n$ to linear.

		\item For the important case when the input datasets are identical $X^{(1)} = X^{(2)} = \cdots X^{(q)} = X$ and the feature matrix $\Phi := X^{\otimes q}$ corresponds to the degree-$q$ polynomial kernel, invoking our Theorem~\ref{main-thrm-tensor-prod} results in a runtime of $\widetilde{O}\left( {\rm poly}(q,\epsilon^{-1}) \cdot s_\lambda^2 n + q \cdot {\rm nnz}( X ) \right)$, which is a factor $q$ larger than the desired input sparsity time.
		On the other hand, \citet{song2021fast} has recently proposed a linear sketch with $\widetilde{O}(n/\epsilon^2)$ rows which satisfies \eqref{spectral-bound} for $\Phi = X^{\otimes q}$ and can be applied in time $\widetilde{O}\left( q^2 \epsilon^{-2} n^2 + nd \right)$, which can be considered to be $\widetilde{O}\left( q^2 \epsilon^{-2} n^2 + {\rm nnz}( X ) \right)$ for dense $X$, i.e., ${\rm nnz}( X ) = \widetilde{\Omega}(nd)$.
		That is, \citet{song2021fast} showed that decoupling the factor of $q$ from ${\rm nnz}( X )$ is possible at the expense of having sub-optimal target dimension $s \approx n/\epsilon^2$ and losing quadratically in $n$ in the runtime. However, it is unclear whether these losses are necessary. Specifically we consider the following fundamental question:
		
		\emph{Can we produce a sampling matrix with $s = O\left( \frac{s_\lambda}{\epsilon^2} \log n \right)$ rows that satisfies \eqref{spectral-bound} for the degree-$q$ polynomial kernel in time $\widetilde{O}\left( {\rm poly}(q,\epsilon^{-1}) \cdot s_\lambda^2 n + {\rm nnz}( X) \right)$?}
		
		We answer the above question positively in Theorem~\ref{main-thrm-selftensor-prod}. 
		Specifically, our Theorem~\ref{main-thrm-selftensor-prod} applies to any matrix $\Phi = X^{\otimes q}$ in time $\widetilde{O}\left( \text{poly}(q,1/\epsilon) \left( s_\lambda^2 + \sqrt{\| K \| / \lambda} \right)\cdot n + dn \right)$, where $K = \Phi^\top \Phi$ is the kernel matrix corresponding to the degree-$q$ polynomial kernel. For large $d$, this runtime is dominated by $\widetilde{O}(dn)$. Thus, for dense datasets with ${\rm nnz}( X ) = \widetilde{\Omega}(nd)$, this runtime has the same asymptotic order as the input sparsity $\text{nnz}(X)$, and is thus optimal up to $\log$ factors.

		\item We generalize our sampling methods to other classes of kernels beyond polynomial, such as the Gaussian and the Neural Tangent Kernels~\cite{jacot2018neural} in Section~\ref{sec:generalization}. For example in Corollary~\ref{main-theorem-Gaussian}, we prove that our sampling method spectrally approximates the Gaussian kernel for dense datasets with squared radius $r$ in time $\widetilde{O}\left( \frac{r^{8}}{\epsilon^4} s_\lambda^2 n + r^3 \sqrt{ \frac{n}{\lambda} } n + nd \right)$. For comparison, the runtime of \cite{song2021fast} is $\widetilde{O}\left( \frac{r^3}{\epsilon^2} \cdot n^2 + nd \right)$, which means that for any $\lambda =\omega (1/{n})$, any $\epsilon = \widetilde{\Omega}(1)$, and any $r = o\left(n^{0.2}\right)$, the result of our Corollary~\ref{main-theorem-Gaussian} is strictly faster. 
		
		\item In addition to our theoretical guarantees, we provide regression and classification experiments in Section~\ref{sec:experiments}, which show our method performs well in practice even for moderately-sized datasets.
		In particular, our empirical results show that our method achieves better testing errors compared to prior results for both Gaussian and Neural Tangent kernels.
		
	\end{itemize}
	
	\subsection{Our Techniques}
	\begin{itemize}
		\item Our algorithm samples $s$ i.i.d. rows of the feature matrix $\Phi = \bigotimes_{j=1}^q X^{(j)}$ according to its ridge leverage scores. We devise a highly optimized version of the recursive sampling framework of \cite{woodruff2020near}, which previously had a runtime of $\widetilde{O}\left( q^{1.5} \sum_{j=1}^q {\rm nnz}\left( X^{(j)} \right) \right)$. 
		By closely examining \cite{woodruff2020near} we isolate the main computational bottleneck of their algorithm and formulate it as a data-structure (DS) problem in Section~\ref{sec:tensornormDS}. 
		In particular, our algorithm crucially relies on an efficient DS that can be constructed in input sparsity time, i.e., $\sum_{j=1}^q {\rm nnz}\left( X^{(j)} \right)$, and enables estimation of $\left\| \left(\bigotimes_{j=1}^q X^{(j)} \right) V \right\|_F^2$ for arbitrary queries $V \in \RR^{n \times r}$ in time ${\rm poly}(q) \cdot {\rm nnz}(V)$. 
		We solve this DS problem in Section~\ref{sec:tensornormDS} and then use it in our importance sampling method for tensor product matrices in Section~\ref{sec:recursive-sampling-tensor-product-features} and Appendix~\ref{appndx:proof-rowsampler-distinct-datasets}.
		
		\item 
		To run our sampling algorithm on the feature matrix $X^{\otimes q}$ of the polynomial kernel in input sparsity time, we crucially need a DS that can be constructed in ${\rm nnz}(X)$ time and can quickly answer queries of the form $\left\| X^{\otimes q} \cdot V \right\|_F^2$.
		Our main technical tool for solving this problem is a collection of sketches $S^{(1)}, S^{(2)}, \ldots S^{(q)}$ which are \emph{correlated} to the extent that they can be simultaneously applied to $X$ in a total of $\widetilde{O}({\rm nnz}(X))$ time, and at the same time are \emph{independent enough} to ensure that $\left\| \left( \bigotimes_{j=1}^q S^{(j)}X \right) V \right\|_F^2 \approx \left\|X^{\otimes q} V \right\|_F^2$. 
		We show in Section~\ref{sec:srht-analysis} that a set of Subsampled Randomized Hadamard Transform (\SRHT) sketches with shared random signs can be applied to any dense dataset $X$ in total time $\widetilde{O}({\rm nnz}(X))$, and also provide an unbiased estimator with small variance for $\left\|X^{\otimes q} V \right\|_F^2$. It is not clear at this point if variants of sparse sketches (e.g., CountSketch) with these properties also exist or not.
		
	\end{itemize}
	
	\subsection{Related Work}
	
	A popular line of work on kernel approximation is based on the Random Fourier Features method~\cite{rahimi2007random}, which works well for shift-invariant kernels and with some modifications can embed the Gaussian
	kernel in constant dimension using a near optimal number of features~\cite{avron2017random}. However, all variants of this method need at least $\Omega( s_\lambda \cdot {\rm nnz}(X))$ runtime which is a factor $s_\lambda$ higher than our desired time.
	
	Another popular kernel approximation approach is the Nystr\"om method~\cite{williams2001using}. While the recursive Nystr\"om sampling of \citet{musco2017recursive} can embed kernel matrices using a near optimal number of landmarks, this method also needs at least $\Omega( s_\lambda \cdot {\rm nnz}(X))$ runtime, which is a factor $s_\lambda$ higher than our desired time.

	For the polynomial kernel, sketching methods have been developed extensively~\cite{avron2014subspace, pham2013fast, woodruff2020near, song2021fast}. For example, \citet{ahle2019oblivious} proposed a subspace embedding for high-degree polynomial kernels as well as the Gaussian kernel. However, their required runtime for the degree $q$ polynomial kernel is at least $\Omega( q \cdot {\rm nnz}(X) )$, which has an undesirable factor $q$.
	Recently, \citet{song2021fast} showed that this sketching method can be accelerated for dense datasets by applying an \SRHT on the input dataset. However, their resulting runtime is $\widetilde{O}(q^2 n^2 + nd)$ which has an undesirable quadratic dependence on $n$.

	\section{Preliminaries}\label{sec:prelim}
	
	Throughout the paper, we use symbols $e_1, e_2, \ldots e_d$ to denote the standard basis vectors in $\RR^d$. For any positive integer $n$, we define the set $[n] = \{ 1,2 ,\ldots n \}$. For a matrix $A$ we use $\| A \|$ to denote its operator norm. We also use $A_{i, \star}$ and $A_{\star, i}$ to denote the $i^{th}$ row and $i^{th}$ column of $A$, respectively.
	We use the notation $\widetilde{O}(f)$ to denote $O( f \cdot {\rm poly}\log f )$, for any $f$.
	For any matrix $\Phi \in \RR^{m \times n}$ and regularizer $\lambda > 0$, the (row) $\lambda$-ridge leverage scores of this matrix are defined as
	\begin{equation}\label{eq-def:leveragescores}
		\ell_{i}^\lambda := \left\| \Phi_{i,\star} (\Phi^\top \Phi + \lambda I)^{-1/2} \right\|_2^2, \text{ for every } i \in [m]. 
	\end{equation}
	
	\begin{definition}[Tensor product]\label{def-tensor-prod}
		Given $x \in \RR^m$ and $y \in \RR^n$ we define the tensor product of these vectors as $x \otimes y = x y^\top$.
		Although tensor products are multidimensional objects, it is convenient to associate them with single-dimensional vectors, so we often associate $x\otimes y$ with $(x_1y_1,x_2y_1, \ldots x_my_1, x_1y_2, \ldots x_my_2, \ldots x_my_n)$.\\
		For shorthand, we use the notation $x^{\otimes p}$ to denote $\underbrace{x\otimes x \otimes \ldots x}_{p \text{ terms}}$, the $p$-fold self-tensoring of $x$.
	\end{definition}
	
	We wish to define the column-wise tensoring of matrices as:
	\begin{definition}\label{def-tensor-prod-matrices}
		Given $A^{(1)}\in \RR^{m_1\times n}, \ldots, A^{(k)}\in \RR^{m_k\times n}$, we define $A^{(1)}\otimes \ldots \otimes A^{(k)}$ to be the matrix in $\RR^{m_1 \ldots m_k \times n}$ whose $j^{\text{th}}$ column is $A^{(1)}_{\star,j} \otimes \ldots \otimes A^{(k)}_{\star,j}$ for every $j\in[n]$.
	\end{definition}
	
	A key property of tensor products that we frequently use is that for any matrices $A, B, C$ with a conforming number of columns, there is a bijective correspondence between the elements of $(A \otimes B) \cdot C^\top$ and $A \cdot ( B \otimes C )^\top$. More precisely, the entry at row $(i,j)$ and column $k$ of $(A \otimes B) \cdot C^\top$ is equal to the entry at row $i$ and column $(j,k)$ of $A \cdot ( B \otimes C )^\top$.
	
	We use a norm-preserving dimensionality reduction technique that can be applied to tensor products in input sparsity time. Specifically, we use the \PolyS transform introduced in \cite{ahle2019oblivious}, which preserves the norms of vectors in $\RR^{d^q}$ and can be applied to tensor product vectors $u_1 \otimes u_2 \otimes \ldots u_q$ very quickly.
	The following lemma follows from Theorem 1.1 of \cite{ahle2019oblivious}.
	
	\begin{lemma}[\PolyS]\label{soda-result}
		For every positive integers $q, d$, and every $\epsilon>0$, there exists a distribution on random matrices $S^q \in \RR^{m \times d^q}$ with $m = O\left(\frac{q}{\epsilon^2} \right)$, called \emph{degree-$q$ \PolyS}, such that,
		\begin{enumerate}
			\item $\Pr\left[ \|S^q Y\|_F^2 \in (1\pm \epsilon)\| Y \|_F^2 \right] \ge 19/20$ for any $Y \in \RR^{d^q \times n}$.
			\item For any vectors $u_1, u_2, \ldots u_q \in \RR^d$, the total time to compute $S^q \left( e_1^{\otimes j} \otimes u_{j+1} \otimes u_{j+2} \otimes \ldots u_q \right)$ for all $j=0, 1, \ldots q$ is
			$O\left( {\frac{q^2 \log^2\frac{q}{\epsilon}}{\epsilon^2}} + \sum_{j=1}^q {\rm nnz}\left(u_{j} \right) \right)$.
		\end{enumerate}
	\end{lemma}
	For a proof of Lemma~\ref{soda-result}, see Appendix~\ref{appndx:prelim-sektch}.
	We also use the Subsampled Randomized Hadamard Transform (\SRHT)~\cite{ailon2009fast}, which is a norm-preserving dimensionality reduction with near linear runtime.
	
	\begin{lemma}[\SRHT  Sketch]\label{lem:srht}
		For every positive integer $d$ and every $\epsilon, \delta>0$, there exists a distribution on random matrices $S \in \RR^{m \times d}$ with $m = O\left(\frac{1}{\epsilon^2} \cdot \log \frac{1}{\epsilon\delta} \log \frac{1}{\delta} \right)$, called \SRHT, such that for any matrix $X \in \RR^{d \times n}$,
		$\Pr\left[ \|S X\|_F^2 \in (1\pm \epsilon)\| X \|_F^2 \right] \ge 1 - \delta$.
		Moreover, $S X$ can be computed in time $O\left(mn + nd\log d \right)$.
	\end{lemma}

	\subsection{Recursive Leverage Score Sampling for $\bigotimes_{j=1}^q X^{(j)}$}\label{sec:recursive-sampling-tensor-product-features} Algorithm~\ref{alg:outerloop} is a generic procedure for sampling the rows of a matrix $\Phi \in \RR^{m \times n}$ with probabilities proportional to their leverage scores, restated from \cite{woodruff2020near}. It starts by generating samples from a crude approximation to the leverage score distribution and then iteratively refines the distribution.
	The core primitive used in Algorithm~\ref{alg:outerloop} is {\sc RowSampler}, which samples rows of a certain matrix with probabilities proportional to their squared norms.
	
	\begin{algorithm}[t]
		\caption{\algoname{Recursive Leverage Score Sampling}}
		{\bf input}: Matrix $\Phi \in \RR^{m \times n}$ and $\lambda , \epsilon, \mu >0$\\
		{\bf output}: Sampling matrix $\Pi \in \RR^{s \times m}$
		\begin{algorithmic}[1]
			\STATE{$s \gets C \frac{\mu}{\epsilon^2} \log_2 n$ for some constant $C$}
			\STATE{$\Pi_0 \gets \{0\}^{1\times m}$, $\lambda_0 \gets \|\Phi\|_F^2/\epsilon$ and $T\gets \log_2 \frac{\lambda_0}{\lambda}$}
			\FOR{$t = 1$ to $T$}
			\STATE{$\Pi_{t} \gets \textsc{RowSampler}\left(\Phi, \Pi_{t-1}\Phi, \lambda_{t-1},s\right)$} \alglinelabel{oversample-alg}
			\STATE{$\lambda_{t} \gets \lambda_{t-1}/2$}
			\ENDFOR
			
			\STATE{\textbf{return} $\Pi_{T}$}
		\end{algorithmic}
		\label{alg:outerloop}
	\end{algorithm}
	
	A \emph{row norm sampler} is defined in \cite{woodruff2020near} as follows,
	\begin{definition}[Row Norm Sampler]\label{def:row-samp}
		Let $\Phi$ be an $m\times n$ matrix and $s$ be some positive integer. A \emph{rank-$s$ row norm sampler} for $\Phi$ is a random matrix $S\in\RR^{s \times m}$ which is constructed by first generating $s$ i.i.d. samples $j_1, j_2, \cdots j_s \in [m]$ from some distribution $\{p_i\}_{i=1}^m$ which satisfies $p_i\ge \frac{1}{4}\frac{\|\phi_{i,\star}\|_2^2}{\|\Phi\|_F^2}$ for all $i\in[m]$, and then letting the $r^{th}$ row of $S$ be $\frac{1}{\sqrt{s \cdot p_{j_r}}}{e}_{j_r}^\top$ for every $r\in [s]$.
	\end{definition}
	Now we restate the correctness guarantee of Algorithm~\ref{alg:outerloop} from \cite{woodruff2020near}.
	\begin{lemma}\label{resursive-rlss-lem}
		Suppose for any matrices $\Phi\in \RR^{m \times n}$ and $B\in \RR^{r\times n}$, any  $\lambda'>0$, and integer $s>0$, the primitive \textsc{RowSampler}$(\Phi,B,\lambda',s)$ returns a rank-$s$ row norm sampler for $\Phi(B^\top B + \lambda' I)^{-1/2}$ as in Definition \ref{def:row-samp}. Then for any $\lambda, \epsilon>0$, any $\Phi \in \RR^{m \times n}$ with statistical dimension $s_\lambda = \|\Phi(\Phi^\top \Phi + \lambda I)^{-1/2}\|_F^2$, and $\mu \ge s_\lambda$, Algorithm \ref{alg:outerloop} returns a sampling matrix $\Pi \in \RR^{s^* \times m}$ with $s^* = O(\frac{\mu}{\epsilon^2} \log n)$ rows such that with probability $1 - \frac{1}{{\rm poly}(n)}$, $\Phi^\top \Pi^\top \Pi \Phi$ is an $(\epsilon , \lambda)$-spectral approximation to $\Phi^\top \Phi$ as in \eqref{spectral-bound}.
	\end{lemma}
	
	Given this lemma, our goal is to run  Algorithm \ref{alg:outerloop} on $\Phi=\bigotimes_{j=1}^q X^{(j)}$ in nearly $\sum_{i}\text{nnz}\left(X^{(i)} \right)$ time. 
	This crucially requires an efficient implementation of \textsc{RowSampler}, which carries out the main computations. We show in Appendix~\ref{appndx:proof-rowsampler-distinct-datasets} that there exists an efficient \textsc{RowSampler} primitive for matrices of the form $\Phi(B^\top B + \lambda I)^{-1/2}$, for any $B$. Our algorithm employs a data-structure for efficient estimation of queries of the form $\left\| \left(\bigotimes_{j=1}^q X^{(j)} \right) V \right\|_F^2$, which we will design and present in Section~\ref{sec:DS}, and heavily exploits various properties of tensor products. See Algorithm~\ref{alg:rotatedrowsampler-poly} and Lemma~\ref{lem:rotatedrowsampler-poly} for details.
	We further prove the following main theorem in Appendix~\ref{appndx:proof-rowsampler-distinct-datasets}. 
	
	\begin{theorem}\label{main-thrm-tensor-prod}
		For any collection of matrices $X^{(1)}, X^{(2)}, \ldots X^{(q)} \in \RR^{d\times n}$ and any $\epsilon, \lambda>0$, if matrix $\Phi := \bigotimes_{j=1}^q X^{(j)}$ has statistical dimension $s_\lambda = \|\Phi(\Phi^\top \Phi + \lambda I)^{-1/2}\|_F^2$ and $\frac{\| \Phi \|_F^2}{\epsilon \lambda} \le {\rm poly}(n)$, then there exists an algorithm that returns a random sampling matrix $\Pi \in \RR^{s \times d^q}$ with $s = O( \frac{s_\lambda}{\epsilon^2} \log n)$ in time $O\left( {\rm poly}(q,\log n, \epsilon^{-1})  \cdot s_\lambda^2 n + \log^4 n \log q \sum_{i}\text{nnz}\left(X^{(i)} \right) \right)$ such that with probability $1 - \frac{1}{{\rm poly}(n)}$, $\Phi^\top \Pi^\top \Pi \Phi$ is an $(\epsilon , \lambda)$-spectral approximation to $\Phi^\top \Phi$ as per \eqref{spectral-bound}.
	\end{theorem}

	\section{Data Structure for Estimating $\left\| \left(\bigotimes_{j=1}^q X^{(j)} \right) V \right\|_F^2$}\label{sec:tensornormDS}\label{sec:DS}
	At the core of our leverage score sampling algorithm, we have a new data-structure (DS) that can efficiently answer queries of the form $\left\| \left(\bigotimes_{j=1}^q X^{(j)} \right) V \right\|_F^2$. In this section, we solve the following DS problem,
	
	\begin{algorithm}[t]
		\caption{DS for estimating $\left\| \left(\bigotimes_{j=1}^q X^{(j)} \right) V \right\|_F^2$}
		{\bf input}: Matrices $X^{(1)},\ldots X^{(q)} \in \RR^{d \times n}$, $\varepsilon > 0$
		
		\begin{algorithmic}[1]
			\STATE{$m \gets C_1 \frac{q}{\varepsilon^2}$, $T \gets C_2 \log n $, $m' \gets  C_3 \frac{\log(1/\varepsilon)}{\varepsilon^2} $}
			
			\STATE{For every $i \in [T]$, let $Q_i \in \RR^{m' \times m}$ be independent copies of the \SRHT as per Lemma~\ref{lem:srht}, and let $S^q_i \in \RR^{m \times d^q}$ be independent copies of the degree-$q$ \PolyS as per Lemma~\ref{soda-result}}
			
			\STATE{Compute $P_{i,j} \gets Q_i \cdot S_i^q \left( E_1^{\otimes j} \otimes X^{(j+1)} \otimes \ldots X^{(q)} \right)$, for every $i \in [T]$ and $j = 0, 1, \ldots q$, where $E_1 \in \RR^{d \times n}$ is defined as $E_1:= [e_1, e_1, \ldots e_1]$}\alglinelabel{alg:DS-sketched-dataset}

			\hspace*{-0.5cm}{\bf Procedure}{ \sc Query ($V, j$)}
			
			\STATE{$\tilde{z}_j \gets \textsc{Median}_{i \in [T] }\left\{ \left\| P_{i,j} \cdot V \right\|_F^2 \right\}$} \alglinelabel{median-alg}
			
			\hspace*{-0.5cm}\textbf{return} $\tilde{z}_j$
		\end{algorithmic}
		\label{alg:TensorNormDS}
	\end{algorithm}

	\paragraph{\textsc{TensorNorm} DS Problem.} For every matrices $X^{(1)}, X^{(2)}, \ldots X^{(q)} \in \RR^{d \times n}$ and every $\epsilon > 0$, we want to design a DS called {\sc TensorNormDS} such that,
	\begin{itemize} 
		\item The time to construct {\sc TensorNormDS} and the memory needed to store it are $\widetilde{O}\left( \sum_{j=1}^q {\rm nnz}\left(X^{(j)} \right) \right)$ and $\widetilde{O} \left( {\rm poly}(q, \epsilon^{-1}) \cdot n  \right)$, respectively.
		\item There exists an algorithm that, given {\sc TensorNormDS} and every query $V \in \RR^{n \times r}$ and $j =0, \ldots q-1$, outputs an estimator $\tilde{z}_j$ in time $\widetilde{O}\left( {\rm poly}(q, \epsilon^{-1}) \cdot {\rm nnz}(V)  \right)$, such that, 
		\begin{equation}\label{eq:tensornormDS-approx-guarantee}
			\tilde{z}_j \in (1\pm \epsilon) \left\| \left(X^{(j+1)} \otimes \ldots X^{(q)} \right) V \right\|_F^2.
		\end{equation}
		
	\end{itemize}
	
	Using \PolyS and \SRHT, we design {\sc TensorNormDS} in Algorithm~\ref{alg:TensorNormDS} and analyze it in the following lemma.
	
	\begin{lemma}[TensorNorm Data-structure]\label{lem:TensorNormDS}
		For any input datasets $X^{(1)}, X^{(2)}, \ldots X^{(q)} \in \RR^{d\times n}$ and any $\epsilon>0$, Algorithm~\ref{alg:TensorNormDS} constructs a DS such that given this DS, the procedure \textsc{Query}$(V, j)$, for any query $V \in \RR^{n \times r}$ and $ j =0, 1, \ldots q$, outputs $\tilde{z}_j$ that satisfies \eqref{eq:tensornormDS-approx-guarantee} with probability $1 - \frac{1}{{\rm poly}(n)}$.
		The time to construct the DS is $O\left( {\frac{q^2 \log^2\frac{q}{\epsilon}}{\epsilon^2}}\cdot  n \log n + \log n \cdot \sum_{j=1}^q {\rm nnz}\left(X^{(j)} \right) \right)$. 
		Additionally, the memory required to store this DS and the runtime of \textsc{Query}$(V, j)$ are $O\left( \frac{q \log(1/\epsilon)}{\epsilon^2} n \log n \right)$ and $O\left( \frac{\log(1/\epsilon)}{\epsilon^2} \log n \cdot {\rm nnz}(V) \right)$, respectively.\\ 
	\end{lemma}
	
	We prove this lemma in Appendix~\ref{appndx:proof-DS}.
	Given this DS and using Algorithm~\ref{alg:outerloop}, we can generate leverage score samples for $\Phi = \bigotimes_{j=1}^q X^{(j)}$.
	
	\section{High Degree Polynomial Kernels}
	Using Theorem~\ref{main-thrm-tensor-prod}, one can spectrally approximate the Gram matrix of a degree-$q$ self tensor product $X^{\otimes q}$, in time $\widetilde{O}\left( {\rm poly}(q,\epsilon^{-1})  \cdot s_\lambda^2 n + q \cdot {\rm nnz} (X ) \right)$. Note that $X^{\otimes q \top} X^{\otimes q}$ is in fact the kernel matrix corresponding to the  degree-$q$ polynomial kernel.
	While this is fast, it is still a factor of $q$ slower than our desired input sparsity runtime (i.e., fastest achievable runtime). We want to understand the following fundamental question: 
	
	\emph{Is the factor $q$ in runtime necessary, or can one achieve a runtime of $\widetilde{O}({\rm nnz} (X ) )$?}
	
	We show that it is possible to shave off the factor $q$ and achieve $\widetilde{O}({\rm nnz} (X ) )$ time complexity, at least for dense datasets $X$. Our main technical tool is a new variant of \SRHT sketches that are partially correlated by sharing the same random signs.
	
	\subsection{\SRHT Sketches with Shared Random Signs}\label{sec:srht-analysis}
	
	Consider the DS problem in Section~\ref{sec:DS} for a self-tensor product matrix $X^{\otimes q}$. To estimate $\left\| X^{\otimes q} \cdot V \right\|_F^2$ for query matrices $V$, we can use \textsc{TensorNormDS} (Algorithm~\ref{alg:TensorNormDS}); however, the time to construct this DS is $\widetilde{O}\left( q \cdot {\rm nnz}(X) \right)$, by Lemma~\ref{lem:TensorNormDS}. Our goal is to improve this runtime by a factor of $q$ and be able to construct this DS in input sparsity time.
	A natural approach for doing so is to first apply a linear sketch, say $S$, on the dataset $X$ to reduce its size (number of rows) and then construct \textsc{TensorNormDS} for $(SX)^{\otimes q}$. To make this work, one needs to ensure that the sketch $S$ satisfies $\left\| ( SX )^{\otimes q} V \right\|_F^2 \approx  \left\|X^{\otimes q} V \right\|_F^2$ for every query $V$ (at least with constant probability). One way of ensuring this condition, as shown in \citep[Lemma~4.5]{song2021fast}, is through requiring $S$ to satisfy the oblivious subspace embedding (OSE) property.
	However, this would require $S$ to have at least $n$ rows, which results in an undesirable quadratic in $n$ running time (recall that our aim is to have a linear in $n$ runtime for constructing the DS). 
	
	On the other hand, an OSE might seem like overkill because we just want to estimate $\left\|X^{\otimes q} V \right\|_F^2$ for some fixed queries $V$. One might hope that the weaker JL property would be sufficient for $S$. However, this is not the case. 
	To see why, suppose for simplicity that $q=2$. Also let $v$ be the all ones vector in $\RR^n$ i.e., $v = {\bf 1}_n$, and let $X \in \RR^{d \times n}$ have orthonormal rows.
	By basic properties of tensor products we have $\left\|X^{\otimes 2} \cdot v \right\|_2^2 = d$ and our estimator is $\left\|(SX)^{\otimes 2} \cdot v \right\|_2^2 = \left\|SX \cdot {\rm diag}(v) \cdot X^\top S^\top \right\|_F^2 = \left\|S S^\top \right\|_F^2$. 
	Now if $S$, for instance, is a random Gaussian matrix, $\left\|S S^\top \right\|_F^2$ is not even an unbiased estimator and has a large bias, i.e., $\EE \left[ \left\|S S^\top \right\|_F^2 \right] \neq \left\|X^{\otimes 2} \cdot v \right\|_2^2 = d$. It is not clear at all that a Gaussian matrix with a small ${\rm poly}(\log n)$ number of rows would be sufficient to have $\left\|S S^\top \right\|_F^2 \approx d$. Note that Sparse JL transforms have even larger variance and bias than Gaussian sketches. The main issue here is the fact that we used a single sketch matrix.
	
	If we had independent JL transforms, $S_1$ and $S_2$, then $\left\|S_1 S_2^\top \right\|_F^2$ would be a good estimator for $\left\|X^{\otimes 2} \cdot v \right\|_2^2 = d$. However, using two identical copies of a single sketch introduces dependencies that are problematic even in the toy example of $q=2$.
	
	Thus, we need to construct a collection of sketches $S^{(1)}, S^{(2)}, \ldots S^{(q)}$ which are \emph{correlated} to the extent that would make computation of $S^{(j)} X$ in total time $\widetilde{O}({\rm nnz}(X))$ possible, and at the same time are \emph{independent enough} to ensure that $\left\| \left( \bigotimes_{j=1}^q S^{(j)}X \right) V \right\|_F^2 \approx \left\|X^{\otimes q} V \right\|_F^2$ while the number of rows of the sketches is small. 
	We achieve this by using a set of correlated \SRHT sketches that can be simultaneously applied to $X$ in a total runtime that only depends on the size of the dataset $X$.
	We prove that for a collection of \SRHT's with shared random signs, the sketched matrices $S^{(j)} X$ can be computed quickly and $\left\| \left( \bigotimes_{j=1}^q S^{(j)}X \right) V \right\|_F^2$ is an unbiased estimator for $\left\|X^{\otimes q} V \right\|_F^2$ with a small variance. It is not clear at this point if variants of sparse sketches (e.g., CountSketch) with these properties exist or not.
	
	Furthermore, note that the eventual use of the DS for estimating $\left\|X^{\otimes q} V \right\|_F^2$ will be in our sampling method in Section~\ref{sec:rowsampler-senftensor} and as it turns out, the queries $V$ that our sampling algorithm produces exhibit some structure. We exploit these structures to prove tighter norm estimation bounds for our new family of correlated \SRHT's in the following lemma.
	
	\begin{lemma}[\SRHT Sketches with Shared Random Signs]\label{lem:SRHT-shared-sign}
		Let $D \in \RR^d$ be a diagonal matrix with i.i.d. Rademacher diagonal entries and let $H \in \RR^{d \times d}$ be the Hadamard matrix and also let $P_1, P_2, \ldots P_q \in \RR^{m \times d}$ be independent random sampling matrices that sample $m$ random coordinates of $\RR^d$. Define the collection of SRHT sketches with shared signs $\left( S^{(1)}, S^{(2)}, \ldots S^{(q)} \right)$ as $S^{(c)} := \frac{1}{\sqrt{m}} \cdot P_c H D$ for $c \in [q]$. For any $X \in \RR^{ d \times n}$, any PSD matrix $K \in \RR^{n \times n}$ with condition number $\kappa:= \frac{\lambda_{\max}(K)}{\lambda_{\min}(K)}$, any matrix $\Sigma \in \RR^{d' \times n}$, and any $ \epsilon, \delta >0 $, if $m = \Omega\left( \left( \frac{1}{\epsilon^2} + \frac{ \kappa}{\epsilon} \right) \cdot \frac{q}{\delta} \log n \right)$, then with probability at least $1 - \delta$,
		\[
		\left\| \left[ \bigotimes_{c=1}^q S^{(c)} X \right]  \left( \Sigma \otimes K \right)^\top \right\|_F^2 \in (1 \pm \epsilon) \left\| X^{\otimes q}  \left( \Sigma \otimes K \right)^\top \right\|_F^2
		\]
		Furthermore, the total time to compute $S^{(1)} X, \ldots S^{(q)} X$ is bounded by $O\left( q m n + n d \log d \right)$.
	\end{lemma}
	
	We prove this lemma in Appendix~\ref{appndx:srht-shared-sign}.
	According to Lemma~\ref{lem:SRHT-shared-sign}, the Kronecker product of \SRHT sketches with shared random signs $S^{(1)} \times S^{(2)} \times \ldots S^{(q)}$ acts as a near-isometry for matrices of the form $X^{\otimes q} \cdot (\Sigma \otimes K)^\top$ with constant probability, as long as the target dimension of the $S^{(c)}$'s is at least $m \approx  (\epsilon^{-2} + \epsilon^{-1}\kappa) q\log n $. 
	If the $S^{(c)}$ sketches were fully independent, as in \cite{ahle2019oblivious}, then a target dimension of $m \approx \epsilon^{-2} q\log n$ would suffice. So the price of using correlated sketches is a factor of $\epsilon\kappa + 1$ increase in the target dimension. On the other hand, letting the sketches $S^{(c)}$ use independent sampling matrices is critical. If we used identical \SRHT's $S^{(1)} = \ldots S^{(q)} = S$, as is done in Lemma~4.5 of \cite{song2021fast}, then to have the guarantee of Lemma~\ref{lem:SRHT-shared-sign}, the sketch $S$ would need to be an OSE, which requires a target dimension of $m = \Omega\left( \frac{q^2}{\epsilon^2} \cdot n \log n \right)$. Lemma~\ref{lem:SRHT-shared-sign} provides a target dimension improvement over the OSE-based results by a factor of $\frac{q n}{1+\epsilon \kappa}$, which is significant.

	Lemma~\ref{lem:SRHT-shared-sign} shows us a way of speeding up the DS given in Algorithm~\ref{alg:TensorNormDS} for self tensor products $X^{\otimes q}$. One can quickly compute sketched datasets $Y^{(r)} = S^{(r)} X$ for every $r \in [q]$, and then apply \textsc{TensorNormDS} to $Y^{(1)}, \ldots Y^{(q)}$, in total time $\widetilde{O}({\rm nnz}(X))$ for dense $X$.
	It turns out that all queries that our sampling algorithm in Section~\ref{sec:rowsampler-senftensor} produces are exactly of the form $V = (\Sigma \otimes K)^\top$. Thus, the combination of Lemma~\ref{lem:SRHT-shared-sign} and Algorithm~\ref{alg:TensorNormDS} is a perfect solution for our sampling algorithm's norm estimation needs.
	
	\subsection{\textsc{RowSampler} for Degree-$q$ Self-Tensor Products}\label{sec:rowsampler-senftensor}
	
	In this section, we design an algorithm that can perform row norm sampling (see Definition \ref{def:row-samp}) on a matrix of the form $X^{\otimes q} (B^\top B + \lambda I)^{-1/2}$ using $\widetilde{O} ({\rm nnz} (X) )$ runtime for dense $X$. 
	Our primitive crucially relies on {\sc TensorNormDS} (Algorithm~\ref{alg:DS-sketched-dataset}) as well as our new variant of \SRHT with shared random signs that we analyzed in Lemma~\ref{lem:SRHT-shared-sign}.
	
	\begin{algorithm}[!t]
		\caption{\algoname{RowSampler} for $X^{\otimes q}$}
		{\bf input}: $q, s\in \mathbb{Z}_+$, $X \in \RR^{d \times n}$, $B \in \RR^{m \times n}$, $\lambda >0$\\
		{\bf output}: Sampling matrix $S \in \RR^{s \times d^q}$
		\begin{algorithmic}[1]
			\STATE{$\kappa \gets \sqrt{ \frac{\| B^\top B \|}{\lambda} + 1 }$}
			\STATE{Generate $H\in\RR^{d' \times n}$ with i.i.d. normal entries with $d'= C_0 q^2 \log n$ rows}
			\STATE{$M \gets H \cdot (B^\top B + \lambda I)^{-1/2}$}\alglinelabel{M-alg-selftensor}
			
			\STATE{For every $k \in [m']$, let $S_k^{(1)}, S_k^{(2)} , \ldots S_k^{(q)} \in \RR^{m'' \times d}$ be independent copies of \SRHT sketches with shared signs as per Lemma~\ref{lem:SRHT-shared-sign}, where $m' = C_1\log n$ and $m'' = C_2 (q^3 + q^2 \kappa ) \log n$}\alglinelabel{srht-sharedsign}
			
			\STATE{For every $k \in [m']$, let $\textsc{TN}^{(k)}$ be the DS in Algorithm~\ref{alg:TensorNormDS} for inputs $\left( S_k^{(1)} X, \ldots S_k^{(q)} X, M \right)$ and $\epsilon=\frac{1}{40q}$}\alglinelabel{tnormDS}

			\STATE{Let $h:[d]\rightarrow[s']$ be a fully independent and uniform hash function with $s' = \lceil {q}^{3}s \rceil$ buckets}
			\STATE{Let $h^{-1}(r)=\left\{ j\in[d]:h(j)=r \right\}$ for every $r\in[s']$}
			
			\STATE{For every $r\in[s']$ and $k \in [m']$, lget $G_r^k \in \RR^{n'\times d_r }$ be independent instances of degree-$1$ \PolyS as per Lemma~\ref{soda-result}, where $d_r=|h^{-1}(r)|$, $n'=C_3q^2$}
			
			\STATE{$W_{r, k} \gets G_r^k \cdot X_{h^{-1}(r),\star}$ for every $k \in [m']$ and $r\in[s']$}\alglinelabel{W-selftensor}
			
			\FOR{$\ell=1$ to $s$}
			
			\STATE{$D^{1} \gets I_n$ and $\beta_\ell \gets s$}\alglinelabel{beta-forloop-selftensor}
			\FOR{$a=1$ to $q$}

			\STATE{$L^a_{r, k} \gets D^{a} \cdot W_{r, k}^{\top}$ for every $k \in [m']$, and $r\in[s']$}
			
			\STATE{$p^{a}_r \gets  \textsc{Median}_{k \in [m']} \textsc{TN}^{(k)}\textsc{.Query}( L^a_{r, k} , a )$ for every $r \in [s']$}
			\STATE{$p^{a}_r \gets  p^{a}_r/ \sum_{t=1}^{s'} p^{a}_t$ for every $r \in [s']$}\alglinelabel{dist-par-selftensor}
			\STATE{Sample $t \in [s']$ from distribution $\{p^a_r\}_{r=1}^{s'}$}
			
			\STATE{Let $q^{a}_i \gets \textsc{Median}_{k \in [m']} \textsc{TN}^{(k)}\textsc{.Query} ( D^{a} X_{i,\star}^\top , a )$ for every $i\in h^{-1}(t)$}
			\STATE{$q^{a}_i \gets  q^{a}_i/ \sum_{j \in h^{-1}(t)} q^{a}_j$ for every $i\in h^{-1}(t)$}\alglinelabel{dist-qai-selftensor}
			
			\STATE{Sample $i_a \in [d]$ from distribution $\{q^a_i \}_{i\in h^{-1}(t)}$}\alglinelabel{isample-alg-selftensor}
			
			\STATE{$D^{a+1}\gets D^{a}\cdot {\rm diag}\left( X^{(a)}_{i_a,\star} \right)$}\alglinelabel{def-Da-selftensor}
			
			\STATE{$\beta_\ell \gets  \beta_\ell \cdot p^{a}_{t} q_{i_a}^{a} $} \alglinelabel{beta-update-selftensor}
			\ENDFOR
			
			\STATE{Let $\ell^{th}$ row of $S$ be $ \beta_\ell^{-1/2}  
				\left({e}_{i_1} \otimes {e}_{i_2} \otimes \cdots {e}_{i_q}\right)^\top$}
			\ENDFOR
			\STATE{\textbf{return} $S$}
		\end{algorithmic}
		\label{alg:rotatedrowsampler-poly-selftensor}
	\end{algorithm}

	\paragraph{Overview of Algorithm \ref{alg:rotatedrowsampler-poly-selftensor}:} The goal of \textsc{RowSampler} is to generate samples $(i_1,i_2, \cdots i_q) \in [d]^q$ with probabilities proportional to the squared norm of the row $(i_1, \cdots i_q)$ of matrix $X^{\otimes q} (B^\top B + \lambda I)^{-1/2}$. Because $(B^\top B + \lambda I)^{-1/2}$ has a large $n \times n$ size, we first compress it without perturbing the distribution of row norms of $ X^{\otimes q} (B^\top B + \lambda I)^{-1/2}$ too much.
	This can be done by applying a JL-transformation to the rows of this matrix (see, e.g., \cite{dasgupta2003elementary}). 
	Let $H \in \RR^{d' \times n}$ be a random matrix with i.i.d. normal entries with $d' = C_0 q^2 \log_2 n$ rows. With probability $1 - \frac{1}{{\rm poly}(n)}$ the norm of each row of the sketched matrix $ X^{\otimes q} (B^\top B + \lambda I)^{-1/2} \cdot H^\top$ is preserved up to a $\left( 1 \pm O (q^{-1}) \right)$ factor. This is done in line~\ref{M-alg-selftensor} of the algorithm by computing $M := H \cdot (B^\top B + \lambda I)^{-1/2}$, which can be computed quickly since $B$ and $H$ have a small number of rows. 
	
	Now the problem is reduced to performing row norm sampling on $ X^{\otimes q} M^\top$. 
	Note that computing the exact row norms of this matrix is out of the question since it has a huge $d^q$ number of rows. However, by using \textsc{TensorNormDS} that we designed in Algorithm~\ref{alg:TensorNormDS} and the new variant of \SRHT sketches we introduced in Lemma~\ref{lem:SRHT-shared-sign} and by exploiting properties of tensor products we can generate samples from the row norm distribution as follows.
	
	By basic properties of tensor products, the entries of $X^{\otimes q} M^\top$ are in bijective correspondence with the entries of $X^{\otimes (q-1)} \cdot (X \otimes M )^\top$, where the entry at row $(i_1,i_2, \cdots i_q)$ and column $j$ of $X^{\otimes q} M^\top$ is equal to the entry at row $(i_2, \ldots i_q)$ and column $(i_1,j)$ of $ X^{\otimes (q-1) } \cdot ( X \otimes M )^\top$. 
	
	Therefore, it is enough to have a procedure to sample $(i_1,i_2, \ldots i_q)$ with probability proportional to the squared norm of the row $(i_2, \ldots i_q)$ of matrix $X^{\otimes (q-1)} \cdot \left( M \cdot {\rm diag}\left( X_{i_1,\star} \right) \right)^\top$ for every $i_1 \in [d]$. 
	We do this task in two steps; first we sample an index $i_1$ with probability proportional to the squared Frobenius norm of $X^{\otimes (q-1)} \cdot \left( M \cdot {\rm diag}\left( X_{i_1,\star} \right) \right)^\top$, and then we perform row norm sampling on the sampled matrix.

	To do the first sampling step above, we need to cheaply estimate the Frobenius norms of matrices $X^{\otimes (q-1)} \cdot \left( M \cdot {\rm diag}\left( X_{i_1,\star} \right) \right)^\top$. We can estimate such norms using \textsc{TensorNormDS} given in Algorithm~\ref{alg:TensorNormDS}. However,  note that $\sum_{j=1}^{q-1} \widetilde{O}({\rm nnz}(X)) = \widetilde{O}(q \cdot {\rm nnz}(X))$ operations are required to build this DS. This is where the \SRHT sketches with shared random signs plays an important role. If we let $S^{(1)}, \ldots S^{(q)} \in \RR^{m'' \times d}$ be the \SRHT sketches with shared signs as per Lemma~\ref{lem:SRHT-shared-sign}, then we can compute $S^{(c)} X$ for all $c \in [q]$ in time $O(nd \log d) = \widetilde{O}({\rm nnz}(X))$, for dense datasets $X$. Now we can cheaply estimate the Frobenius norms of matrices $\left( \bigotimes_{c=1}^{q-1} S^{(c)}X \right) \cdot \left( M \cdot {\rm diag}\left( X_{i_1,\star} \right) \right)^\top$ up to a small perturbation using \textsc{TensorNormDS} (Algorithm~\ref{alg:TensorNormDS}) because the sketched matrices $S^{(c)} X$ have small sizes.
	We let the target dimension of these sketches be $m'' = C_2 (q^3 + q^2 \kappa) \log n$, where $\kappa = \sqrt{ \frac{\| B^\top B \| }{\lambda} + 1 }$ is the condition number of $(B^\top B + \lambda I)^{-1/2}$. Thus, by Lemma~\ref{lem:SRHT-shared-sign} and using the fact that matrix $M = H (B^\top B + \lambda I)^{-1/2}$ for a JL matrix $H$, the Frobenius norm of $\left( \bigotimes_{c=1}^{q-1} S^{(c)}X \right) \cdot \left( M \cdot {\rm diag}\left( X_{i_1,\star} \right) \right)^\top$ is within a factor $(1 \pm O(q^{-1}))$ of the Frobenius norm of $X^{\otimes (q-1)} \left( (B^\top B + \lambda I) \cdot {\rm diag}\left( X_{i_1,\star} \right) \right)^\top$.

	After this point, we will have an index $i_1 \in [d]$ sampled from the correct distribution and all that is left to do is to carry out row norm sampling on $X^{\otimes (q-1)} \left( M \cdot {\rm diag}\left( X^{(1)}_{i_1,\star} \right) \right)^\top$. Note that we have made progress because this matrix has $d^{q-1}$ rows, so we have reduced the size of our problem by a factor of $d$. 
	Algorithm~\ref{alg:rotatedrowsampler-poly} recursively repeats this process of reshaping and sketching and sampling with the aid of our DS, $q$ times until having all $q$ indices $i_1,i_2, \cdots i_q$. Note that the actual procedure requires more work because we need to generate $s$ i.i.d. samples from the distribution of row norms,  and in order to ensure that the runtime does not lose a multiplicative factor of $s$, resulting in $s\cdot \text{nnz}( X)$ total time, we need to perform additional sketching and a random partitioning of the rows of the datasets to ${q}^{3} s$ buckets. We also boost the success probability of all these operations, when necessary, using the  median trick.
	
	The formal guarantee on Algorithm~\ref{alg:rotatedrowsampler-poly-selftensor} is given in the following lemma.

	\begin{lemma}\label{lem:rotatedrowsampler-poly-selftensor}
		For any matrix $X \in \RR^{d\times n}$ and $B\in\RR^{m\times n}$, any $\lambda>0$ and any positive integers $q,s$, with probability at least $1 - \frac{1}{{\rm poly}(n)}$, Algorithm~\ref{alg:rotatedrowsampler-poly-selftensor} outputs a rank-$s$ row norm sampler for $ X^{\otimes q} (B^\top B + \lambda I)^{-1/2}$ as per Definition \ref{def:row-samp} in time $O\left( m^2n + q^{8} s^2 n  \log^3 n + q^3 \kappa n \log^3 n + nd \log^4n \right)$, where $\kappa = \sqrt{\| B^\top B \| / \lambda +1 } $.
	\end{lemma}
	We prove Lemma~\ref{lem:rotatedrowsampler-poly-selftensor} in Appendix~\ref{appndx-proof-rowsampler-selftensor}.
	Now we can give our main theorem about spectrally approximating the degree-$q$ polynomial kernel matrix $X^{\otimes q \top} X^{\otimes q}$ using nearly $\text{nnz}(X)$ runtime for dense datasets.
	
	\begin{theorem}\label{main-thrm-selftensor-prod}
		For any dataset $X \in \RR^{d\times n}$ and any $\epsilon, \lambda>0$, if matrix $\Phi := X^{ \otimes q}$ has statistical dimension $s_\lambda = \|\Phi(\Phi^\top \Phi + \lambda I)^{-1/2}\|_F^2$ and $\frac{\| \Phi \|_F^2}{\epsilon \lambda} \le {\rm poly}(n)$, then there exists an algorithm that returns a random sampling matrix $\Pi \in \RR^{s \times d^q}$ with sampling dimension $s = O( \frac{s_\lambda}{\epsilon^2} \log n)$ in time $O\left( \frac{q^{8} s_\lambda^2 n \log^5 n}{\epsilon^4} + \sqrt{ \frac{\| \Phi^\top \Phi \|}{\lambda} } q^3 n\log^3n + nd \log^5 n \right)$ such that with probability $1 - \frac{1}{{\rm poly}(n)}$, $\Phi^\top \Pi^\top \Pi \Phi$ is an $(\epsilon , \lambda)$-spectral approximation to $\Phi^\top \Phi$ as per \eqref{spectral-bound}.
	\end{theorem}
	For a proof of this theorem see Appendix~\ref{appndx:proof-maintheorem}.
	
	{\bf Remark on the runtime of Theorem~\ref{main-thrm-selftensor-prod}.}
	Assuming that $\frac{\| \Phi^\top \Phi \|}{\lambda} \le \text{poly}(q/\epsilon)\cdot s_\lambda^4$, the low order term of our algorithm's runtime is $\widetilde{O}(\text{poly}(q/\epsilon) \cdot s_\lambda^2 n)$. While the quadratic dependence on $s_\lambda$ might seem like a limitation, we argue that for a wide range of downstream applications this is not an issue. 
	In particular, for applications such as regression or PCA, one needs to either invert or compute the SVD of the approximated Gram matrix $(\Pi \Phi)^\top (\Pi \Phi)$ and both of these operations require $s^2 n$ runtime, where $s$ is the target dimension of the matrix $\Pi$. 
	Note that for any method to achieve the spectral approximation guarantee of \eqref{spectral-bound}, the target dimension has to be at least $s = \Omega(s_\lambda)$~\cite{avron2019universal}. Thus, the runtime of solving the mentioned downstream learning tasks using any sketching or sampling method is at least $\Omega(s_\lambda^2 n)$, which shows that quadratic dependence on $s_\lambda$ is unavoidable. For comparison against prior results note that, the sketch in~\cite{song2021fast} has a target dimension of $m \approx n/\epsilon^2$. Thus, the total time of using their algorithm to approximately solve kernel ridge regression (KRR) or PCA is $\Theta(n^3 / \epsilon^4 +q^2n^2/\epsilon^2 + dn)$.

	\section{Generalization to Other Kernels}\label{sec:generalization}
	In this section we generalize our sampling algorithms to other classes of kernels such as Gaussian, dot-product, and Neural Tangent kernels. We start by defining a class of kernels that encompasses all aforementioned kernels,
	
	\begin{definition}[Generalized Polynomial Kernel]\label{def;GPK}
		Given a positive integer $q$, a vector of coefficients $\alpha \in \RR^{q+1}$, a vector $v \in \RR^{n}$, and a dataset $X \in \RR^{d \times n}$, we define the corresponding \emph{generalized polynomial kernel (GPK)} matrix $K \in \RR^{n \times n}$ as $K := {\rm diag}(v) \left( \sum_{j =0}^q \alpha_j^2 \cdot X^{\otimes j \top} X^{\otimes j} \right) {\rm diag}(v)$.
		The GPK matrix can be expressed as a Gram matrix $K = \Phi^\top \Phi$ for
		\begin{equation}\label{feature-matrix-GPK}
			\Phi := \bigoplus_{j=0}^q \alpha_j X^{\otimes j} \cdot {\rm diag}(v).
		\end{equation}
	\end{definition}
	
	We show in Appendix~\ref{appndx-GPK}, how to adapt our leverage score sampling method to the GPK feature matrix $\Phi$ defined in \eqref{feature-matrix-GPK} and prove the following main theorem,
	
	\begin{theorem}\label{main-thrm-gpk}
		Let $\Phi \in \RR^{m \times n}$ and $K$ be the GPK feature matrix and kernel matrix defined in Definition~\ref{def;GPK}. For any $\epsilon, \lambda>0$, if $\Phi$ has statistical dimension $s_\lambda = \|\Phi(K + \lambda I)^{-1/2}\|_F^2$ and $\frac{\| \Phi \|_F^2}{\epsilon \lambda} \le {\rm poly}(n)$, then there exists an algorithm that returns a random sampling matrix $\Pi \in \RR^{s \times m}$ with $s = O( \frac{s_\lambda}{\epsilon^2} \log n)$ rows in time $O\left( \frac{q^{8} s_\lambda^2 n \log^5 n}{\epsilon^4} + \sqrt{ \frac{\| K \|}{\lambda} } q^3 n\log^3n + nd \log^5 n \right)$ such that with probability $1 - \frac{1}{{\rm poly}(n)}$, $\Phi^\top \Pi^\top \Pi \Phi$ is an $(\epsilon , \lambda)$-spectral approximation to $K$ as per \eqref{spectral-bound}.
	\end{theorem}
	
	\paragraph{Gaussian Kernel.} We show in Appendix~\ref{appndx-applications-Gauss} that the class of GPK kernels contains a good approximation to the Gaussian kernel matrix for datasets with bounded $\ell_2$ norm and therefore, we have the following corollary of Theorem~\ref{main-thrm-gpk}:
	
	\begin{corollary}[Application to Gaussian Kernel]\label{main-theorem-Gaussian}
		For any $r>0$ and dataset $x_1, \ldots x_n \in \RR^{n}$ with $\max_{i\in [n]} \| x_i \|_2^2 \le r$, any $\lambda, \epsilon>0$, if $K \in \RR^{n \times n}$ is the Gaussian kernel matrix, i.e., $K_{i,j} := e^{-\| x_i - x_j \|_2^2/2}$, with statistical dimension $s_\lambda = {\rm tr}\left( K (K + \lambda I)^{-1} \right)$, then there exists an algorithm that computes $Z \in \RR^{s \times n}$ with $s = O( \frac{s_\lambda}{\epsilon^2} \log n)$ in time $\widetilde{O}\left( \frac{r^{8}}{\epsilon^4} s_\lambda^2 n + r^3 \sqrt{ \frac{\| K \|}{\lambda} } n + nd \right)$ such that with probability $1 - \frac{1}{{\rm poly}(n)}$, $Z^\top Z$ is an $(\epsilon , \lambda)$-spectral approximation to $K$.
	\end{corollary}
	Note that for the Gaussian kernel we have $\| K \| \le {\rm tr}(K) = n$. Therefore, for constant $\epsilon$, the runtime of Corollary~\ref{main-theorem-Gaussian} is always upper bounded by $\widetilde{O}\left( r^{8} s_\lambda^2 n + r^3 \sqrt{ \frac{n}{\lambda} } \cdot n + nd \right)$. 
	For comparison, the runtime of \cite{song2021fast} for spectrally approximating the Gaussian kernel matrix is $\widetilde{O}\left( r^3 \cdot n^2 + nd \right)$, which means that for any $\lambda =\omega (1/{n})$ and any $r = o\left(n^{0.2}\right)$, our runtime is strictly faster than the runtime of \cite{song2021fast}.

	\paragraph{Neural Tangent Kernel (NTK).} 
	We consider the NTK corresponding to an infinitely wide neural network with two layers and ReLU activation function. This kernel function is defined as follows for any $x,y \in \RR^{d}$~\cite{zandieh2021scaling}
	\begin{align}
		&\Theta_{{\tt ntk}}(x,y) := \|x\|_2\|y\|_2 \cdot k_{{\tt ntk}}\left( \frac{\langle x , y \rangle}{\|x\|_2\|y\|_2} \right),\label{eq:def-ntk}\\
		&k_{{\tt ntk}}(\beta) := \frac{1}{\pi} \left( \sqrt{1-\beta^2} + 2 \beta(\pi - \arccos{\beta}) \right).\nonumber
	\end{align}
	We show in Appendix~\ref{appndx-applications-ntk} that there exists a GPK that well-approximates $\Theta_{{\tt ntk}}(x,y)$ defined in \eqref{eq:def-ntk} on datasets with bounded $\ell_2$ norm. Thus, we have the following corollary of Theorem~\ref{main-thrm-gpk}:
	
	\begin{corollary}[Application to NTK]\label{main-theorem-ntk}
		For any $r>0$ and dataset $x_1, \ldots x_n \in \RR^{n}$ with $\max_{i\in [n]} \| x_i \|_2^2 \le r$, any $\lambda, \epsilon>0$, if $K \in \RR^{n \times n}$ is the NTK kernel matrix, i.e., $K_{i,j} := \Theta_{{\tt ntk}}(x_i,x_j)$ as per \eqref{eq:def-ntk}, with statistical dimension $s_\lambda = {\rm tr}\left( K (K + \lambda I)^{-1} \right)$, then there exists an algorithm that computes $Z \in \RR^{s \times n}$ with $s = O( \frac{s_\lambda}{\epsilon^2} \log n)$ in time $\widetilde{O}\left( \left( \frac{nr}{\epsilon\lambda} \right)^{16} \frac{s_\lambda^2 n}{\epsilon^4} + nd \right)$, such that with probability $1 - \frac{1}{{\rm poly}(n)}$, $Z^\top Z$ is an $(\epsilon , \lambda)$-spectral approximation to $K$.
	\end{corollary}
	Note that, for constant $\epsilon$ and any $r = (\log n)^{O(1)}$, the runtime of Corollary~\ref{main-theorem-ntk} is upper bounded by $\widetilde{O}\left( \left( \frac{n}{\lambda} \right)^{16} \cdot s_\lambda^2 n  + nd \right)$. 
	For comparison, the runtime of \cite{song2021fast} for spectrally approximating the NTK on datasets with unit radius $r=1$ is $\widetilde{O}\left( n^{11/3} + nd \right)$, which means that for any $\lambda =\omega (n^{5/6})$, our runtime is strictly faster than the runtime of \cite{song2021fast}.
	Furthermore, the random features proposed in \cite{zandieh2021scaling} requires $\widetilde{O}\left( (n/\lambda) \cdot n d^2 \right)$ operations to spectrally approximate the NTK, which is slower than our runtime for high dimensional datasets with $d = \omega\left( (n/\lambda)^{15} \right)$. Additionally, Corollary~\ref{main-theorem-ntk} applies to datasets with arbitrary radius $r$ while both of \cite{song2021fast} and \cite{zandieh2021scaling} only apply to datasets with unit radius.
	
	\section{Experiments}\label{sec:experiments}
	In this section we apply our sampling algorithm to accelerate regression and classification on real-world datasets.
	We approximately solve the kernel ridge regression problem by running least squares regression on the features sampled by our algorithm.
	We also reduce the classification problem to regression by applying a one-hot encoding to the labels of classes and then use our fast regression method to solve it. 
	In the experiments, we focus on ridge regression with a Gaussian kernel as well as the depth-$1$ Neural Tangent kernel, and compare our result from Corollaries~\ref{main-theorem-Gaussian} and \ref{main-theorem-ntk} to various popular sampling and sketching methods for Gaussian and Neural Tangent kernels. 
	The classification error rate and root mean square error (RMSE) on the testing sets are summarized in Table~\ref{tab-regression-classification} (average over 5 trials with different random seeds). For each task, the number of features and sketching dimensions are chosen to be equal across all different methods. Thus, we can compare different methods given that the memory needed to store the approximate kernel matrices is equal for all methods.

	\setlength{\tabcolsep}{6pt}
	\begin{table}[t]
		\caption{Approximate kernel ridge regression/classification with Gaussian and Neural Tangent kernels. We denote the ridge parameter by $\lambda$, and the number of samples or sketching dimension of different methods by $s$. The RMSE and classification error rates are measured on the testing sets for each task.} \label{tab-regression-classification}
		
		\centering
		\scalebox{0.85}{
			\begin{tabular}{@{}lcccccccc@{}}
				\hline
				Data-set:& \multicolumn{2}{c}{MNIST} &   \multicolumn{2}{c}{Location of CT}\\
				\hline
				$n / d$ & \multicolumn{2}{c}{$60{,}000$ / $784$} & \multicolumn{2}{c}{$53{,}500$ / $384$} \\
				$\lambda$ / $s$ & \multicolumn{2}{c}{$1$ / $1{,}000$} & \multicolumn{2}{c}{$0.5$ / $2{,}000$} \\
				Kernel function & \multicolumn{2}{c}{$\Theta_{\tt ntk}(x,y)$} & $\Theta_{\tt ntk}(x,y)$ & $e^{-\frac{\|x-y\|^2}{40}}$ \\
				Metric & \multicolumn{2}{c}{Error ($\%$)} & RMSE & RMSE  \\
				\hline\hline
				\makecell[l]{Fourier Features\\ {\small\citep{rahimi2008random}}} & \multicolumn{2}{c}{--} & {--} & 4.92   \\
				\hline
				\makecell[l]{PolySketch \\ {\small \citep{ahle2019oblivious}}\\ {\small \citep{zandieh2021scaling}}} & \multicolumn{2}{c}{5.92} &  4.87 & 5.05 \\
				\hline
				\makecell[l]{Accelerated PolySketch \\ {\small \citep{song2021fast}}} & \multicolumn{2}{c}{6.07}  & 4.93 & 5.14 \\
				\hline
				\makecell[l]{Adaptive Sampling \\ {\small \citep{woodruff2020near}}} & \multicolumn{2}{c}{5.87}  & 4.72 & 4.76 \\
				\midrule
				\makecell[l]{Our Method \\ Corollaries~\ref{main-theorem-Gaussian} and \ref{main-theorem-ntk}} & \multicolumn{2}{c}{\bf 5.44}   & {\bf 4.71} & {\bf 4.76}  \\
				\hline
			\end{tabular}
		}
		
	\end{table}
	
	While our theoretical results guarantee that for large enough datasets in high dimensions our method performs better than prior work, our experiments verify that even for moderately-sized datasets with dimension $d < 1000$ our method performs well.
	In particular, we achieve the best RMSE and classification error rate compared to all other methods under the condition that the number of sampled features or sketching dimension is fixed for each method. We remark that the Fourier features method~\cite{rahimi2008random} only applies to shift invariant kernels such as the Gaussian kernel and cannot be used for Neural Tangent kernels. On the other hand, the sketching methods of \cite{ahle2019oblivious} and \cite{song2021fast} can be used to sketch the Taylor expansion of the NTK, as was previously done in \cite{zandieh2021scaling}.
	
	\paragraph{Accuracy/memory trade-off.} Figure~\ref{fig-mnist-class} shows the trade-off of various methods for MNIST classification using the NTK kernel function. We plot the testing set accuracy as a function of the number of samples or sketching dimension, which is a parameter that directly controls the memory usage of different methods. It can been seen that our method has the best accuracy/memory trade-off.

	\begin{figure}[!t]
		\centering
		\includegraphics[width=0.42\textwidth]{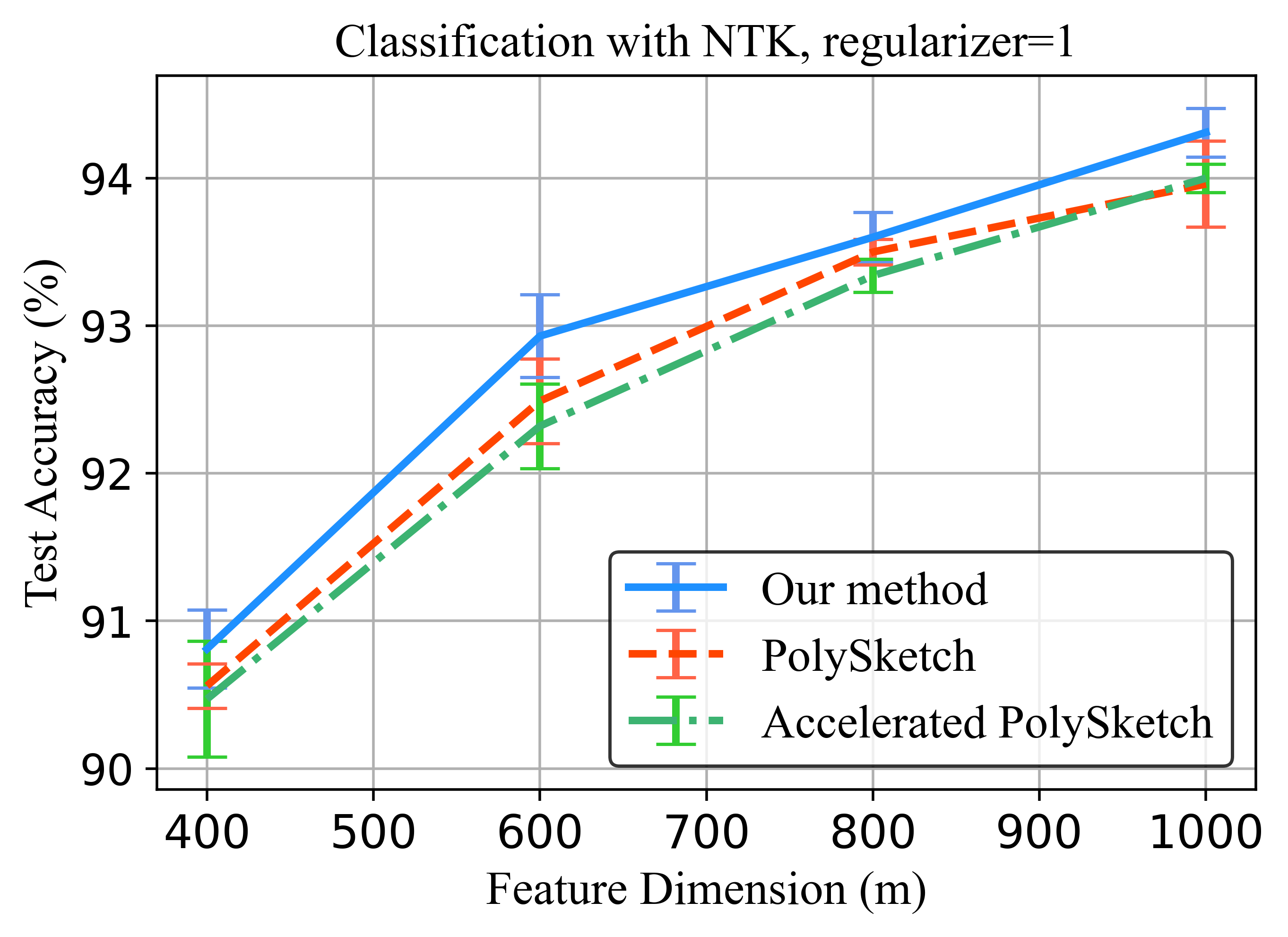}
		\caption{Approximate classification of the MNIST dataset using depth-$1$ Neural Tangent KRR. The ridge parameter is $\lambda=1$. The classification error rates are measured on the testing set.}\label{fig-mnist-class}
	\end{figure}

	\section*{Acknowledgements}
	David Woodruff would like to thank NSF grant No. CCF-1815840, NIH grant 5401 HG 10798-2, ONR grant N00014-18-1-2562, and a Simons Investigator Award. 
	Amir Zandieh was supported by the Swiss NSF grant No. P2ELP2\textunderscore 195140.


	\nocite{langley00}
	
	\bibliography{example_paper}
	\bibliographystyle{icml2022}

	\newpage
	\appendix
	\onecolumn

	\section{Preliminary Sketching Results}\label{appndx:prelim-sektch}
	In this section we provide preliminary sketching results. In particular, we provide a proof of Lemma~\ref{soda-result}.
	
	\begin{proofof}{Lemma~\ref{soda-result}}
		By invoking Corollary 4.1 of \cite{ahle2019oblivious}, we find that there exists a random sketch $S^q \in\RR^{m \times d^q}$ such that if $m = C \cdot q \cdot\varepsilon^{-2}$ for some absolute constant $C$, then this sketch satisfies the $(\epsilon, 1/20, 2)$-JL-moment property. It follows from the definition of the JL-moment property along with Minkowski's Inequality that for any $Y\in \RR^{d^q \times n}$, 
		\[ \EE \left[ \left| \left\| S^q Y \right\|_F^2 - \left\| Y \right\|_F^2 \right|^2 \right] \le \epsilon^2/20 \cdot \left\| Y \right\|_F^4. \]
		Thus, by applying Markov's inequality on $\left| \left\| S^q Y \right\|_F^2 - \left\| Y \right\|_F^2 \right|^2$, we find that
		\begin{align*}
			\Pr\left[ \left\| S^q Y \right\|_F^2 \in (1\pm\varepsilon) \|Y\|_F^2\right] \ge 19/20.
		\end{align*}
		This immediately proves the first statement of the lemma.
		
		\begin{figure*}
			\centering
			
			\scalebox{0.8}{
				\begin{tikzpicture}[<-, level/.style={sibling distance=75mm/#1,level distance = 2.3cm}]
					\node [arn_t] (z){}
					child {node [arn_t] (a){}edge from parent [electron]
						child {node [arn_t] (b){}edge from parent [electron]
						}
						child {node [arn_t] (e){}edge from parent [electron]
						}
					}
					child { node [arn_t] (h){}edge from parent [electron]
						child {node [arn_t] (i){}edge from parent [electron]
						}
						child {node [arn_t] (l){}edge from parent [electron]
						}
					};
					
					\node []	at (z.south)	[label=\large{${\bf S_{\text{base}}}$}]	{};

					\node []	at (a.south)	[label=\large{${\bf S_{\text{base}}}$}]	{};

					\node []	at (b.south)	[label=\large${\bf T_{\text{base}}}$]	{};
					\node []	at (e.south)	[label=\large${\bf T_{\text{base}}}$] {};
					\node []	at (h.south)	[label=\large{${\bf S_{\text{base}}}$}] {};

					\node []	at (i.south)	[label=\large${\bf T_{\text{base}}}$] {};
					\node []	at (l.south)	[label=\large${\bf T_{\text{base}}}$] {};

					\draw[draw=black, ->] (2.2,0.2) -- (0.7,0.1);
					\draw[draw=black, ->] (4.5,0) -- (3.8,-1.7);
					
					\node [] at (2.2,0.5) [label=right:\large{ internal nodes: {\sc TensorSketch}}]	{}
					edge[->, bend right=45] (-3.6,-1.7);
					
					\draw[draw=black, ->] (4.0,-6.1) -- (5.2,-5.2);
					\draw[draw=black, ->] (2.8,-6.1) -- (2,-5.2);
					\draw[draw=black, ->] (1.9,-6.1) -- (-1.6,-5.15);
					
					\node [] at (1,-6.5) [label=right:\large{leaves: {\sc CountSketch}}]	{}
					edge[->, bend left=15] (-5.3,-5.1);
					
				\end{tikzpicture}
			}
			\par

			\caption{The structure of sketch $S^q$ proposed in Theorem 1.1 of \cite{ahle2019oblivious}: the sketch matrices in nodes of the tree labeled with $S_{\text{base}}$ and $T_{\text{base}}$ are independent instances of degree-2 {\sc TensorSketch} and {\sc CountSketch}, respectively.} \label{sketchingtree}
		\end{figure*}
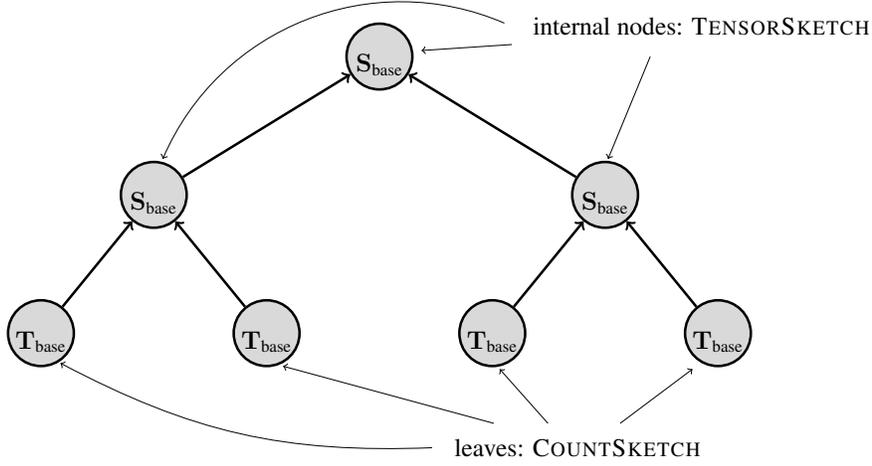
		
		It was shown in \cite{ahle2019oblivious} that the sketch $S^q$ can be represented by a binary tree with $q$ leaves. As shown in Figure~\ref{sketchingtree}, the leaves are independent copies of {\sc CountSketch} and the internal nodes are independent instances of degree-2 {\sc TensorSketch}~\cite{pham2013fast}, which can sketch 2-fold tensor products efficiently.
		The sketch $S^q$ can be applied to tensor product vectors of the form $u_1 \otimes u_2 \otimes \ldots u_q$ by recursive application of $O(q)$ independent instances of {\sc CouuntSketch}~\cite{charikar2002finding} and degree-2 {\sc TensorSketch} \cite{pham2013fast} on vectors $u_i$ and their sketched versions. 
		The use of {\sc CountSketch} in the leaves of this sketch structure ensures input sparsity runtime for sketching sparse input vectors. 
		
		\paragraph{Runtime analysis:} By Theorem 1.1 of \cite{ahle2019oblivious}, for any collection of vectors $u_1, u_2, \ldots u_q \in \RR^d$, $S^q \left( u_1 \otimes u_2 \otimes \ldots u_q \right)$ can be computed in time $O\left( q m \log m + \sum_{j=1}^q {\rm nnz}(u_j) \right)$. 
		From the binary tree structure of the sketch, shown in Figure~\ref{sketchingtree}, it follows that once we compute $S^q \left( u_1 \otimes u_2 \otimes \ldots u_q \right)$, then $S^q \left( e_1 \otimes u_{2} \otimes u_{3} \otimes \ldots u_q \right)$ can be computed by updating the path from one of the leaves to the root of the binary tree. This exactly amounts to applying an instance of {\sc CountSketch} on ${e}_1$ and then applying $O(\log q)$ instances of degree-2 {\sc TensorSketch} on the intermediate nodes of the tree.
		This can be computed in a total additional runtime of $O( m\log m \log q )$. By this argument, it follows that $S^q \left( e_1^{\otimes j} \otimes u_{j+1} \otimes u_{j+2} \ldots u_q \right)$ can be computed sequentially for all $j=0,1,2, \cdots q$ in total time $O\left( q m \log m \log q + \sum_{j=1}^q {\rm nnz}(u_j) \right)$. 
		By plugging in the value $m = O\left( \frac{q}{\varepsilon^2} \right)$, this runtime will be upper bounded by $O\left( \frac{q^2 \log^2 \frac{q}{\varepsilon}}{\varepsilon^2} + \sum_{j=1}^q {\rm nnz}(u_j) \right)$, which gives the second statement of the lemma.

	\end{proofof}

	In order to prove our main result about \SRHT with shared random signs in Lemma~\ref{lem:SRHT-shared-sign}, we use  Khintchine's inequality. We provide a formal statement of this inequality in the following lemma.
	\begin{lemma}[Khintchine's inequality~\cite{haagerup2007best}] \label{lem:khintchine}
		Let $t$ be a positive integer, $x\in \RR^d$, and $(\sigma_i)_{i\in[d]}$
		be independent Rademacher $\pm1$ random variables.
		Then
		\begin{align*}
			\left( E \left[ |\langle \sigma ,  x \rangle|^t \right] \right)^{1/t} \,\le C_t\, \|x\|_2,
		\end{align*}
		where $C_t\le \sqrt2 \left(\frac{\Gamma((t+1)/2)}{\sqrt\pi}\right)^{1/t}\le\sqrt{t}$ for all $t\ge 1$. Consequently, by Minkowski's Inequality along with Markov's inequality, for any $\delta>0$ and any matrix $X \in \RR^{d \times n}$, we have
		\[ \Pr\left[ \left\| X^\top \cdot \sigma \right\|_2 \ge 2\sqrt{ \log_2 \frac{1}{\delta}} \cdot \| X \|_F \right] \le \delta. \]
		
	\end{lemma}
	
	\subsection{Proof of Lemma~\ref{lem:TensorNormDS}}\label{appndx:proof-DS}
	
	Let $P_{i,j}$ be the matrices defined in line~\ref{alg:DS-sketched-dataset} of Algorithm~\ref{alg:TensorNormDS}. For every $V \in \RR^{n \times r}$, we can write,
	\[ P_{i,j} \cdot V =  Q_i \cdot S_i^q \cdot \left( \left( E_1^{\otimes j} \otimes X^{(j+1)} \otimes X^{(j+2)} \otimes \ldots X^{(q)} \right) \cdot V \right), \]
	where $S^q_i$ is an instance of degree-$q$ \PolyS and $Q_i$ is an \SRHT.
	By Lemma~\ref{soda-result} and Lemma~\ref{lem:srht} and a union bound, for every fixed $i \in [T]$ and $j\in \{ 0,1, 2, \ldots q \}$ the following holds,
	\begin{equation}\label{sketched-norm-prob-error}
		\Pr\left[ \|P_{i,j} \cdot V \|_F^2 \in (1\pm \epsilon)\left\| \left( E_1^{\otimes j} \otimes X^{(j+1)} \otimes X^{(j+2)} \otimes \ldots X^{(q)} \right) \cdot V \right\|_F^2 \right] \ge 9/10
	\end{equation}
	Using the properties of tensor products and the definition of matrix $E_1$ we have,
	\[\left\| \left( E_1^{\otimes j} \otimes X^{(j+1)} \otimes X^{(j+2)} \otimes \ldots X^{(q)} \right) \cdot V \right\|_F^2 = \left\| \left( X^{(j+1)} \otimes X^{(j+2)} \otimes \ldots X^{(q)} \right) \cdot V \right\|_F^2\]
	Because $\tilde{z}_j$ is defined as the median over $T = \Omega(\log n)$ independent copies in line~\ref{median-alg} of Algorithm~\ref{alg:TensorNormDS}, using the above equality and \eqref{sketched-norm-prob-error} we have,
	\[\Pr\left[ \tilde{z} \in (1\pm \epsilon)\left\| \left( X^{(j+1)} \otimes X^{(j+2)} \otimes \ldots X^{(q)} \right) \cdot V \right\|_F^2 \right] \ge 1 - \frac{1}{{\rm poly}(n)}.\]
	This proves the first statement of the  lemma.

	\paragraph{Runtime and Memory:} The time to compute $P_{i,j}$ for a fixed $i$ and all $j=0,1, \ldots q$ is $O\left( {\frac{q^2 \log^2\frac{q}{\epsilon}}{\epsilon^2}} \cdot n + \sum_{j=1}^q {\rm nnz}\left(X^{(j)} \right) \right)$, by Lemma~\ref{soda-result} and Lemma~\ref{lem:srht}. Therefore, the total time to compute $P_{i,j}$ for all $i \in [T]$ and all $j=0,1, \ldots q$ is $O\left( {\frac{q^2 \log^2\frac{q}{\epsilon}}{\epsilon^2}} \cdot n \log n + \log n \cdot \sum_{j=1}^q {\rm nnz}\left(X^{(j)} \right) \right)$. Since matrices $P_{i,j}$ are of size $m' \times n$, the total memory needed to store them for all $i$ and $j$ is $O\left( \frac{q \log(1/\epsilon)}{\epsilon^2} \cdot n \log n \right)$. Finally note that the runtime of \textsc{Query}$(V, j)$ is dominated by time needed to compute the product $P_{i,j} \cdot V$ for $i \in [T]$. This can be done in $O\left( \frac{\log(1/\epsilon)}{\epsilon^2} \cdot \log n \cdot {\rm nnz}(V) \right)$ operations.
	
	\section{Spectral Approximation to Tensor Product Matrices $\Phi = \bigotimes_{j=1}^q X^{(j)}$}\label{appndx:proof-rowsampler-distinct-datasets}
	In this section we design the \textsc{RowSampler} procedure which can perform \emph{row norm sampling} as per Definition \ref{def:row-samp} on $\Phi (B^\top B + \lambda I)^{-1/2}$ for $\Phi=\bigotimes_{j=1}^q X^{(j)}$ using $\widetilde{O} \left( \sum_{i}\text{nnz}\left(X^{(i)} \right) \right)$ runtime. 
	Our primitive crucially relies on {\sc TensorNormDS}, given in Algorithm~\ref{alg:TensorNormDS}, to quickly estimate norm queries of the form $\left\| \left(\bigotimes_{j=1}^q X^{(j)}\right) V \right\|_F^2$.

	\begin{algorithm}[!t]
		\caption{\algoname{RowSampler} for $\Phi = \bigotimes_{j=1}^q X^{(j)}$}
		{\bf input}: $q,s \in \mathbb{Z}_+$, $X^{(1)}, \ldots X^{(q)} \in \RR^{d \times n}$, $B \in \RR^{m \times n}$, $\lambda>0$\\
		{\bf output}: Sampling matrix $S \in \RR^{s \times d^q}$
		\begin{algorithmic}[1]
			\STATE{Generate $H\in\RR^{d' \times n}$ with i.i.d. normal entries with $d'= C_1 q \log n$ rows}
			\STATE{$M \gets H \cdot (B^\top B + \lambda I)^{-1/2}$}\alglinelabel{M-alg}
			
			\STATE{Let \textsc{TNorm} be the DS in Algorithm~\ref{alg:TensorNormDS} for inputs $\left( X^{(1)}, X^{(2)}, \ldots , X^{(q)}, M \right)$ and $\epsilon=\frac{1}{20q}$} \alglinelabel{Tnorm-instantiation-rowsampler}

			\STATE{Let $h:[d]\rightarrow[s']$ be a fully independent and uniform hash function with $s' = \lceil {q}^{2}s \rceil$ buckets}
			\STATE{Define the set $h^{-1}(r):=\left\{ j\in[d]:h(j)=r \right\}$ for every $r\in[s']$}
			
			\STATE{For every $r\in[s']$ and $k \in [m']$, let $G_r^k \in \RR^{n'\times d_r}$ be independent instances of degree-$1$ \PolyS as per Lemma~\ref{soda-result}, where $d_r=|h^{-1}(r)|$, $n'=C_2q^2$, and $m' = C_3\log n$}
			
			\STATE{$W_{r, k}^a \gets G_r^k \cdot X^{(a)}_{h^{-1}(r),\star}$ for every $a \in [q]$, $k \in [m']$, and $r\in[s']$}\alglinelabel{W}

			\FOR{$\ell=1$ to $s$}
			
			\STATE{$D^{1} \gets I_n$ and $\beta_\ell \gets s$}\alglinelabel{beta-forloop}
			\FOR{$a=1$ to $q$}

			\STATE{$L^a_{r, k} \gets D^{a} \cdot W_{r, k}^{a\top}$ for every $k \in [m']$, and $r\in[s']$}
			
			\STATE{$p^{a}_r \gets  \textsc{Median}_{k \in [m']} \left\{\textsc{TNorm.Query}( L^a_{r, k} , a )\right\}$ for every $r \in [s']$}
			\STATE{$p^{a}_r \gets  p^{a}_r/ \sum_{t=1}^{s'} p^{a}_t$ for every $r \in [s']$}\alglinelabel{dist-par}
			\STATE{Sample $t \in [s']$ from distribution $\{p^a_r\}_{r=1}^{s'}$}

			\STATE{$q^{a}_i \gets \textsc{TNorm.Query}\left( D^{a} \cdot X^{(a)\top}_{i,\star} , a \right)$ for every $i\in h^{-1}(t)$}
			\STATE{$q^{a}_i \gets  q^{a}_i/ \sum_{j \in h^{-1}(t)} q^{a}_j$ for every $i\in h^{-1}(t)$}\alglinelabel{dist-qai}
			
			\STATE{Sample $i_a \in [d]$ from distribution $\{q^a_i \}_{i\in h^{-1}(t)}$}\alglinelabel{isample-alg}
			
			\STATE{$D^{a+1}\gets D^{a}\cdot {\rm diag}\left( X^{(a)}_{i_a,\star} \right)$}\alglinelabel{def-Da}
			
			\STATE{$\beta_\ell \gets  \beta_\ell \cdot p^{a}_{t} q_{i_a}^{a,t} $} \alglinelabel{beta-update}
			\ENDFOR

			\STATE{Let the $\ell^{th}$ row of $S$ be $ \beta_\ell^{-1/2}  
				\left({e}_{i_1} \otimes {e}_{i_2} \otimes \cdots {e}_{i_q}\right)^\top$}
			\ENDFOR
			\STATE{\textbf{return} $S$}
		\end{algorithmic}
		\label{alg:rotatedrowsampler-poly}
	\end{algorithm}
	
	\paragraph{Overview of Algorithm \ref{alg:rotatedrowsampler-poly}:} The goal of \textsc{RowSampler} is to generate a sample $(i_1,i_2, \cdots i_q) \in [d]^q$ with probability proportional to the squared norm of the row $(i_1, \cdots i_q)$ of matrix $\left( \bigotimes_{j=1}^q X^{(j)} \right) \cdot (B^\top B + \lambda I)^{-1/2}$. Because $(B^\top B + \lambda I)^{-1/2}$ has a large $n \times n$ size, we first compress it using random projection techniques without perturbing the row norm distribution of $\left( \bigotimes_{j=1}^q X^{(j)} \right) \cdot (B^\top B + \lambda I)^{-1/2}$ too much.
	This can be done by applying a JL-transformation to the rows of this matrix (see, e.g., \cite{dasgupta2003elementary}). 
	Let $H \in \RR^{d' \times n}$ be a random matrix with i.i.d. normal entries with $d' = C_1 q \log_2 n$ rows. With probability $1 - \frac{1}{{\rm poly}(n^q)}$ the norm of each row of the sketched matrix $\left( \bigotimes_{j=1}^q X^{(j)} \right) \cdot (B^\top B + \lambda I)^{-1/2} \cdot H^\top$ is preserved up to a $(1\pm 0.1)$ factor and hence by a union bound, with probability $1 - \frac{1}{{\rm poly}(n^q)}$, all row norms of the sketched matrix are within a $(1\pm 0.1)$ factor of the original row norms. This is done in line~\ref{M-alg} of the algorithm by computing $M := H \cdot (B^\top B + \lambda I)^{-1/2}$, which can be computed quickly since matrices $B$ and $H$ have few rows. 
	
	Now the problem is reduced to performing row norm sampling on $\left( \bigotimes_{j=1}^q X^{(j)} \right) \cdot M^\top$. 
	Note that computing the exact row norms of this matrix is out of the question since it has a huge $d^q$ number of rows. However, by using \textsc{TensorNormDS} that we designed in Algorithm~\ref{alg:TensorNormDS} and exploiting the properties of tensor products we can approximately generate samples from the row norm distribution in near input sparsity time as follows:
	
	First note that by basic properties of tensor products, the entries of $\left( X^{(1)} \otimes X^{(2)}  \ldots X^{(q)} \right) \cdot M^\top$ are in bijective correspondence with the entries of $\left(X^{(1)} \otimes M \right) \cdot \left( X^{(2)} \otimes X^{(3)}  \ldots X^{(q)} \right)^\top$. More precisely, the entry at row $(i_1,i_2, \cdots i_q)$ and column $j$ of $\left( X^{(1)} \otimes X^{(2)}  \ldots X^{(q)} \right) \cdot M^\top$ is equal to the entry at row $(i_1,j)$ and column $(i_2, \ldots i_q)$ of $\left( X^{(1)} \otimes M \right) \cdot \left( X^{(2)} \otimes \ldots X^{(q)} \right)^\top$. 
	
	Therefore, it is enough to have a procedure to sample $(i_1,i_2, \ldots i_q)$ with probability proportional to the squared norm of column $(i_2, \ldots i_q)$ of matrix $\left( M \cdot {\rm diag}\left( X^{(1)}_{i_1,\star} \right) \right) \cdot \left( X^{(2)} \otimes \ldots X^{(q)} \right)^\top$ for every $i_1 \in [d]$. 
	To this end, we first sample an index $i_1$ with probability proportional to the squared Frobenius norm of $\left( M \cdot {\rm diag}\left( X^{(1)}_{i_1,\star} \right) \right) \cdot \left( X^{(2)} \otimes \ldots X^{(q)} \right)^\top$, and then perform column norm sampling on the sampled matrix. 
	We can cheaply estimate the Frobenius norms of matrices $\left( M \cdot {\rm diag}\left( X^{(1)}_{i_1,\star} \right) \right) \cdot \left( X^{(2)} \otimes \ldots X^{(q)} \right)^\top$ up to $\left(1\pm \frac{1}{20q}\right)$ perturbation using \textsc{TensorNormDS} (Algorithm~\ref{alg:TensorNormDS}).
	
	After this point, we will have an index $i_1 \in [d]$ sampled from the right distribution and all that is left to do is to carry out row norm sampling on $\left( X^{(2)} \otimes \ldots X^{(q)} \right)\cdot \left( M \cdot {\rm diag}\left( X^{(1)}_{i_1,\star} \right) \right)^\top$. Note that we have made progress because this matrix has $d^{q-1}$ rows, meaning that we have reduced the size of our problem by a factor of $d$. 
	Algorithm~\ref{alg:rotatedrowsampler-poly} recursively repeats this process of reshaping, norm estimation, and sampling $q$ times until having all $q$ indices $i_1,i_2, \cdots i_q$. 
	
	Note that the actual procedure requires more work because we need to generate $s$ i.i.d. samples with the row norm distribution and to ensure that the runtime does not lose a multiplicative factor of $s$, resulting in $s\cdot \sum_{j \in [q]} \text{nnz}\left( X^{(j)} \right)$ total time, we need to do extra sketching and a random partitioning of the rows of the datasets to ${q}^{2} s$ buckets. Moreover, we use the median trick to boost the success probabilities of our randomized operations, when needed.
	
	The formal guarantee on Algorithm~\ref{alg:rotatedrowsampler-poly} is given in the following lemma.
	
	\begin{lemma}\label{lem:rotatedrowsampler-poly}
		For any matrices $X^{(1)}, X^{(2)}, \ldots X^{(q)} \in \RR^{d\times n}$ and $B\in\RR^{m\times n}$, any $\lambda>0$ and any positive integers $q,s$, with probability at least $1 - \frac{1}{{\rm poly}(n)}$, Algorithm~\ref{alg:rotatedrowsampler-poly} outputs a ranks-$s$ row norm sampler for the matrix $\left( X^{(1)} \otimes X^{(2)} \otimes \ldots X^{(q)} \right) \cdot (B^\top B + \lambda I)^{-1/2}$ as per Definition \ref{def:row-samp} in time $O\left( m^2 n + q^{7} s^2 n \log^3 n + \log^3n \log q \sum_{j=1}^q {\rm nnz}\left( X^{(j)} \right) \right)$.
		
	\end{lemma}
	
	\begin{proof}
		
		All rows of the sampling matrix $S\in\RR^{s\times d^q}$ (the output of Algorithm \ref{alg:rotatedrowsampler-poly}) have independent and identical distributions because for each  $\ell \in[s]$, the $\ell^{th}$ row of the matrix $S$ is constructed by sampling indices $i_1,i_2,\cdots i_q$ in line~\ref{isample-alg} completely independent of the sampled values for other rows $\ell'\neq \ell$.
		Thus, it is enough to consider the distribution of the $\ell^{th}$ row of $S$ for some arbitrary $\ell \in[s]$. 
		
		Let $I:=(I_1,I_2, \cdots I_q)$ be a vector-valued random variable that takes values in $[d]^q$ with the following conditional probability distribution for every $a=1,2, \cdots q$ and every $i \in [d]$,
		\begin{equation}\label{eq:def-ia-dist}
			\Pr\left[ I_a =i | I_1=i_1, I_2=i_2, \cdots I_{a-1}=i_{a-1} \right] := p^a_{h(i)} \cdot q^{a}_{i},
		\end{equation}
		where distributions $\{p^a_r\}_{r \in [s']}$ and $\{q^{a}_i\}_{i \in h^{-1}(t)}$ for every $t \in [s']$ are defined as per lines~\ref{dist-par} and \ref{dist-qai} of the algorithm.
		One can see that the random vector $(i_1,i_2, \cdots i_q)$ obtained by stitching together the random indices sampled in line~\ref{isample-alg} of the algorithm, is in fact a copy of the random variable $I$ defined above.

		Let $\beta_\ell$ be the quantity computed in line \ref{beta-update} of the algorithm. If $i_1,i_2, \cdots i_q \in[d]$ are the indices sampled in line~\ref{isample-alg} of the algorithm, then using the conditional distribution of $I$ in \eqref{eq:def-ia-dist}, we find that the value of $\beta_\ell$ is equal to the following,
		\begin{align*}
			\beta_\ell &= s \cdot \prod_{a=1}^q p^a_{h(i_a)} q^{a}_{i_a} \\
			&= s \cdot \prod_{a=1}^q \Pr \left[ I_a =i_a | I_1=i_1, I_2=i_2, \cdots I_{a-1}=i_{a-1} \right]\\
			&= s \cdot \Pr \left[ I =(i_1, i_2, \ldots i_q) \right],
		\end{align*}
		where $p^a$ and $q^{a}$ are the distributions computed in lines~\ref{dist-par} and \ref{dist-qai} of the algorithm.
		Hence, for any $i_1,i_2, \cdots i_q \in [d]$, the distribution of $S_{\ell,\star}$ is,
		\begin{align}
			&\Pr\left[ S_{\ell,\star}= \beta_\ell^{-1/2} ({ e}_{i_1} \otimes { e}_{i_2} \otimes \cdots { e}_{i_q})^\top \right]\nonumber\\
			&\qquad= \Pr \left[ I =(i_1, i_2, \ldots i_q) \right] = \frac{\beta_\ell}{s}. \label{dist-sampling-matrix-poly}
		\end{align}
		We will use \eqref{dist-sampling-matrix-poly} later.

		By Lemma~\ref{lem:TensorNormDS} and the way $\textsc{TNorm}$ is constructed in line~\ref{Tnorm-instantiation-rowsampler} of the algorithm, we have the following inequalities for any $r\in[s']$, $k \in [m']$, $i \in [d]$, and any $a=1,2, \ldots q$ , with probability at least $1 - \frac{1}{{\rm poly}(n)}$,
		\begin{align}
			&\textsc{TNorm.Query}\left( L^a_{r,k} , a \right) \in \left( 1\pm \frac{1}{20q} \right) \left\| \left(X^{(a+1)} \otimes \ldots X^{(q)} \otimes M \right) D^{a} W_{r, k}^{a\top}  \right\|_F^2,\label{tensorsketch}\\
			&\textsc{TNorm.Query}\left( D^{a} \cdot X^{(a)\top}_{i,\star} , a \right) \in \left( 1\pm \frac{1}{20q} \right) \cdot \left\| \left(X^{(a+1)} \otimes \ldots X^{(q)} \otimes M \right) D^{a} \cdot X^{(a)\top}_{i,\star} \right\|_2^2. \label{Tnorm-xi}
		\end{align}
		By union bounding over $qds'm'$ events, with probability at least $1 - \frac{1}{{\rm poly}(n)}$, \eqref{tensorsketch} and \eqref{Tnorm-xi} hold simultaneously for all $a\in[q]$, $k \in [m']$, $i\in[d]$, and all $r\in[s']$.

		Furthermore, note that $W_{r,k}^a$ is defined in line~\ref{W} as $W_{r,k}^a=G_r^k \cdot X^{(a)}_{h^{-1}(r),\star}$, where $G_r^k$ is an instance of the degree-$1$ \PolyS as per Lemma~\ref{soda-result} with target dimension $n'=C_2q^2$. 
		By the first statement of Lemma~\ref{soda-result}, the \PolyS $G_r^k$ approximately preserves the Frobenius norm of any fixed matrix with constant probability. In particular, for every $a\in[q], r\in[s'], k \in [m']$, with probability at least $9/10$ the following holds,
		\begin{align}
			\left\| \left(X^{(a+1)} \otimes \ldots X^{(q)} \otimes M \right) D^{a} W_{r,k}^{a\top} \right\|_F^2 \in \left(1\pm \frac{1}{50q}\right) \left\| \left(X^{(a+1)} \otimes \ldots X^{(q)} \otimes M \right) D^{a} \left( X^{(a)}_{h^{-1}(r),\star} \right)^\top \right\|_F^2. \label{jl-poly}
		\end{align}
		By taking the median of $m' = \Omega(\log n)$ independent instances of $G_r^k$, the success probability in \eqref{jl-poly} gets boosted. Thus, by combining this inequality with \eqref{tensorsketch} using a union bound, and applying the median trick, with probability at least $1 - \frac{1}{{\rm poly}(n)}$ the following holds simultaneously for all $a \in [q]$ and $r \in [s']$,
		\begin{align}
			\textsc{Median}_{k \in [m']} \left\{\textsc{TNorm.Query}\left( L^{a}_{r, k} , a \right)\right\} \in \left( 1 \pm \frac{1}{14q} \right) \left\| \left(X^{(a+1)} \otimes \ldots X^{(q)} \right) D^{a} \left( X^{(a)}_{h^{-1}(r),\star} \otimes M \right)^\top \right\|_F^2 \label{jl-countsketch-median}
		\end{align}
		Note that to obtain the above inequality we used the property of tensor products regarding the bijective correspondence between entries of $\left(X^{(a+1)} \otimes \ldots X^{(q)} \otimes M \right) D^{a} \left( X^{(a)}_{h^{-1}(r),\star} \right)^\top$ and $\left(X^{(a+1)} \otimes \ldots X^{(q)} \right) D^{a} \left( X^{(a)}_{h^{-1}(r),\star} \otimes M \right)^\top$
		
		By plugging the above inequality along with \eqref{Tnorm-xi} into \eqref{eq:def-ia-dist}, we conclude that with high probability the following bound holds simultaneously for all $a \in [q]$,
		\begin{align}
			&\Pr[I_a =i | I_1=i_1, I_2=i_2, \cdots I_{a-1}=i_{a-1}]\nonumber\\ 
			&\qquad\ge \left( 1-\frac{1}{5q} \right) \cdot \frac{ \left\| \left(X^{(a+1)} \otimes \ldots X^{(q)} \right) D^{a} {\rm diag} \left(X^{(a)}_{i,\star} \right) M^\top \right\|_F^2 }{ \left\| \left(X^{(a+1)} \otimes \ldots X^{(q)} \right) D^{a} \left( X^{(a)} \otimes M \right)^\top \right\|_F^2}.\label{lower-bnd-Ij}
		\end{align}
		Again to obtain the above inequality we used the property of tensor products regarding the bijective correspondence between the entries of vector $\left(X^{(a+1)} \otimes \ldots X^{(q)} \otimes M \right) D^{a} \cdot X^{(a)\top}_{i,\star}$ and matrix $\left(X^{(a+1)} \otimes \ldots X^{(q)} \right) D^{a} {\rm diag} \left( X^{(a)}_{i,\star}  \right) M^\top$
		
		It follows from the properties of tensor products and the definition of $D^a$ in line~\ref{def-Da} of the algorithm, that
		\begin{align*}
			\left\| \left(X^{(a+2)} \otimes \ldots X^{(q)} \right) D^{a+1}  \left( X^{(a+1)} \otimes M \right)^\top \right\|_F^2
			&=\left\| \left(X^{(a+2)} \otimes \ldots X^{(q)} \right) D^{a} {\rm diag}\left( X^{(a)}_{i_a,\star} \right)  \left( X^{(a+1)} \otimes M \right)^\top \right\|_F^2\\
			&= \left\| \left(X^{(a+1)} \otimes X^{(a+2)} \otimes \ldots X^{(q)} \right) D^{a} {\rm diag}\left( X^{(a)}_{i_a,\star} \right) M^\top \right\|_2^2
		\end{align*}
		Using this equality and inequality \eqref{lower-bnd-Ij}, we have:
		\begin{align}
			\Pr\left[ I=(i_1 , i_2, \cdots i_q) \right] &= \prod_{a=1}^{q} \Pr \left[I_a =i_a | I_1=i_1, \cdots I_{a-1}=i_{a-1} \right]\nonumber\\
			&\ge \prod_{a=1}^{q} \left(1-\frac{1}{5q}\right) \frac{ \left\| \left(X^{(a+1)} \otimes \ldots X^{(q)} \right) D^{a} \cdot {\rm diag} \left(X^{(a)}_{i_a,\star} \right) M^\top \right\|_F^2 }{ \left\| \left(X^{(a+1)} \otimes \ldots X^{(q)} \right) D^{a} \left( X^{(a)} \otimes M \right)^\top \right\|_F^2}\nonumber\\
			& \ge \frac{3}{4} \cdot \frac{ \left\| \mathbf{1}_n^\top \cdot D^{q} \cdot {\rm diag}\left( X_{i_q,\star}^{(q)} \right) M^\top \right\|_F^2}{ \left\| \left(X^{(2)} \otimes \ldots X^{(q)} \right) D^{1}  \left( X^{(1)} \otimes M \right)^\top \right\|_F^2}\nonumber\\ 
			&= \frac{3}{4} \cdot \frac{ \left\| \left[ \left(X^{(1)} \otimes X^{(2)} \otimes \ldots X^{(q)} \right) \cdot M^\top \right]_{(i_1,i_2,\cdots i_q),\star} \right\|_2^2}{ \left\| \left(X^{(1)} \otimes X^{(2)} \otimes \ldots X^{(q)} \right) \cdot M^\top \right\|_F^2} \label{conditional-prob-bound}
		\end{align}
		By plugging \eqref{conditional-prob-bound} back in 
		\eqref{dist-sampling-matrix-poly} we find that,
		\begin{align*}
			\Pr\left[ S_{\ell,\star}= \beta_\ell^{-1/2} ({ e}_{i_1} \otimes { e}_{i_2} \otimes \cdots { e}_{i_q})^\top \right] \ge \frac{3}{4} \cdot \frac{ \left\| \left[ \left(X^{(1)} \otimes \ldots X^{(q)} \right) \cdot M^\top \right]_{(i_1,i_2,\cdots i_q),\star} \right\|^2}{ \left\| \left(X^{(1)} \otimes \ldots X^{(q)} \right) \cdot M^\top \right\|_F^2 }
		\end{align*}
		
		Matrix $M$ is defined as $M=H \cdot (B^\top B +\lambda I)^{-1/2}$ where $H$ is a random matrix with i.i.d. Gaussian entries with $d'=C_1q\log n$ rows. Therefore, $H$ is a JL-transform, so for every $(i_1,i_2,\cdots i_q) \in[d]^q$, with probability $1-\frac{1}{\text{poly}(n^q)}$,
		\begin{align*}
			({d'})^{-1}{\left\| \left[ \left(X^{(1)} \otimes \ldots X^{(q)} \right) \cdot M^\top \right]_{(i_1,i_2,\cdots i_q),\star} \right\|_2^2}  \in \left(1\pm 0.1 \right) \left\| \left[ X^{(1)} \otimes \ldots X^{(q)} \right]_{(i_1,i_2,\cdots i_q),\star}  (B^\top B +\lambda I)^{-1/2}\right\|_2^2.
		\end{align*}
		Therefore, by union bounding over $d^q$ rows of $\left(X^{(1)} \otimes \ldots X^{(q)} \right) \cdot M^\top$, the above holds simultaneously for all $(i_1,i_2,\cdots i_q) \in[d]^q$ with probability $1 - \frac{1}{{\rm poly}(n^q)}$. Therefore, with high probability in $n$,
		\begin{align*}
			&\Pr\left[ S_{\ell,\star}= \beta_\ell^{-1/2} ({ e}_{i_1} \otimes { e}_{i_2} \otimes \cdots { e}_{i_q})^\top \right]\\
			&\qquad\ge \frac{1}{2} \cdot \frac{ \left\| \left[ \left(X^{(1)} \otimes \ldots X^{(q)} \right) \cdot (B^\top B +\lambda I)^{-1/2} \right]_{(i_1,i_2,\cdots i_q),\star} \right\|_2^2}{\left\| \left(X^{(1)} \otimes \ldots X^{(q)} \right) \cdot (B^\top B +\lambda I)^{-1/2} \right\|_F^2}
		\end{align*}
		Because $\frac{\beta_\ell}{s}$ is the probability of sampling row $(i_1,i_2,\cdots i_q)$ of $\left(X^{(1)} \otimes \ldots X^{(q)} \right) \cdot (B^\top B +\lambda I)^{-1/2}$, the above inequality proves that with high probability, matrix $S$ is a rank-$s$ row norm sampler for $\left(X^{(1)} \otimes \ldots X^{(q)} \right) \cdot (B^\top B +\lambda I)^{-1/2}$ as in Definition \ref{def:row-samp}.

		\paragraph{Runtime:} One of the expensive steps of this algorithm is the computation of $M$ in line~\ref{M-alg} which takes $O(m^2 n + q m n \log n)$ operations since $B$ has rank at most $m$. 
		Another expensive step is the computation of the \textsc{TNorm} data-structure in line~\ref{Tnorm-instantiation-rowsampler}. By Lemma \ref{lem:TensorNormDS}, this DS for $\epsilon = \frac{1}{20q}$ can be formed in time $O\left( q^4 \log^2 q \cdot n \log n + \log n \cdot \sum_{j=1}^q \text{nnz}\left( X^{(j)} \right) \right)$.

		By Lemma~\ref{soda-result}, matrices $W^a_{r,k}$ for all $r\in[s']$, $k \in [m']$ and $a \in [q]$ in line~\ref{W} of the algorithm can be computed in total time $O\left( q^3 s' n\log^2 n + \log n \cdot \sum_{j=1}^q \text{nnz}\left( X^{(j)} \right) \right)$.
		
		The matrix $W_{r, k}^{a}$ for every $k \in [m']$, and $r\in[s']$, has size $O(q^2) \times n$.
		Thus, by Lemma~\ref{lem:TensorNormDS}, computing the distribution $\{ p^a_r \}_{r=1}^{s'}$ in line~\ref{dist-par} takes time $O\left(q^{4} s' \cdot n \log^2n \log q \right)$ for a fixed $a\in[q]$ and a fixed $\ell \in[s]$.  Therefore, the total time to compute this distribution for all $a$ and $\ell$ is $O\left(q^{7} s^2 \cdot n \log^2 n\log q \right)$.

		The runtime of computing the distribution $\{ q^{a}_i \}_{i\in h^{-1}(t)}$ in line~\ref{dist-qai} depends on the sparsity of $X^{(a)}_{h^{-1}(t),\star}$, i.e.,  $\text{nnz}\left( X^{(a)}_{h^{-1}(t),\star} \right)$. To bound the sparsity of $X^{(a)}_{h^{-1}(t),\star}$, note that, $\text{nnz}\left( X^{(a)}_{h^{-1}(t),\star} \right) = \sum_{i=1}^d \mathbbm{1}_{\{i \in h^{-1}(t)\}} \cdot \text{nnz}\left( X^{(a)}_{i,\star} \right)$.
		Since the hash function $h$ is fully independent, by invoking Bernstein's inequality, we find that for every $t\in[s']$ and $a \in [q]$,  with high probability in $n$, ${\rm nnz}\left( X^{(a)}_{h^{-1}(t),\star} \right) = O\left( \left( {{\rm nnz}\left( X^{(a)} \right)/s'} + n \right) \log n \right)$. 
		By union bounding over $qs'$ events, with high probability in $n$, ${\rm nnz}\left( X^{(a)}_{h^{-1}(t),\star} \right) = O\left( \left( {{\rm nnz}\left( X^{(a)} \right)/s'} + n \right) \log n \right)$, simultaneously for all $t\in[s']$ and $a \in [q]$.
		
		Therefore, by Lemma~\ref{lem:TensorNormDS}, the distribution $\{ q^{a}_i \}_{i\in h^{-1}(t)}$ in line~\ref{dist-qai} of the algorithm can be computed in total time $O\left( q^3 s n \log^3 n \log q + \log^3 n \log q\cdot \sum_{j=1}^q {\rm nnz}\left( X^{(j)} \right) \right)$ for all $a \in [q]$ and all $\ell \in [s]$.

		The total runtime of Algorithm~\ref{alg:rotatedrowsampler-poly} is thus $O\left( m^2 n + q^{7} s^2 n \log^2 n \log q + \log^3n \log q \cdot \sum_{j=1}^q {\rm nnz}\left( X^{(j)} \right) \right)$.
	\end{proof}
	
	Now we can prove our main theorem about spectrally approximating the Gram matrix $\Phi^\top \Phi$ for matrices of the form $\Phi=\bigotimes_{j=1}^q X^{(j)}$ using nearly $\sum_{i}\text{nnz}\left(X^{(i)} \right)$ runtime.

	\begin{proofof}{Theorem~\ref{main-thrm-tensor-prod}}
		The theorem follows by invoking Lemmas~\ref{resursive-rlss-lem} and \ref{lem:rotatedrowsampler-poly}. To find the sampling matrix $\Pi$, run Algorithm~\ref{alg:outerloop} on $\Phi$ with $\mu=s_\lambda$ and for the \textsc{RowSampler} primitive, invoke Algorithm~\ref{alg:rotatedrowsampler-poly}. 
		By Lemma~\ref{lem:rotatedrowsampler-poly}, Algorithm~\ref{alg:rotatedrowsampler-poly} outputs a row norm sampler as per Definition~\ref{def:row-samp} with probability $1 - \frac{1}{{\rm poly}(n)}$. Therefore, since the total number of times Algorithm~\ref{alg:rotatedrowsampler-poly} is invoked by Algorithm~\ref{alg:outerloop} is $\log \frac{\| \Phi \|_F^2}{\epsilon\lambda} = O(\log n)$, by a union bound,
		the preconditions of Lemma~\ref{resursive-rlss-lem} are satisfied with high probability. Thus, it follows that $\Pi$ satisfies the following spectral approximation guarantee
		\[ \frac{\Phi^\top \Phi + \lambda I}{1+\epsilon} \preceq \Phi^\top \Pi^\top \Pi \Phi + \lambda I \preceq \frac{\Phi^\top \Phi + \lambda I}{1-\epsilon}.\]
		Algorithm~\ref{alg:outerloop} invokes the \textsc{RowSampler} primitive $\log \frac{\| \Phi \|_F^2}{\epsilon \lambda} = O(\log n)$ times. Thus, by Lemma~\ref{lem:rotatedrowsampler-poly}, the runtime of finding $\Pi$ is $O\left( \frac{q^7 \cdot s_\lambda^2 \cdot n}{\epsilon^4} \log^5 n \log q + \log^4 n \log q \cdot \sum_{i}\text{nnz}\left(X^{(i)} \right) \right)$.
		
	\end{proofof}

	\section{Proof of Lemma~\ref{lem:SRHT-shared-sign}}
	\label{appndx:srht-shared-sign}
	First, by properties of tensor products and using the definitions of sketch matrices $S^{(c)} = \frac{1}{\sqrt{m}} \cdot P_c H D$, we obtain
	\begin{equation}\label{eq:norm-tensor-SX-expansion}
		\left( S^{(1)} X \right) \otimes \left( S^{(2)} X \right) \otimes \ldots \left( S^{(q)} X \right) = \frac{1}{m^{q/2}} \cdot \left( P_1 \times P_2 \times \ldots P_q \right) \cdot \left( H D X \right)^{\otimes q},
	\end{equation}
	where $P_1 \times P_2 \times \ldots P_q$ denotes the Kronecker product of the sampling matrices $P_1, P_2, \ldots P_q$ and is of size $m^q \times d^q$.
	Now let $x_1, x_2, \ldots x_n \in \RR^d$ denote the columns of $X$. By Khintchine's inequality (Lemma~\ref{lem:khintchine}) along with a union bound over the $d$ entries of the vector $HD x_\ell$, the following holds with probability $1 - \frac{1}{{\rm poly}(n)}$, for every $\ell \in [n]$:
	\[ \left\| HD \cdot x_\ell \right\|_\infty^2 \le O \left( \log n \right) \cdot \|x_\ell\|_2^2. \]
	Therefore, using the definition of tensor product, the following holds with probability $1 - \frac{1}{{\rm poly}(n)}$, simultaneously for all $\ell \in [n]$ and all $r \in [q]$
	\begin{equation}\label{eq:flatness-single-col-SRHT}
		\left\| (HD \cdot x_\ell)^{\otimes r} \right\|_\infty^2 \le O \left( \log n \right)^r \cdot \left\| x_\ell^{\otimes r} \right\|_2^{2}.
	\end{equation}
	From now on we condition on the above inequality holding for every $r \in [q]$ and every $\ell \in [n]$.
	
	Now let us consider the matrix $\left(H D X\right)^{\otimes q} \cdot (\Sigma \otimes K)^\top$. This matrix has $d^q$ rows and $n$ columns. If we let $\lambda_{\min}$ be the smallest eigenvalue of $K$, then using the properties of the tensor product of matrices, the Frobenius norm of this matrix satisfies the following inequality,
	\begin{align}
		\left\| \left( H D X\right)^{\otimes q} \cdot \left( \Sigma \otimes K \right)^\top \right\|_F^2 &= \left\| \left( \Sigma \otimes \left( H D X\right)^{\otimes q} \right) \cdot K^\top \right\|_F^2 \nonumber\\
		&\ge \lambda_{\min}^2 \cdot \left\| \Sigma \otimes \left( H D X\right)^{\otimes q} \right\|_F^2= d^q \cdot \lambda_{\min}^2 \cdot \left\| \Sigma \otimes X^{\otimes q} \right\|_F^2\label{eq:norm-srht-lower-bound}
	\end{align}
	Furthermore, if we let $\lambda_{\max}$ be the largest eigenvalue of $K$, then for any row $\jj \in [d]^q$ of the matrix $\left( H D X\right)^{\otimes q} \cdot (\Sigma \otimes  K)^\top$, the following upper bound holds,
	\begin{align*}
		\left\| \left[ \left( H D X\right)^{\otimes q} \cdot \left( \Sigma \otimes K \right)^\top \right]_{\jj , \star} \right\|_2^2 &= \left\| \Sigma \cdot {\rm diag}\left(\left[ \left( H D X\right)^{\otimes q} \right]_{\jj , \star}\right) \cdot K^\top \right\|_F^2 \\
		&\le \lambda_{\max}^2 \cdot \left\| \Sigma \cdot {\rm diag}\left(\left[ \left( H D X\right)^{\otimes q} \right]_{\jj , \star}\right) \right\|_F^2\\ 
		&= \lambda_{\max}^2 \cdot \sum_{\ell=1}^n \left| \left[ \left( H D \cdot x_\ell \right)^{\otimes q} \right](\jj) \right|^2 \cdot \| \Sigma_{\star,\ell} \|_2^2
	\end{align*}
	By incorporating \eqref{eq:flatness-single-col-SRHT} into the above inequality for $r=q$, we find that for any $\jj \in [d]^q$,
	\[ \left\| \left[ \left( H D X\right)^{\otimes q} (\Sigma \otimes K)^\top \right]_{\jj , \star} \right\|_2^2 \le O \left( \log n \right)^q \lambda_{\max}^2 \sum_{\ell=1}^n \left\| x_\ell^{\otimes q} \right\|_2^{2} \| \Sigma_{\star,\ell} \|_2^2 = O \left( \log n \right)^q \cdot \lambda_{\max}^2 \left\| \Sigma \otimes X^{\otimes q} \right\|_F^2 \]
	In fact, we can prove a stronger version of the above inequality which will turn out to be very useful in our analysis. Let $\jj \in [d]^q$ be some arbitrary index vector. Also, let $S \subseteq [q]$ be some arbitrary subset. Let us denote the subset of indices in $[d]^q$ that agree with $\jj$ on $S$ by $[d]^q_{\jj_S}$ and formally define it as follows:
	\[ [d]^q_{\jj_S} := \{ \ii \in [d]^q : \ii_t = \jj_t \text{ for all } t \in S \}. \]
	Using this notation along with the properties of tensor products and \eqref{eq:flatness-single-col-SRHT} we have the following for every $\jj \in [d]^q$ and $S \subseteq [q]$,
	\begin{align*}
		\sum_{\ii \in [d]^q_{\jj_S}} \left\| \left[ \left( H D X\right)^{\otimes q} \cdot (\Sigma \otimes K)^\top \right]_{\ii , \star} \right\|_2^2 &= \left\| \left( H D X\right)^{\otimes (q - |S|)} \cdot \prod_{t \in S} {\rm diag}\left( [H D X]_{\jj_t, \star} \right) \cdot (\Sigma \otimes K )^\top \right\|_F^2\\
		&\le \lambda_{\max}^2 \cdot \left\| \left( \Sigma \otimes (H D X)^{\otimes (q - |S|)} \right) \prod_{t \in S} {\rm diag}\left( [H D X]_{\jj_t, \star} \right) \right\|_F^2 \\
		&= \lambda_{\max}^2 \cdot \sum_{\ell \in [n]} \left\| \left(\Sigma_{\star,\ell} \otimes ( H D \cdot x_\ell )^{\otimes (q - |S|)}\right) \cdot \prod_{t \in S} [H D \cdot x_\ell](\jj_t) \right\|_2^2 \\
		&= \lambda_{\max}^2 \cdot \sum_{\ell \in [n]} \left\| \Sigma_{\star,\ell} \otimes ( H D \cdot x_\ell )^{\otimes (q - |S|)} \right\|_2^2 \cdot \prod_{t \in S} \left| [H D \cdot x_\ell](\jj_t) \right|^2 \\
		&\le \lambda_{\max}^2 \cdot d^{q - |S|} \cdot \sum_{\ell \in [n]} \left\| \Sigma_{\star,\ell} \otimes x_\ell^{\otimes (q - |S|)} \right\|_2^2 \cdot \prod_{t \in S} O(\log n) \cdot \left\| x_\ell \right\|_2^2 \\
		&= O(\log n)^{|S|} \cdot  \lambda_{\max}^2 \cdot d^{q - |S|} \cdot \left\| \Sigma \otimes X^{\otimes q} \right\|_F^2,
	\end{align*}
	where the fifth line above follows from \eqref{eq:flatness-single-col-SRHT} for $r=1$.
	Now by combining the above with \eqref{eq:norm-srht-lower-bound} we find the following for every non-empty set $S \subseteq [q]$,
	\begin{equation}\label{flatness-set-S}
		\max_{\jj \in [d]^q} \left\{ \sum_{\ii \in [d]^q_{\jj_S}} \left\| \left[ \left( H D X\right)^{\otimes q} \cdot (\Sigma \otimes K)^\top \right]_{\ii , \star} \right\|_2^2 \right\} \le O \left( \frac{\log n}{d}\right)^{|S|} \cdot \kappa^2 \cdot \left\| \left( H D X\right)^{\otimes q} \cdot (\Sigma \otimes K )^\top \right\|_F^2,
	\end{equation}
	where $\kappa = \frac{\lambda_{\max}}{\lambda_{\min}}$ is the condition number of $K$. This inequality shows that the rows of $ \left( H D X\right)^{\otimes q} \cdot (\Sigma \otimes K)^\top$ are ``flat'' and the Frobenius norm of this matrix is spread-out evenly over the rows of this matrix. In addition to \eqref{flatness-set-S}, we can prove a stronger inequality for the case of sets of cardinality one. Specifically, we prove a stronger version of \eqref{flatness-set-S} for any singleton set $S$, i.e., $|S|=1$. We start by denoting the sole element of set $S$ by $\tilde{s}$, i.e., $S = \{ \tilde{s} \}$. So when $S = \{ \tilde{s} \}$, using the definition of $[d]^q_{\jj_S}$ we have $[d]^q_{\jj_S} = \{ \ii \in [d]^q : \ii_{\tilde{s}} = \jj_{\tilde{s}} \}$. Therefore, by properties of tensor products, we can write for any $\jj_{\tilde{s}} \in [d]$:
	\begin{align*}
		\sum_{\ii \in [d]^q_{\jj_S}} \left\| \left[ \left( H D X\right)^{\otimes q} \cdot (\Sigma \otimes K)^\top \right]_{\ii , \star} \right\|_2^2 &= \left\| \left( H D X\right)^{\otimes (q-1)} \cdot {\rm diag}\left( [HDX]_{\jj_{\tilde{s}},\star} \right) (\Sigma \otimes K)^\top \right\|_F^2 \\
		&= d^{q-1} \cdot \left\| X^{\otimes (q-1)} \cdot {\rm diag}\left( [HDX]_{\jj_{\tilde{s}},\star} \right) (\Sigma \otimes K)^\top \right\|_F^2\\
		&= d^{q-1} \cdot \left\| \left[ HDX \cdot \left( \Sigma \otimes K \otimes X^{\otimes (q-1)} \right)^\top \right]_{\jj_{\tilde{s}},\star} \right\|_2^2.
	\end{align*}
	Using the above inequality along with Khintchine's inequality from Lemma~\ref{lem:khintchine}, we find that the following holds for any $S = \{ \tilde{s} \}$, with probability at least $1 - \frac{1}{{\rm poly}(n)}$,
	\begin{align*}
		\sum_{\ii \in [d]^q_{\jj_S}} \left\| \left[ \left( H D X\right)^{\otimes q} \cdot (\Sigma \otimes K)^\top \right]_{\ii , \star} \right\|_2^2 &\le O(\log n) \cdot d^{q-1} \cdot \left\| X \cdot \left( \Sigma \otimes K \otimes X^{\otimes (q-1)} \right)^\top \right\|_F^2\\
		&= O(\log n) \cdot d^{q-1} \cdot \left\| \left( X \otimes X^{\otimes (q-1)} \right) (\Sigma \otimes K)^\top \right\|_F^2\\
		&= O\left( \frac{\log n}{d} \right) \cdot \left\|  (HDX)^{\otimes q} \cdot (\Sigma \otimes K)^\top \right\|_F^2
	\end{align*}
	Now using the above inequality and union bounding over all $\tilde{s} \in [d]$ and $\jj_{\tilde{s}} \in [d]$, we can conclude that with probability at least $1 - \frac{1}{{\rm poly}(n)}$, the following holds simultaneously for all singleton sets $S = \{ \tilde{s} \} \subseteq [q]$ and all $\jj \in [d]^q$,
	\begin{equation}\label{flatness-set-S-singletone}
		\max_{\jj \in [d]^q} \left\{ \sum_{\ii \in [d]^q_{\jj_S}} \left\| \left[ \left( H D X\right)^{\otimes q} \cdot (\Sigma \otimes K)^\top \right]_{\ii , \star} \right\|_2^2 \right\} \le O \left( \frac{\log n}{d}\right) \cdot \left\| \left( H D X\right)^{\otimes q} \cdot (\Sigma \otimes K )^\top \right\|_F^2,
	\end{equation}
	which is a stronger upper bound than \eqref{flatness-set-S} by a factor of $\kappa^2$.

	Now recall that, by \eqref{eq:norm-tensor-SX-expansion}, we have the following,
	\[ \left\| \left( S^{(1)} X \right) \otimes \ldots \left( S^{(q)} X \right) \cdot (\Sigma \otimes K)^\top \right\|_F^2 = \frac{1}{m^{q}} \cdot \left\| \left( P_1 \times P_2 \times \ldots P_q \right) \cdot \left( H D X\right)^{\otimes q} \cdot (\Sigma \otimes K)^\top \right\|_F^2 \]
	Therefore, to simplify the notation, if we denote the vector corresponding to row norms of $ \left( H D X\right)^{\otimes q} \cdot (\Sigma \otimes K)^\top$ by $y \in \RR^{d^q}$,
	\[y_{\jj} := \left\| \left[ \left( H D X\right)^{\otimes q} \cdot (\Sigma \otimes K)^\top \right]_{\jj,\star} \right\|_2~~~~~ \text{ for every } \jj \in [d]^q, \]
	then it suffices to prove that 
	\begin{equation}\label{eq:srht-sampling-error-simplified-version}
		\Pr_{P_1,\ldots P_q} \left[\frac{1}{m^{q}} \cdot \left\| \left( P_1 \times P_2 \times \ldots P_q \right) \cdot y \right\|_2^2 \in (1 \pm \epsilon) \frac{\| y \|_2^2}{d^q} \right] \ge 1 - \delta 
	\end{equation}
	given the fact that the $P_i$ are independent random sampling matrices and conditioned on $y$ satisfying the following flatness property for any non-empty set $S \subseteq [q]$ (by combining \eqref{flatness-set-S} and \eqref{flatness-set-S-singletone}):
	\begin{equation}\label{eq:flatness-row-norms-vector}
		\max_{\jj \in [d]^q} \sum_{\ii \in [d]^q_{\jj_S}}  \left| y_{\ii} \right|^2 \le O \left( \frac{\log n}{d}\right)^{|S|} \cdot \left( \kappa^2 \cdot \mathbbm{1}_{\{ |S|>1 \}} + \mathbbm{1}_{\{ |S|=1 \}} \right) \cdot \left\| y \right\|_2^2. 
	\end{equation}
	In order to prove \eqref{eq:srht-sampling-error-simplified-version}, first note that $\frac{1}{m^{q}} \cdot \left\| \left( P_1 \times P_2 \times \ldots P_q \right) \cdot y \right\|_2^2$ is an unbiased estimator, i.e.,
	\begin{align*}
		\EE_{P_1, \ldots P_q} \left[ \frac{1}{m^{q}} \cdot \left\| \left( P_1 \times \ldots P_q \right) \cdot y \right\|_2^2 \right] &= \frac{1}{m^q} \cdot \sum_{i_1,i_2, \ldots i_q \in [d]} \Pr[i_1 \in P_1] \cdot \ldots \Pr[i_q \in P_q] \cdot |y_{(i_1, \ldots i_q)}|^2\\
		&= \frac{1}{d^q} \cdot \sum_{i_1,i_2, \ldots i_q \in [d]} |y_{(i_1, \ldots i_q)}|^2\\
		&= \frac{\| y \|_2^2}{d^q},
	\end{align*}
	where by $\Pr[i_c \in P_c]$ we mean the probability that $i_c$ is sampled by matrix $P_c$, and this quantity is equal to $\Pr[i_c \in P_c] \equiv \frac{m}{d}$.
	Next we bound the variance of this estimator and then finish the proof by Chebyshev's inequality.
	\begin{align}
		&\EE_{P_1, \ldots P_q} \left[ \left(\frac{1}{m^{q}} \cdot \left\| \left( P_1 \times P_2 \times \ldots P_q \right) \cdot y \right\|_2^2 \right)^2 \right]\nonumber \\ 
		&\qquad = \frac{1}{m^{2q}} \sum_{\jj, \ii \in [d]^q } \Pr[\ii_1, \jj_1 \in P_1] \cdot \ldots  \Pr[\ii_q, \jj_q \in P_q] \cdot |y_{\ii}|^2 \cdot |y_{\jj}|^2\nonumber\\
		&\qquad = \frac{1}{m^{2q}} \sum_{S \subseteq [q]} \sum_{\jj \in [d]^q } \sum_{\ii \in [d]^q_{\jj_S}} \left(\prod_{t \in [q] \setminus S} \mathbbm{1}_{\{ \jj_t \neq \ii_t \}}\right) \cdot \Pr[\ii_1, \jj_1 \in P_1] \cdot \ldots  \Pr[\ii_q, \jj_q \in P_q] \cdot |y_{\ii}|^2 \cdot |y_{\jj}|^2\nonumber\\
		&\qquad = \frac{1}{m^{2q}} \sum_{S \subseteq [q]} \sum_{\jj \in [d]^q } \sum_{\ii \in [d]^q_{\jj_S}} \left(\prod_{t \in [q] \setminus S} \mathbbm{1}_{\{ \jj_t \neq \ii_t \}} \cdot \Pr[\ii_t, \jj_t \in P_t] \right) \cdot \left(\prod_{t' \in S} \Pr[\jj_{t'} \in P_{t'}] \right) \cdot |y_{\ii}|^2 \cdot |y_{\jj}|^2 \label{variance-diagonal-terms}
	\end{align}
	Where the second line follows because $P_1, \ldots P_q$ are independent and the third line follows from the definition of the set $[d]^q_{\jj_S}$. Now we can bound \eqref{variance-diagonal-terms} by noting that for any $i\neq j$, the collision probability $ \Pr[i, j \in P_t] = \frac{m(m-1)}{d(d-1)} \le \left(\frac{m}{d}\right)^2$ and $\Pr[j \in P_t] = \frac{m}{d}$. We can write,
	\begin{align*}
		\EE_{P_1, \ldots P_q} \left[ \left(\frac{1}{m^{q}} \cdot \left\| \left( P_1 \times P_2 \times \ldots P_q \right) \cdot y \right\|_2^2 \right)^2 \right] &\le \frac{1}{m^{2q}} \sum_{\jj \in [d]^q } \sum_{\ii \in [d]^q} \left( \frac{m}{d} \right)^{2q} \cdot |y_{\ii}|^2 \cdot |y_{\jj}|^2 \\
		&\qquad + \frac{1}{m^{2q}} \sum_{\emptyset \neq S \subseteq [q]} \sum_{\jj \in [d]^q } \sum_{\ii \in [d]^q_{\jj_S}} \left(\frac{m}{d} \right)^{2q-|S|} \cdot |y_{\ii}|^2 \cdot |y_{\jj}|^2\\
		&= \frac{\| y \|_2^4}{d^{2q}} + \sum_{\emptyset \neq S \subseteq [q]} \frac{1}{m^{|S|} \cdot d^{2q-|S|}} \sum_{\jj \in [d]^q }|y_{\jj}|^2 \sum_{\ii \in [d]^q_{\jj_S}} |y_{\ii}|^2\\
		&\le  \frac{\| y \|_2^4}{d^{2q}} + \sum_{\substack{S \subseteq [q] \\ |S| = 1}}  \frac{O( \log n )  \left\| y \right\|_2^2}{m \cdot d^{2q}} \sum_{\jj \in [d]^q }|y_{\jj}|^2 \\
		& + \sum_{ \substack{S \subseteq [q] \\ |S| > 1}}  \frac{O( \log n )^{|S|} \cdot \kappa^2 \left\| y \right\|_2^2}{m^{|S|} \cdot d^{2q}} \sum_{\jj \in [d]^q }|y_{\jj}|^2\\
		&\le \frac{\| y \|_2^4}{d^{2q}} +O\left( \frac{q \log n}{m} + \frac{q^2 \kappa^2 \log^2 n }{m^2} \right) \cdot \frac{\| y \|_2^4}{d^{2q}}, 
	\end{align*}
	where the fourth and fifth lines above follow from the fact that $y$ satisfies the condition in \eqref{eq:flatness-row-norms-vector}.
	Therefore, the above inequality along with the fact that $\frac{1}{m^{q}} \cdot \left\| \left( P_1 \times P_2 \times \ldots P_q \right) \cdot y \right\|_2^2$ is an unbiased estimator implies that,
	\[ {\rm Var}_{P_1, \ldots P_q} \left[ \frac{1}{m^{q}} \cdot \left\| \left( P_1 \times P_2 \times \ldots P_q \right) \cdot y \right\|_2^2 \right] =  O\left( \frac{q \log n}{m} + \frac{q^2 \kappa^2 \log^2 n }{m^2} \right) \cdot \frac{\| y \|_2^4}{d^{2q}} \]
	Thus if $m = C \left( \frac{1}{\epsilon^2} + \frac{ \kappa}{\epsilon} \right) \cdot \frac{q}{\delta} \log n $ for a large enough constant $C$, by using the definition of vector $y$ together with Chebyshev's inequality and a union bound, we have the following,
	\[ \Pr \left[ \left\| \left( \left( S^{(1)} X \right) \otimes \ldots \left( S^{(q)} X \right)  \right) \cdot (\Sigma \otimes K)^\top \right\|_F^2 \in (1 \pm \epsilon) \left\| X^{\otimes q} \cdot (\Sigma \otimes K)^\top \right\|_F^2\right] \ge 1 - \delta,\]
	so the lemma statement follows.
	
	The runtime of applying all sketches to $X$ consists of the time to compute $Y = H D X$ and the time to compute $P_r Y$ for every $r \in [q]$. The time to compute $Y$ is $O(n d \log d)$ by using the FFT algorithm and the time to compute all $P_r Y$ matrices is $O(q m n).$
	
	\section{Leverage Score Sampler for Polynomial Kernel}
	
	\subsection{Proof of Lemma~\ref{lem:rotatedrowsampler-poly-selftensor}}\label{appndx-proof-rowsampler-selftensor}
	
	All rows of the sampling matrix $S\in\RR^{s\times d^q}$ (the output of Algorithm \ref{alg:rotatedrowsampler-poly-selftensor}) have independent and identical distributions because for each  $\ell \in[s]$, the $\ell^{th}$ row of $S$ is constructed by sampling indices $i_1,i_2,\cdots i_q$ in line~\ref{isample-alg-selftensor} completely independent of the sampled values for other rows $\ell'\neq \ell$.
	Thus, it is enough to consider the distribution of the $\ell^{th}$ row of $S$ for some arbitrary $\ell \in[s]$. 
	Let $I:=(I_1,\ldots I_q)$ be a vector-valued random variable that takes values in $[d]^q$ with the following conditional probability distribution for every $a=1,2, \cdots q$ and every $i \in [d]$,
	\begin{equation}\label{eq:def-ia-dist-selftensor}
		\Pr\left[ I_a =i | I_1=i_1, \cdots I_{a-1}=i_{a-1} \right] := p^a_{h(i)} \cdot q^{a}_{i},
	\end{equation}
	where distributions $\{p^a_r\}_{r \in [s']}$ and $\{q^{a}_i\}_{i \in h^{-1}(t)}$ for every $t \in [s']$ are defined as per lines~\ref{dist-par-selftensor} and \ref{dist-qai-selftensor} of the algorithm.
	One can verify that the random vector $(i_1,i_2, \cdots i_q)$ obtained by stitching together the random indices generated in line~\ref{isample-alg-selftensor} of the algorithm, is in fact a copy of $I$ defined above.

	Let $\beta_\ell$ be the quantity computed in line \ref{beta-update-selftensor} of the algorithm. If $i_1,i_2, \cdots i_q \in[d]$ are the indices sampled in line~\ref{isample-alg-selftensor} of the algorithm, then using the conditional distribution of $I$ in \eqref{eq:def-ia-dist-selftensor}, we find that the value of $\beta_\ell$ is equal to the following,
	\begin{align*}
		\beta_\ell &= s \cdot \prod_{a=1}^q p^a_{h(i_a)} q^{a}_{i_a} \\
		&= s \cdot \prod_{a=1}^q \Pr \left[ I_a =i_a | I_1=i_1, \cdots I_{a-1}=i_{a-1} \right]\\
		&= s \cdot \Pr \left[ I =(i_1, i_2, \ldots i_q) \right],
	\end{align*}
	where $p^a$ and $q^{a}$ are the distributions computed in lines~\ref{dist-par-selftensor} and \ref{dist-qai-selftensor} of the algorithm.
	Hence, for any $i_1,i_2, \cdots i_q \in [d]$, the distribution of $S_{\ell,\star}$ is,
	\begin{align}
		\Pr\left[ S_{\ell,\star}= \beta_\ell^{-1/2} ({ e}_{i_1} \otimes { e}_{i_2} \otimes \cdots { e}_{i_q})^\top \right]= \Pr \left[ I =(i_1, i_2, \ldots i_q) \right] = \frac{\beta_\ell}{s}. \label{dist-sampling-matrix-poly-selftensor}
	\end{align}
	Now to ease the notation we define $Y_k^{(c)} := \bigotimes_{j=c}^q S_k^{(j)} X$ for every $k \in [m']$ and $c \in [q]$, where $S_k^{(c)}$ are the \SRHT sketches with shared signs drawn in line~\ref{srht-sharedsign} of the algorithm.
	From the definition of $\textsc{TN}^{(k)}$ in line~\ref{tnormDS} and by invoking Lemma~\ref{lem:TensorNormDS} we have the following inequalities for any $r\in[s']$, $k \in [m']$, $i \in [d]$, and any $a=1,2, \ldots q$ , with probability at least $1 - \frac{1}{{\rm poly}(n)}$,
	\begin{align}
		&\textsc{TN}^{(k)}\textsc{.Query}\left( L^a_{r,k} , a \right) \in \left( 1\pm \frac{1}{40q} \right) \left\| \left( Y_k^{(a+1)} \otimes M \right) D^{a} W_{r, k}^{\top}  \right\|_F^2,\label{tensorsketch-selftensor}\\
		&\textsc{TN}^{(k)}\textsc{.Query}\left( D^{a} X_{i,\star}^\top , a \right) \in \left( 1\pm \frac{1}{40q} \right) \left\| \left( Y_k^{(a+1)} \otimes M \right) D^{a} X_{i,\star}^\top \right\|_2^2 \label{Tnorm-xi-selftensor}
	\end{align}
	By union bounding over $qds'm'$ events, \eqref{tensorsketch-selftensor} and \eqref{Tnorm-xi-selftensor} hold simultaneously for all $a\in[q]$, $k \in [m']$, $i\in[d]$, and all $r\in[s']$ with high probability. From now on we condition on \eqref{tensorsketch-selftensor} and \eqref{Tnorm-xi-selftensor}. 
	
	Furthermore, note that $W_{r,k}$ is defined in line~\ref{W-selftensor} as $W_{r,k}=G_r^k \cdot X_{h^{-1}(r),\star} $, where $G_r^k$ is a degree-$1$ \PolyS with target dimension $n'=C_3q^2$. 
	By Lemma~\ref{soda-result}, $G_r^k$ approximately preserves Frobenius norm of any fixed matrix with constant probability. In particular, for every $a\in[q], r\in[s'], k \in [m']$, with probability at least $19/20$:
	\begin{align}
		\left\| \left( Y_k^{(a+1)} \otimes M \right) D^{a} W_{r,k}^{\top} \right\|_F^2 \in \left(1\pm \frac{1}{80q}\right) \left\| Y_k^{(a+1)}  D^{a} \left( X_{h^{-1}(r),\star} \otimes M \right)^\top \right\|_F^2. \label{jl-poly-selftensor}
	\end{align}
	To obtain the above inequality we used the fact that there is a bijective correspondence between entries of $\left( Y_k^{(a+1)} \otimes M \right) D^{a} X_{h^{-1}(r),\star}^\top $ and $Y_k^{(a+1)}  D^{a} \left( X_{h^{-1}(r),\star} \otimes M \right)^\top$.
	
	Additionally, note that $M = H \cdot (B^\top B + \lambda I)^{-1/2}$ for a random Gaussian matrix $H$ with $d' = \Omega(q^2 \log n)$ rows. Therefore, $H$ is a JL-transform. So if we define $A := (B^\top B + \lambda I)^{-1/2}$ for ease of notation, then with probability $1 - \frac{1}{{\rm poly}(n)}$, the following holds for any $a\in[q], r\in[s']$:
	\begin{align*}
		\left\| Y_k^{(a+1)} D^{a} \left( X_{h^{-1}(r),\star} \otimes M \right)^\top \right\|_F^2  \in \left(1\pm \frac{1}{80q}\right) \left\| Y_k^{(a+1)} D^{a} \left( X_{h^{-1}(r),\star} \otimes A \right)^\top \right\|_F^2.
	\end{align*}
	By union bounding over $qs'$ events we can conclude that the above inequality holds simultaneously for all $a\in[q]$ and $r\in[s']$. From now on we condition on the above inequality holding. By combining this condition with \eqref{jl-poly-selftensor} we find that with probability at least $19/20$ the following holds:
	\begin{align}
		\left\| \left( Y_k^{(a+1)} \otimes M \right) D^{a} W_{r,k}^{\top} \right\|_F^2  \in \left(1\pm \frac{1}{39q}\right) \left\| Y_k^{(a+1)} D^{a} \left( X_{h^{-1}(r),\star} \otimes A \right)^\top \right\|_F^2. \label{eq:deg2tensorsketch}
	\end{align}
	Using the definition of matrices $Y_k^{(c)} := \bigotimes_{j=c}^q S_k^{(j)} X$ and by Lemma~\ref{lem:SRHT-shared-sign}, because the number of rows of $S^{(c)}_k$ is $m'' = \Omega(q^3 + q^2 \kappa \log n)$, the following holds with probability at least $19/20$ for any $a\in[q], r\in[s'], k \in [m']$,
	\begin{align*}
		\left\| Y_k^{(a+1)} D^{a} \left( X_{h^{-1}(r),\star} \otimes A \right)^\top \right\|_F^2 \in \left(1\pm \frac{1}{80q}\right) \left\| X^{\otimes (q - a)} D^{a} \left( X_{h^{-1}(r),\star} \otimes A \right)^\top \right\|_F^2.
	\end{align*}
	By combining the above with \eqref{eq:deg2tensorsketch} using a union bound, and plugging the result into \eqref{tensorsketch-selftensor} we find that with probability at least $9/10$ the following holds,
	\begin{align*}
		\textsc{TN}^{(k)}\textsc{.Query}\left( L^a_{r,k} , a \right) \in \left(1\pm \frac{1}{10q}\right) \left\| X^{\otimes (q - a)} \cdot D^{a} \left( X_{h^{-1}(r),\star} \otimes A \right)^\top \right\|_F^2.
	\end{align*}
	By taking the median of $m' = \Omega(\log n)$ independent instances of $\textsc{TN}^{(k)}\textsc{.Query}\left( L^a_{r,k} , a \right)$, the success probability of the above gets boosted.  Thus, by a union bound, with probability at least $1 - \frac{1}{{\rm poly}(n)}$ the following holds simultaneously for all $a \in [q]$ and $r \in [s']$,
	\begin{align}
		\textsc{Median}_{k \in [m']} \left\{ \textsc{TN}^{(k)}\textsc{.Query}\left( L^a_{r,k} , a \right) \right\} \in \left( 1 \pm \frac{1}{10q} \right) \left\| X^{\otimes (q - a)}  D^{a} \left( X_{h^{-1}(r),\star} \otimes A \right)^\top \right\|_F^2 \label{jl-countsketch-median-selftensor}
	\end{align}

	Similarly, we can use the fact that there is a bijective correspondence between the entries of $\left( Y_k^{(a+1)} \otimes M \right) D^{a} X_{i,\star}^\top$ and $Y_k^{(a+1)} D^{a} {\rm diag} \left(X_{i,\star} \right) M^\top$ along with $M = H \cdot A$ to conclude that with probability $1 - \frac{1}{{\rm poly}(n)}$, the following holds for any $a \in[q], r\in[s'], k \in [m'],i \in [d]$:
	\begin{align}
		\left\| \left( Y_k^{(a+1)} \otimes M \right) D^{a} X_{i,\star}^\top \right\|_2^2 \in \left(1\pm \frac{1}{80q}\right) \left\| Y_k^{(a+1)} D^{a} \cdot {\rm diag} \left(X_{i,\star} \right) A \right\|_F^2\label{eq:condition-JL}
	\end{align}
	By a union bound over $qs'm' d$ events we can conclude that the above inequality holds simultaneously for all $a \in[q], r\in[s'], k \in [m'], i \in [d]$. From now on we condition on the above inequality holding. Then, by using the definition of matrices $Y_k^{(c)} := \bigotimes_{j=c}^q S^{(j)}_k X$ and invoking Lemma~\ref{lem:SRHT-shared-sign}, the following holds with probability at least $19/20$ for any $a\in[q], r\in[s'], k \in [m'], i \in [d]$,
	\begin{align*}
		\left\| Y_k^{(a+1)} D^{a} {\rm diag} \left(X_{i,\star} \right) A \right\|_F^2 \in \left(1\pm \frac{1}{80q}\right) \left\| X^{\otimes (q - a)} \cdot D^{a} {\rm diag} \left(X_{i,\star} \right)  A \right\|_F^2
	\end{align*}
	By combining this with the condition in \eqref{eq:condition-JL} and \eqref{Tnorm-xi-selftensor} we find that with probability at least $19/20$:
	\begin{align*}
		\textsc{TN}^{(k)}\textsc{.Query}\left( D^{a} X_{i,\star}^\top , a \right) \in \left(1\pm \frac{1}{15q}\right) \left\| X^{\otimes (q - a)} D^{a} {\rm diag} \left(X_{i,\star} \right) A \right\|_F^2
	\end{align*}
	By taking the median of $m' = \Omega(\log n)$ independent instances of $\textsc{TN}^{(k)}\textsc{.Query}\left( D^{a} X_{i,\star}^\top , a \right)$, the success probability of the above gets boosted. Thus, by applying the median trick and then using a union bound, with probability at least $1 - \frac{1}{{\rm poly}(n)}$ the following holds simultaneously for all $a \in [q] , i \in [d] $ and $r \in [s']$,
	\begin{align*}
		\textsc{Median}_{k \in [m']} \left\{ \textsc{TN}^{(k)}\textsc{.Query}\left( D^{a} X_{i,\star}^\top , a \right) \right\} \in \left(1\pm \frac{1}{15q}\right) \left\| X^{\otimes (q - a)} D^{a} {\rm diag} \left(X_{i,\star} \right) A \right\|_F^2
	\end{align*}
	Plugging the above inequality along with \eqref{jl-countsketch-median-selftensor} into \eqref{eq:def-ia-dist-selftensor}, we conclude that with high probability the following bound holds simultaneously for all $a \in [q]$ and all $i \in [d]$,
	\begin{align}
		\Pr[I_a =i | I_1=i_1, I_2=i_2, \cdots I_{a-1}=i_{a-1}] \ge \left( 1-\frac{1}{3q} \right) \cdot \frac{ \left\| X^{\otimes (q - a)} D^{a} \cdot {\rm diag} \left(X_{i,\star} \right) A \right\|_F^2 }{ \left\| X^{\otimes (q-a+1)} D^{a} A \right\|_F^2}.\label{lower-bnd-Ij-selftensor}
	\end{align}
	Thus, using the definition of $D^a$ and $A = (B^\top B + \lambda I)^{-1/2}$, we have
	\begin{align*}
		\Pr\left[ I=(i_1 , i_2, \cdots i_q) \right]
		&= \prod_{a=1}^{q} \Pr \left[I_a =i_a | I_1=i_1, \cdots I_{a-1}=i_{a-1} \right]\\
		&\ge \prod_{a=1}^{q} \left(1-\frac{1}{3q}\right) \frac{ \left\| X^{\otimes (q - a)} D^{a} \cdot {\rm diag} \left(X_{i_a,\star} \right) A \right\|_F^2 }{ \left\| X^{\otimes (q-a+1)} D^{a} A \right\|_F^2}\\
		& \ge \frac{1}{2} \cdot \frac{ \left\| \mathbf{1}_n^\top \cdot D^{q} \cdot {\rm diag}\left( X_{i_q,\star} \right) A \right\|_F^2}{ \left\| X^{\otimes q} D^{1} A \right\|_F^2}\\ 
		&= \frac{1}{2} \cdot \frac{ \left\| \left[ X^{\otimes q} \cdot (B^\top B + \lambda I)^{-1/2} \right]_{(i_1,i_2,\cdots i_q),\star} \right\|_2^2}{ \left\| X^{\otimes q} \cdot (B^\top B + \lambda I)^{-1/2} \right\|_F^2}
	\end{align*}
	This shows that, with high probability in $n$,
	\begin{align*}
		\Pr\left[ S_{\ell,\star}= \beta_\ell^{-1/2} ({ e}_{i_1} \otimes { e}_{i_2} \otimes \cdots { e}_{i_q})^\top \right] \ge \frac{1}{2} \cdot \frac{ \left\| \left[ X^{\otimes q} \cdot (B^\top B + \lambda I)^{-1/2} \right]_{(i_1,i_2,\cdots i_q),\star} \right\|_2^2}{ \left\| X^{\otimes q} \cdot (B^\top B + \lambda I)^{-1/2} \right\|_F^2}
	\end{align*}
	Because $\frac{\beta_\ell}{s}$ is the probability of sampling row $(i_1,i_2,\cdots i_q)$ of $ X^{\otimes q} (B^\top B +\lambda I)^{-1/2}$, the above inequality proves that with high probability, matrix $S$ is a rank-$s$ row norm sampler for $ X^{\otimes q} (B^\top B +\lambda I)^{-1/2}$ as in Definition \ref{def:row-samp}.

	\paragraph{Runtime:} The first expensive step of this algorithm is the computation of $M$ in line~\ref{M-alg-selftensor} which takes $O(m^2 n + q^2 m n \log n)$ operations since $B$ has rank at most $m$. 
	The next expensive computation is the computation of $S^{(c)}_k X$ for $c \in [q]$ and $k \in [m']$ in line~\ref{tnormDS} of the algorithm. By Lemma~\ref{lem:SRHT-shared-sign}, the total time to compute these sketched matrices is $O\left( (q^4 + q^3 \kappa) n \log^2 n + nd \log^2 n \right)$.
	Another expensive step is the construction of the $\textsc{TN}^{(k)}$ data-structure in line~\ref{tnormDS} for $k \in [m']$. By Lemma \ref{lem:TensorNormDS}, these DS's for $\epsilon = \frac{1}{40q}$ and all $k \in [m']$ can be formed in total time $O\left( q^4 \log^2 q \cdot n \log^2 n + (q^4 + q^3\kappa) n \log^3 n \right)$. 
	
	By Lemma~\ref{soda-result}, matrices $W_{r,k}$ for all $r\in[s']$ and $k \in [m']$ in line~\ref{W-selftensor} of the algorithm can be computed in total time $O\left( q^2 s'  n\log^2 n + \log n \cdot \text{nnz}\left( X \right) \right)$.
	
	The matrix $W_{r, k}$ for every $k \in [m']$, and $r\in[s']$, has size $O(q^2) \times n$.
	Thus, by Lemma~\ref{lem:TensorNormDS}, computing the distribution $\{ p^a_r \}_{r=1}^{s'}$ in line~\ref{dist-par-selftensor} takes time $O\left(q^{4} s' \cdot n \log^2n \log q \right)$ for a fixed $a\in[q]$ and a fixed $\ell \in[s]$.  Therefore, the total time to compute this distribution for all $a$ and $\ell$ is $O\left(q^{8} s^2 \cdot n \log^2 n\log q \right)$. 
	
	The runtime of computing the distribution $\{ q^{a}_i \}_{i\in h^{-1}(t)}$ in line~\ref{dist-qai-selftensor} depends on the sparsity of $X_{h^{-1}(t),\star}$, i.e.,  $\text{nnz}\left( X_{h^{-1}(t),\star} \right)$. To bound the sparsity of $X_{h^{-1}(t),\star}$, note that, the hash function $h$ is fully independent.
	Thus, by invoking Bernstein's inequality, we find that, $\Pr \left[ {\rm nnz}\left( X_{h^{-1}(t),\star} \right) = O\left(\log n \cdot \left( \sqrt{\frac{n}{s'} \text{nnz}\left( X \right)} + n \right) \right) \right] \ge 1 - \frac{1}{{\rm poly}(n)}$.
	Hence, by union bounding over $qs'$ events, with high probability in $n$, ${\rm nnz}\left( X_{h^{-1}(t),\star} \right) = O\left( \left( {{\rm nnz}\left( X \right)/s'} + n \right) \log n \right)$, simultaneously for all $t\in[s']$ and $a \in [q]$.

	Therefore, by Lemma~\ref{lem:TensorNormDS}, the distribution $\{ q^{a}_i \}_{i\in h^{-1}(t)}$ in line~\ref{dist-qai-selftensor} of the algorithm can be computed in total time $O\left( q^3 s n \log^4 n \log q + \log^4 n \log q\cdot {\rm nnz}\left( X \right) \right)$ for all $a \in [q]$ and all $\ell \in [s]$.

	The total runtime of Algorithm~\ref{alg:rotatedrowsampler-poly-selftensor} is thus $O( m^2n + q^{8} s^2 n\log^2 n \log q + q^3\kappa n\log^3 n + d n\log^4n )$.

	\subsection{Proof of Theorem~\ref{main-thrm-selftensor-prod}}\label{appndx:proof-maintheorem}
	The theorem follows by invoking Lemmas~\ref{resursive-rlss-lem} and \ref{lem:rotatedrowsampler-poly-selftensor}. To find the sampling matrix $\Pi$, run Algorithm~\ref{alg:outerloop} on $\Phi$ with $\mu=s_\lambda$ and for the \textsc{RowSampler} primitive, invoke Algorithm~\ref{alg:rotatedrowsampler-poly-selftensor}. By Lemma~\ref{lem:rotatedrowsampler-poly-selftensor}, Algorithm~\ref{alg:rotatedrowsampler-poly-selftensor} outputs a row norm sampler as per Definition~\ref{def:row-samp}, with probability $1 - \frac{1}{{\rm poly}(n)}$. Therefore, since the total number of times Algorithm~\ref{alg:rotatedrowsampler-poly-selftensor} is invoked by Algorithm~\ref{alg:outerloop} is $\log \frac{\| \Phi \|_F^2}{\epsilon\lambda} = O(\log n)$, by a union bound, the preconditions of Lemma~\ref{resursive-rlss-lem} are satisfied with high probability.
	Thus, it follows that $\Pi$ satisfies the following spectral approximation guarantee
	\[ \frac{\Phi^\top \Phi + \lambda I}{1+\epsilon} \preceq \Phi^\top \Pi^\top \Pi \Phi + \lambda I \preceq \frac{\Phi^\top \Phi + \lambda I}{1-\epsilon}.\]
	
	The only thing that remains is bounding the runtime. In the proof of Lemma~\ref{resursive-rlss-lem} in \cite{woodruff2020near}, it is shown that with high probability at any iteration $t \in [T]$ of Algorithm~\ref{alg:outerloop},
	the following holds,
	\[ \frac{\Phi^\top \Phi + \lambda_t I}{1+\epsilon} \preceq \Phi^\top \Pi_t^\top \Pi_t \Phi + \lambda_t I \preceq \frac{\Phi^\top \Phi + \lambda_t I}{1-\epsilon}.\]
	Therefore, $\left\|  \Phi^\top \Pi_t^\top \Pi_t \Phi \right\| = O(\| \Phi^\top \Phi \|)$. Now note that
	Algorithm~\ref{alg:outerloop} invokes the \textsc{RowSampler} primitive $T = \log \frac{\| \Phi \|_F^2}{\epsilon \lambda} = O(\log n)$ times. Thus, by Lemma~\ref{lem:rotatedrowsampler-poly-selftensor}, the runtime of finding $\Pi$ is 
	the sum of $O\left( \frac{q^{8} s_\lambda^2 n \log^4 n }{\epsilon^4} + \sqrt{ \frac{\left\|  \Phi^\top \Pi_t^\top \Pi_t \Phi \right\|}{\lambda_t} } q^3 n \log^3 n + n d \log^4 n \right)$ for all $t \in [T]$.
	Since $\lambda_t = 2^{T - t}\lambda$ has a geometric decay, the total time complexity is $O\left( \frac{q^{8} s_\lambda^2 n \log^5 n}{\epsilon^4} + \sqrt{ \frac{ \left\| \Phi^\top \Phi \right\|}{\lambda} } q^3 n \log^3 n + n d \log^5 n \right)$.
	
	\section{Spectral Approximation to Generalized Polynomial Kernels}\label{appndx-GPK}
	In this section we design an algorithm that can produce a spectral approximation to the GPK defined in Definition~\ref{def;GPK}. Our approach is to perform leverage score sampling on the GPK feature matrix $\Phi$ defined in \eqref{feature-matrix-GPK}. We do this by invoking our recursive sampling method given in Algorithm~\ref{alg:outerloop} on $\Phi$. 
	Our central contribution is the design of \textsc{RowSampler} algorithm for the GPK feature matrix $\Phi$ that runs in input sparsity time. 
	This procedure can perform \emph{row norm sampling} as per Definition \ref{def:row-samp} on $\Phi (B^\top B + \lambda I)^{-1/2}$ for $\Phi = \bigoplus_{j=0}^q \alpha_j X^{\otimes j} {\rm diag}(v)$ using $\widetilde{O} \left( {\rm nnz}(X) \right)$ runtime. 
	Our primitive is an adaptation and generalization of Algorithm~\ref{alg:rotatedrowsampler-poly-selftensor}.

	\begin{algorithm}[!t]
		\caption{\algoname{RowSampler} for GPK features $\Phi = \bigoplus_{j=0}^q \alpha_j X^{\otimes j} {\rm diag}(v)$}
		{\bf input}: $q, s\in \mathbb{Z}_+$, $X \in \RR^{d \times n}$, $v \in \RR^{n}$, $\alpha \in \RR^{q+1}$, $B \in \RR^{m \times n}$, $\lambda >0$\\
		{\bf output}: Sampling matrix $S \in \RR^{s \times d^q}$
		\begin{algorithmic}[1]
			\STATE{$\kappa \gets \sqrt{ \frac{\| B^\top B \|}{\lambda} + 1 }$}
			\STATE{Generate $H\in\RR^{d' \times n}$ with i.i.d. normal entries with $d'= C_0 q^2 \log n$ rows}
			\STATE{$M \gets H \cdot (B^\top B + \lambda I)^{-1/2}$}\alglinelabel{M-alg-gpk}
			
			\STATE{For every $k \in [m']$, let $S_k^{(1)}, S_k^{(2)} , \ldots S_k^{(q)} \in \RR^{m'' \times d}$ be independent copies of \SRHT sketches with shared signs as per Lemma~\ref{lem:SRHT-shared-sign}, where $m' = C_1\log n$ and $m'' = C_2 (q^3 + q^2 \kappa ) \log n$}\alglinelabel{srht-sharedsign-gpk}
			
			\STATE{For every $k \in [m']$, let $\textsc{TNorm}^{(k)}$ be the DS in Algorithm~\ref{alg:TensorNormDS} for inputs $\left( S_k^{(1)} X, \ldots S_k^{(q)} X, M \right)$ and $\epsilon=\frac{1}{40q}$}\alglinelabel{tnormDS-gpk}

			\STATE{Let $h:[d]\rightarrow[s']$ be a fully independent and uniform hash function with $s' = \lceil {q}^{3}s \rceil$ buckets}
			\STATE{Let $h^{-1}(r)=\left\{ j\in[d]:h(j)=r \right\}$ for every $r\in[s']$}
			
			\STATE{For every $r\in[s']$ and $k \in [m']$, let $G_r^k \in \RR^{n'\times d_r }$ be independent instances of degree-$1$ \PolyS as per Lemma~\ref{soda-result}, where $d_r=|h^{-1}(r)|$, $n'=C_3q^2$}
			
			\STATE{$W_{r, k} \gets G_r^k \cdot X_{h^{-1}(r),\star}$ for every $k \in [m']$ and $r\in[s']$}\alglinelabel{W-gpk}

			\STATE{$f_j \gets \alpha_j^2 \cdot \textsc{Median}_{k \in [m']} \textsc{TNorm}^{(k)}\textsc{.Query}( v , q-j )$ for every $j = 0,1, \ldots q$}
			\STATE{$f_j \gets f_j / \sum_{i=0}^q f_i$ for every $j = 0,1, \ldots q$}\alglinelabel{dist-fj}
			
			\FOR{$\ell=1$ to $s$}

			\STATE{Sample $b \in \{0,1,\ldots q\}$ from distribution $\{ f_j \}_{j=0}^q$}\alglinelabel{degree-b-gpk}

			\STATE{$D^{1} \gets {\rm diag}(v)$ and $\beta_\ell \gets s\cdot f_b$}\alglinelabel{beta-forloop-gpk}
			\FOR{$a=1$ to $b$}

			\STATE{$L^a_{r, k} \gets D^{a} \cdot W_{r, k}^{\top}$ for every $k \in [m']$, and $r\in[s']$}
			
			\STATE{$p^{a}_r \gets  \textsc{Median}_{k \in [m']} \textsc{TNorm}^{(k)}\textsc{.Query}( L^a_{r, k} , a + q - b)$ for every $r \in [s']$}
			\STATE{$p^{a}_r \gets  p^{a}_r/ \sum_{t=1}^{s'} p^{a}_t$ for every $r \in [s']$}\alglinelabel{dist-par-gpk}
			\STATE{Sample $t \in [s']$ from distribution $\{p^a_r\}_{r=1}^{s'}$}
			
			\STATE{Let $q^{a}_i \gets \textsc{Median}_{k \in [m']} \textsc{TNorm}^{(k)}\textsc{.Query} ( D^{a} X_{i,\star}^\top , a + q-b )$ for every $i\in h^{-1}(t)$}
			\STATE{$q^{a}_i \gets  q^{a}_i/ \sum_{j \in h^{-1}(t)} q^{a}_j$ for every $i\in h^{-1}(t)$}\alglinelabel{dist-qai-gpk}
			
			\STATE{Sample $i_a \in [d]$ from distribution $\{q^a_i \}_{i\in h^{-1}(t)}$}\alglinelabel{isample-alg-gpk}
			
			\STATE{$D^{a+1}\gets D^{a}\cdot {\rm diag}\left( X^{(a)}_{i_a,\star} \right)$}\alglinelabel{def-Da-gpk}
			
			\STATE{$\beta_\ell \gets  \beta_\ell \cdot p^{a}_{t} q_{i_a}^{a} $} \alglinelabel{beta-update-gpk}
			\ENDFOR
			
			\IF{$b>0$}
			\STATE{Let $\ell^{th}$ row of $S$ be $ \beta_\ell^{-1/2}  
				\big( \underbrace{0,0, \ldots 0}_{\frac{d^b - 1}{d-1} \text{ zeros}}, {e}_{i_1} \otimes {e}_{i_2} \otimes \cdots {e}_{i_b}, \underbrace{0,0, \ldots 0}_{\frac{d^{q+1} - d^{b+1}}{d-1} \text{ zeros}} \big)$}
			\ELSE
			\STATE{Let $\ell^{th}$ row of $S$ be $\beta_\ell^{-1/2}  
				\big( 1, \underbrace{0,0, \ldots 0}_{\frac{d^{q+1} - d}{d-1} \text{ zeros}} \big)$}
			\ENDIF
			\ENDFOR
			\STATE{\textbf{return} $S$}
		\end{algorithmic}
		\label{alg:rotatedrowsampler-GPK}
	\end{algorithm}
	
	The formal guarantee on Algorithm~\ref{alg:rotatedrowsampler-GPK} is given in the following lemma.
	\begin{lemma}\label{lem:rotatedrowsampler-GPK}
		For any matrix $X \in \RR^{d\times n}$, any vector $v \in \RR^n$, any positive integers $q,s$, and any $\alpha \in \RR^{q+1}$, let $\Phi$ be the GPK feature matrix defined in \eqref{feature-matrix-GPK}.
		For any matrix $B\in\RR^{m\times n}$ and any $\lambda>0$, with probability at least $1 - \frac{1}{{\rm poly}(n)}$, Algorithm~\ref{alg:rotatedrowsampler-GPK} outputs a rank-$s$ row norm sampler for $ \Phi (B^\top B + \lambda I)^{-1/2}$ as per Definition \ref{def:row-samp}, in time $O\left( m^2n + q^{8} s^2 n  \log^3 n + q^3 \kappa n \log^3 n + nd \log^4n \right)$, where $\kappa = \sqrt{\| B^\top B \| / \lambda +1 } $.
	\end{lemma}
	\begin{proof}
		All rows of the sampling matrix $S\in\RR^{s\times d^q}$ (the output of Algorithm \ref{alg:rotatedrowsampler-poly-selftensor}) have independent and identical distributions because for each  $\ell \in[s]$, the $\ell^{th}$ row of the matrix $S$ is constructed by sampling the degree $b$ and indices $i_1,i_2,\cdots i_q$ in lines\ref{degree-b-gpk} and \ref{isample-alg-gpk}, respectively, completely independent of the sampled values for other rows $\ell'\neq \ell$.
		Thus, it is enough to consider the distribution of the $\ell^{th}$ row of $S$ for some arbitrary $\ell \in[s]$. 
		
		Let $U$ be a random variable that takes values in $\{ 0,1, \ldots q \}$ with the following distribution for every $a=0,1, \ldots q$
		\begin{equation}\label{eq:U-dist}
			\Pr[U=a] = f_a
		\end{equation}
		where $\{ f_a \}_{a=0}^q$ is the distribution defined in line~\ref{dist-fj} of the algorithm.
		Additionally, for any $b \in \{ 1, \ldots q\}$, let $I^b:=(I_1,\ldots I_b)$ be a vector-valued random variable that takes values in $[d]^b$ with the following conditional probability distribution for every $a=1,2, \cdots b$ and every $i \in [d]$,
		\begin{equation}\label{eq:def-ia-dist-gpk}
			\Pr\left[ I_a =i | I_1=i_1, \cdots I_{a-1}=i_{a-1} \right] := p^a_{h(i)} \cdot q^{a}_{i},
		\end{equation}
		where distributions $\{p^a_r\}_{r \in [s']}$ and $\{q^{a}_i\}_{i \in h^{-1}(t)}$ for every $t \in [s']$ are defined as per lines~\ref{dist-par-gpk} and \ref{dist-qai-gpk} of the algorithm.
		One can verify that conditioned on Algorithm~\ref{alg:rotatedrowsampler-GPK} sampling some $b=1,\ldots q$ in line~\ref{degree-b-gpk}, the random vector $(i_1,i_2, \cdots i_b)$ obtained by stitching together the random indices generated in line~\ref{isample-alg-gpk} of the algorithm, is in fact a copy of $I^b$ defined above. Note that if the algorithm samples degree $b=0$ in line~\ref{degree-b-gpk} then the algorithm does not sample any indices in  line~\ref{isample-alg-gpk}.

		Let $\beta_\ell$ be the quantity computed in line \ref{beta-update-gpk} of the algorithm. If $b\in \{1, \ldots q\}$ is the degree sampled in line~\ref{degree-b-gpk} and $i_1,i_2, \cdots i_b \in[d]$ are the indices sampled in line~\ref{isample-alg-gpk} of the algorithm, then using the distribution of $U$ in \eqref{eq:U-dist} and the conditional distribution of $I^b$ in \eqref{eq:def-ia-dist-gpk}, we find that the value of $\beta_\ell$ is equal to the following,
		\begin{align*}
			\beta_\ell &= s \cdot f_b \cdot \prod_{a=1}^b p^a_{h(i_a)} q^{a}_{i_a} \\
			&= s \cdot \Pr[U=b] \cdot \prod_{a=1}^b \Pr \left[ I_a =i_a | I_1=i_1, \cdots I_{a-1}=i_{a-1} \right]\\
			&= s \cdot \Pr \left[ I^b =(i_1, i_2, \ldots i_b) \right]\cdot \Pr[U=b],
		\end{align*}
		where $p^a$ and $q^{a}$ are the distributions computed in lines~\ref{dist-par-gpk} and \ref{dist-qai-gpk} of the algorithm.
		Hence, for any $b=1,\ldots q$ and any $i_1,i_2, \cdots i_b \in [d]$, the distribution of $S_{\ell,\star}$ is,
		\begin{align}
			&\Pr\left[ S_{\ell,\star}= \beta_\ell^{-1/2} \big( \underbrace{0,0, \ldots 0}_{\frac{d^b - 1}{d-1} \text{ zeros}}, {e}_{i_1} \otimes {e}_{i_2} \otimes \cdots {e}_{i_b}, \underbrace{0,0, \ldots 0}_{\frac{d^{q+1} - d^{b+1}}{d-1} \text{ zeros}} \big) \right]\nonumber\\
			&\qquad= \Pr \left[ I^b =(i_1, i_2, \ldots i_b) \right]\cdot \Pr[U=b] = \frac{\beta_\ell}{s}. \label{dist-sampling-matrix-poly-selftensor}
		\end{align}
		Furthermore, if $b=0$ is the degree sampled in line~\ref{degree-b-gpk} of the algorithm then $\beta_\ell = s\cdot f_0 = s\cdot \Pr[U=0]$. Thus,
		\[ \Pr\left[ S_{\ell,\star}= \beta_\ell^{-1/2} \big( 1, \underbrace{0,0, \ldots 0}_{\frac{d^{q+1} - d}{d-1} \text{ zeros}} \big) \right] = \Pr[U=0] = \frac{\beta_\ell}{s}. \]
		
		Now to ease the notation we define $Y_k^{(c)} := \bigotimes_{j=c}^q S_k^{(j)} X$ for every $k \in [m']$ and $c \in [q]$, where $S_k^{(c)}$ are the \SRHT sketches with shared signs drawn in line~\ref{srht-sharedsign-gpk} of the algorithm.
		Using the definition of $\textsc{TNorm}^{(k)}$ in line~\ref{tnormDS-gpk} and by invoking Lemma~\ref{lem:TensorNormDS} we have the following inequality for any $k \in [m']$ and any $j=0,1, \ldots q$:
		\begin{equation}\label{Tnorm-taylorterms-gpk}
			\textsc{TNorm}^{(k)}\textsc{.Query}\left( v , q-j \right) \in \left( 1\pm \frac{1}{40q} \right) \left\| \left( Y_k^{(q-j+1)} \otimes M \right) v \right\|_2^2.
		\end{equation}
		By union bounding over $(q+1)m'$ events, \eqref{Tnorm-taylorterms-gpk} holds simultaneously for all $j \in \{0,1,\ldots q\}$, and $k \in [m']$, with high probability. From now on we condition on \eqref{Tnorm-taylorterms-gpk}. 
		Now, note that $M = H \cdot (B^\top B + \lambda I)^{-1/2}$ for a random Gaussian matrix $H$ with $d' = \Omega(q^2 \log n)$ rows. Therefore, $H$ is a JL-transform. So if we define $A := (B^\top B + \lambda I)^{-1/2}$ for ease of notation, then with probability $1 - \frac{1}{{\rm poly}(n)}$, the following holds for any $j\in\{ 0,1, \ldots q \}, k \in [m']$:
		\[ \left\| \left( Y_k^{(q-j+1)} \otimes M \right) v \right\|_2^2 \in \left( 1 \pm \frac{1}{80q} \right) \left\| Y_k^{(q-j+1)} \cdot {\rm diag}(v) A \right\|_F^2. \]
		To obtain the above inequality we used the fact that there is a bijective correspondence between entries of vector $\left( Y_k^{(q-j+1)} \otimes M \right) v$ and matrix $Y_k^{(q-j+1)} \cdot {\rm diag} (v) M^\top$.
		Now, using the above inequality along with the definition of matrices $Y_k^{(c)} := \bigotimes_{j=c}^q S_k^{(j)} X$ and by invoking Lemma~\ref{lem:SRHT-shared-sign}, because the number of rows of $S^{(c)}_k$'s is $m'' = \Omega(q^3 + q^2 \kappa \log n)$, the following holds with probability at least $9/10$ for any $j\in\{ 0,1, \ldots q \}, k \in [m']$,
		\begin{align*}
			\left\| \left( Y_k^{(q-j+1)} \otimes M \right) v \right\|_2^2 \in \left(1\pm \frac{1}{39q}\right) \left\| X^{\otimes j} \cdot {\rm diag}(v)A \right\|_F^2.
		\end{align*}
		By plugging the above into \eqref{Tnorm-taylorterms-gpk}, we find that with probability at least $9/10$ the following holds,
		\begin{align*}
			\textsc{TNorm}^{(k)}\textsc{.Query}\left( v , q-j \right) \in \left(1\pm \frac{1}{19q}\right) \left\| X^{\otimes j} \cdot {\rm diag}(v)A \right\|_F^2.
		\end{align*}
		By taking the median of $m' = \Omega(\log n)$ independent instances of $\textsc{TNorm}^{(k)}\textsc{.Query}\left( v , q-j \right)$, the success probability of the above gets boosted.  Thus, by a union bound, with probability at least $1 - \frac{1}{{\rm poly}(n)}$, the following holds simultaneously for all $j\in\{ 0,1, \ldots q \}$,
		\[ \textsc{Median}_{k \in [m']} \left\{ \textsc{TNorm}^{(k)}\textsc{.Query}\left( v , q-j \right) \right\} \in \left(1\pm \frac{1}{19q}\right) \left\| X^{\otimes j} \cdot {\rm diag}(v)A \right\|_F^2.\]
		Therefore, using the above along with \eqref{eq:U-dist} and definition of $f_j$ in line~\ref{dist-fj} of the algorithm as well as $A = (B^\top B + \lambda I)^{-1/2}$, with high probability in $n$, for any $b=0,1,\ldots q$ we have
		\begin{align}
			\Pr[U=b] &= f_b\nonumber \\ 
			&\ge \left(1\pm \frac{1}{9q}\right) \frac{\alpha_b^2 \cdot \left\| X^{\otimes b} \cdot {\rm diag}(v)A \right\|_F^2}{\sum_{j=0}^q \alpha_j^2 \cdot \left\| X^{\otimes j} \cdot {\rm diag}(v)A \right\|_F^2}\nonumber \\
			&= \left(1\pm \frac{1}{9q}\right) \frac{\alpha_b^2 \cdot \left\| X^{\otimes b} \cdot {\rm diag}(v) (B^\top B + \lambda I)^{-1/2} \right\|_F^2}{\left\| \Phi \cdot (B^\top B + \lambda I)^{-1/2} \right\|_F^2},\label{eq:dist-fj-gpk}
		\end{align}
		where the last line follows from the definition of $\Phi = \bigoplus_{j=0}^q \alpha_j X^{\otimes j} {\rm diag}(v)$.

		Moreover, suppose that $b\in \{ 1,2, \ldots q \}$. From the definition of $\textsc{TNorm}^{(k)}$ in line~\ref{tnormDS-gpk} and by invoking Lemma~\ref{lem:TensorNormDS} we have the following inequalities for any $r\in[s']$, $k \in [m']$, $i \in [d]$, and any $a=1,2, \ldots b$ , with probability at least $1 - \frac{1}{{\rm poly}(n)}$,
		\begin{align}
			&\textsc{TNorm}^{(k)}\textsc{.Query}\left( L^a_{r,k} , a + q-b \right) \in \left( 1\pm \frac{1}{40q} \right) \left\| \left( Y_k^{(a + q-b +1)} \otimes M \right) D^{a} W_{r, k}^{\top}  \right\|_F^2,\label{tensorsketch-gpk}\\
			&\textsc{TNorm}^{(k)}\textsc{.Query}\left( D^{a} X_{i,\star}^\top , a + q-b \right) \in \left( 1\pm \frac{1}{40q} \right) \left\| \left( Y_k^{(a + q-b +1)} \otimes M \right) D^{a} X_{i,\star}^\top \right\|_2^2 \label{Tnorm-xi-gpk}
		\end{align}
		By union bounding over $qds'm'$ events, \eqref{tensorsketch-gpk}, and \eqref{Tnorm-xi-gpk} hold simultaneously for all $a\in[b]$, $k \in [m']$, $i\in[d]$, and all $r\in[s']$ with high probability. From now on we condition on \eqref{tensorsketch-gpk} and \eqref{Tnorm-xi-gpk}.

		Furthermore, note that $W_{r,k}$ is defined in line~\ref{W-gpk} as $W_{r,k}=G_r^k \cdot X_{h^{-1}(r),\star} $, where $G_r^k$ is a degree-$1$ \PolyS with target dimension $n'=C_3q^2$. 
		By Lemma~\ref{soda-result}, $G_r^k$ approximately preserves the Frobenius norm of any fixed matrix with constant probability. In particular, for every $a\in[b], r\in[s'], k \in [m']$, with probability at least $19/20$:
		\begin{align}
			\left\| \left( Y_k^{(a +q-b +1)} \otimes M \right) D^{a} W_{r,k}^{\top} \right\|_F^2 \in \left(1\pm \frac{1}{80q}\right) \left\| Y_k^{(a +q-b +1)}  D^{a} \left( X_{h^{-1}(r),\star} \otimes M \right)^\top \right\|_F^2. \label{jl-poly-gpk}
		\end{align}
		To obtain the above inequality we used the fact that there is a bijective correspondence between entries of $\left( Y_k^{(a +q-b +1)} \otimes M \right) D^{a} X_{h^{-1}(r),\star}^\top $ and $Y_k^{(a +q-b +1)}  D^{a} \left( X_{h^{-1}(r),\star} \otimes M \right)^\top$.
		Additionally, we use the fact that $M = H \cdot A$ for a JL-transform $H$. So, with probability $1 - \frac{1}{{\rm poly}(n)}$, the following holds for any $a\in[b], r\in[s']$:
		\begin{align*}
			\left\| Y_k^{(a +q-b +1)} D^{a} \left( X_{h^{-1}(r),\star} \otimes M \right)^\top \right\|_F^2 \in \left(1\pm \frac{1}{80q}\right) \left\| Y_k^{(a +q-b +1)} D^{a} \left( X_{h^{-1}(r),\star} \otimes A \right)^\top \right\|_F^2.
		\end{align*}
		By union bounding over $qs'$ events we can conclude that the above inequality holds simultaneously for all $a\in[b], r\in[s']$. From now on we condition on the above inequality holding. By combining this condition with \eqref{jl-poly-gpk} we find that with probability at least $19/20$ the following holds:
		\begin{align}
			\left\| \left( Y_k^{(a+ q-b +1)} \otimes M \right) D^{a} W_{r,k}^{\top} \right\|_F^2  \in \left(1\pm \frac{1}{39q}\right) \left\| Y_k^{(a+ q-b +1)} D^{a} \left( X_{h^{-1}(r),\star} \otimes A \right)^\top \right\|_F^2. \label{eq:deg2tensorsketch-gpk}
		\end{align}
		Using the definition of matrices $Y_k^{(c)} := \bigotimes_{j=c}^q S_k^{(j)} X$ and by Lemma~\ref{lem:SRHT-shared-sign}, because the number of rows of $S^{(c)}_k$ is $m'' = \Omega(q^3 + q^2 \kappa \log n)$, the following holds with probability at least $19/20$ for any $a\in[b], r\in[s'], k \in [m']$,
		\begin{align*}
			\left\| Y_k^{(a+ q-b +1)} D^{a} \left( X_{h^{-1}(r),\star} \otimes A \right)^\top \right\|_F^2 \in \left(1\pm \frac{1}{80q}\right) \left\| X^{\otimes (b - a)} D^{a} \left( X_{h^{-1}(r),\star} \otimes A \right)^\top \right\|_F^2.
		\end{align*}
		By combining the above with \eqref{eq:deg2tensorsketch-gpk} and a union bound, plugging the result into \eqref{tensorsketch-gpk} we find that with probability at least $9/10$ the following holds,
		\begin{align*}
			\textsc{TNorm}^{(k)}\textsc{.Query}\left( L^a_{r,k} , a \right) \in \left(1\pm \frac{1}{10q}\right) \left\| X^{\otimes (b - a)} \cdot D^{a} \left( X_{h^{-1}(r),\star} \otimes A \right)^\top \right\|_F^2.
		\end{align*}
		By taking the median of $m' = \Omega(\log n)$ independent instances of $\textsc{TN}^{(k)}\textsc{.Query}\left( L^a_{r,k} , a \right)$, the success probability of the above gets boosted.  Thus, by a union bound, with probability at least $1 - \frac{1}{{\rm poly}(n)}$ the following holds simultaneously for all $a \in [b]$ and $r \in [s']$,
		\begin{align}
			\textsc{Median}_{k \in [m']} \left\{ \textsc{TNorm}^{(k)}\textsc{.Query}\left( L^a_{r,k} , a \right) \right\} \in \left( 1 \pm \frac{1}{10q} \right) \left\| X^{\otimes (b - a)}  D^{a} \left( X_{h^{-1}(r),\star} \otimes A \right)^\top \right\|_F^2 \label{jl-countsketch-median-gpk}
		\end{align}
		
		Similarly, we can use the fact that there is a bijective correspondence between the entries of $\left( Y_k^{(a + q-b +1)} \otimes M \right) D^{a} X_{i,\star}^\top$ and $Y_k^{(a + q-b +1)} D^{a} {\rm diag} \left(X_{i,\star} \right) M^\top$ along with $M = H \cdot A$ to conclude that with probability $1 - \frac{1}{{\rm poly}(n)}$, the following holds for any $a \in[b], r\in[s'], k \in [m'],i \in [d]$:
		\begin{align}
			\left\| \left( Y_k^{(a + q-b +1)} \otimes M \right) D^{a} X_{i,\star}^\top \right\|_2^2 \in \left(1\pm \frac{1}{80q}\right) \left\| Y_k^{(a + q-b +1)} D^{a} \cdot {\rm diag} \left(X_{i,\star} \right) A \right\|_F^2\label{eq:condition-JL-gkp}
		\end{align}
		By a union bound over $qs'm' d$ events we can conclude that the above inequality holds simultaneously for all $a \in[b], r\in[s'], k \in [m'], i \in [d]$. From now on we condition on the above inequality holding. Then by using the definition of matrices $Y_k^{(c)} := \bigotimes_{j=c}^q S^{(j)}_k X$ and invoking Lemma~\ref{lem:SRHT-shared-sign}, the following holds with probability at least $19/20$ for any $a\in[b], r\in[s'], k \in [m'], i \in [d]$,
		\begin{align*}
			\left\| Y_k^{(a+ q-b +1)} D^{a} {\rm diag} \left(X_{i,\star} \right) A \right\|_F^2 \in \left(1\pm \frac{1}{80q}\right) \left\| X^{\otimes (b - a)} \cdot D^{a} {\rm diag} \left(X_{i,\star} \right)  A \right\|_F^2
		\end{align*}
		By combining this with the condition in \eqref{eq:condition-JL-gkp} and \eqref{Tnorm-xi-gpk} we find that with probability at least $19/20$:
		\begin{align*}
			\textsc{TNorm}^{(k)}\textsc{.Query}\left( D^{a} X_{i,\star}^\top , a+ q-b  \right) \in \left(1\pm \frac{1}{19q}\right) \left\| X^{\otimes (b - a)} D^{a} {\rm diag} \left(X_{i,\star} \right) A \right\|_F^2
		\end{align*}
		By taking the median of $m' = \Omega(\log n)$ independent instances of $\textsc{TNorm}^{(k)}\textsc{.Query}\left( D^{a} X_{i,\star}^\top , a+q-b \right)$, the success probability of the above gets boosted. Thus, by applying the median trick and then using a union bound, with probability at least $1 - \frac{1}{{\rm poly}(n)}$ the following holds simultaneously for all $a \in [b] , i \in [d] $ and $r \in [s']$,
		\begin{align*}
			\textsc{Median}_{k \in [m']} \left\{ \textsc{TNorm}^{(k)}\textsc{.Query}\left( D^{a} X_{i,\star}^\top , a+q-b \right) \right\} \in \left(1\pm \frac{1}{19q}\right) \left\| X^{\otimes (b - a)} D^{a} {\rm diag} \left(X_{i,\star} \right) A \right\|_F^2
		\end{align*}
		Plugging the above inequality along with \eqref{jl-countsketch-median-gpk} into \eqref{eq:def-ia-dist-gpk}, we conclude that with high probability the following bound holds simultaneously for all $a \in [b]$ and all $i \in [d]$,
		\begin{align}
			&\Pr[I_a =i | I_1=i_1, I_2=i_2, \cdots I_{a-1}=i_{a-1}] \ge \left( 1-\frac{1}{3q} \right) \cdot \frac{ \left\| X^{\otimes (b - a)} D^{a} \cdot {\rm diag} \left(X_{i,\star} \right) A \right\|_F^2 }{ \left\| X^{\otimes (b-a+1)} D^{a} A \right\|_F^2}.\label{lower-bnd-Ij-gpk}
		\end{align}
		Thus, using the definition of $D^a$ and $A = (B^\top B + \lambda I)^{-1/2}$, for any $b \in \{ 1,2, \ldots q\}$, we have
		\begin{align*}
			\Pr\left[ I^b =(i_1 , i_2, \cdots i_b) \right]
			&= \prod_{a=1}^{q} \Pr \left[I_a =i_a | I_1=i_1, \cdots I_{a-1}=i_{a-1} \right]\\
			&\ge \prod_{a=1}^{b} \left(1-\frac{1}{3q}\right) \frac{ \left\| X^{\otimes (b - a)} D^{a} \cdot {\rm diag} \left(X_{i_a,\star} \right) A \right\|_F^2 }{ \left\| X^{\otimes (b-a+1)} D^{a} A \right\|_F^2}\\
			& \ge \frac{1}{2} \cdot \frac{ \left\| \mathbf{1}_n^\top \cdot D^{b} \cdot {\rm diag}\left( X_{i_b,\star} \right) A \right\|_F^2}{ \left\| X^{\otimes b} D^{1} A \right\|_F^2}\\ 
			&= \frac{1}{2} \cdot \frac{ \left\| \left[ X^{\otimes b} \cdot {\rm diag}(v)  (B^\top B + \lambda I)^{-1/2} \right]_{(i_1,i_2,\cdots i_b),\star} \right\|_2^2}{ \left\| X^{\otimes b} \cdot {\rm diag}(v) (B^\top B + \lambda I)^{-1/2} \right\|_F^2}
		\end{align*}

		This together with \eqref{eq:dist-fj-gpk}, shows that for any $b \in \{1,2 \ldots q\}$, with high probability in $n$,
		\begin{align*}
			&\Pr\left[ S_{\ell,\star}= \beta_\ell^{-1/2} \big( \underbrace{0,0, \ldots 0}_{\frac{d^b - 1}{d-1} \text{ zeros}}, {e}_{i_1} \otimes {e}_{i_2} \otimes \cdots {e}_{i_b}, \underbrace{0,0, \ldots 0}_{\frac{d^{q+1} - d^{b+1}}{d-1} \text{ zeros}} \big) \right]\nonumber\\
			&\qquad= \Pr \left[ I^b =(i_1, i_2, \ldots i_b) \right]\cdot \Pr[U=b] \\
			&\qquad\ge \frac{1}{3} \cdot \frac{ \left\| \left[ X^{\otimes b} {\rm diag}(v)  (B^\top B + \lambda I)^{-1/2} \right]_{(i_1,i_2,\cdots i_b),\star} \right\|_2^2}{ \left\| X^{\otimes b} {\rm diag}(v)  (B^\top B + \lambda I)^{-1/2} \right\|_F^2} \cdot \frac{\alpha_b^2 \cdot \left\| X^{\otimes b} \cdot {\rm diag}(v) (B^\top B + \lambda I)^{-1/2} \right\|_F^2}{\left\| \Phi \cdot (B^\top B + \lambda I)^{-1/2} \right\|_F^2} \\
			&\qquad= \frac{1}{3} \cdot \frac{ \alpha_b^2 \cdot  \left\| \left[ X^{\otimes b} {\rm diag}(v)  (B^\top B + \lambda I)^{-1/2} \right]_{(i_1,i_2,\cdots i_b),\star} \right\|_2^2}{\left\| \Phi \cdot (B^\top B + \lambda I)^{-1/2} \right\|_F^2}.
		\end{align*}
		The numerator above is exactly equal to the norm of row $(i_1,i_2,\cdots i_b)$ of the $b^{th}$ block of the matrix $ \Phi (B^\top B +\lambda I)^{-1/2}$ (note that $\Phi$ has $q+1$ blocks and its $b^{th}$ block is $\alpha_b \cdot X^{\otimes b}{\rm diag}(v)$). On the other hand if $b=0$, we have,
		\[ \Pr\left[ S_{\ell,\star}= \beta_\ell^{-1/2} \big( 1, \underbrace{0,0, \ldots 0}_{\frac{d^{q+1} - d}{d-1} \text{ zeros}} \big) \right] = \Pr[U=0] \ge \left(1\pm \frac{1}{9q}\right) \frac{\alpha_0^2 \cdot \left\| X^{\otimes 0} \cdot {\rm diag}(v) (B^\top B + \lambda I)^{-1/2} \right\|_F^2}{\left\| \Phi \cdot (B^\top B + \lambda I)^{-1/2} \right\|_F^2}. \]
		The numerator above is exactly equal to the norm of (the sole row of) the $0^{th}$ block of the matrix $ \Phi (B^\top B +\lambda I)^{-1/2}$.
		
		Because $\frac{\beta_\ell}{s}$ is the probability of sampling row $(i_1,i_2,\cdots i_b)$ in the $b^{th}$ block of the matrix $ \Phi (B^\top B +\lambda I)^{-1/2}$ or the sole row of the zero-th block, the above inequalities prove that with high probability, matrix $S$ is a rank-$s$ row norm sampler for $ \Phi (B^\top B +\lambda I)^{-1/2}$ as in Definition \ref{def:row-samp}.

		\paragraph{Runtime:} The first expensive step of this algorithm is the computation of $M$ in line~\ref{M-alg-gpk} which takes $O(m^2 n + q^2 m n \log n)$ operations since $B$ has rank at most $m$. 
		The next expensive computation is the computation of $S^{(c)}_k X$ for $c \in [q]$ and $k \in [m']$ in line~\ref{tnormDS-gpk} of the algorithm. By Lemma~\ref{lem:SRHT-shared-sign}, the total time to compute these sketched matrices is $O\left( (q^4 + q^3 \kappa) n \log^2 n + nd \log^2 n \right)$.
		Another expensive step is the construction of the $\textsc{TNorm}^{(k)}$ data-structure in line~\ref{tnormDS} for $k \in [m']$. By Lemma \ref{lem:TensorNormDS}, these DS's for $\epsilon = \frac{1}{40q}$ and all $k \in [m']$ can be formed in total time $O\left( q^4 \log^2 q \cdot n \log^2 n + (q^4 + q^3\kappa) n \log^3 n \right)$. 
		
		By Lemma~\ref{soda-result}, matrices $W_{r,k}$ for all $r\in[s']$ and $k \in [m']$ in line~\ref{W-gpk} of the algorithm can be computed in total time $O\left( q^2 s'  n\log^2 n + \log n \cdot \text{nnz}\left( X \right) \right)$.
		
		The matrix $W_{r, k}$ for every $k \in [m']$, and $r\in[s']$, has size $O(q^2) \times n$.
		Thus, by Lemma~\ref{lem:TensorNormDS}, computing the distribution $\{ p^a_r \}_{r=1}^{s'}$ in line~\ref{dist-par-gpk} takes time $O\left(q^{4} s' \cdot n \log^2n \log q \right)$ for a fixed $a\in[b]$ and a fixed $\ell \in[s]$.  Therefore, the total time to compute this distribution for all $a$ and $\ell$ is $O\left(q^{8} s^2 \cdot n \log^2 n\log q \right)$. 
		
		The runtime of computing the distribution $\{ q^{a}_i \}_{i\in h^{-1}(t)}$ in line~\ref{dist-qai-gpk} depends on the sparsity of $X_{h^{-1}(t),\star}$, i.e.,  $\text{nnz}\left( X_{h^{-1}(t),\star} \right)$. To bound the sparsity of $X_{h^{-1}(t),\star}$, note that, $\text{nnz}\left( X_{h^{-1}(t),\star} \right) = \sum_{i=1}^d \mathbbm{1}_{\{i \in h^{-1}(t)\}} \cdot \text{nnz}\left( X_{i,\star} \right)$.
		Since the hash function $h$ is fully independent, by invoking Bernstein's inequality, we find that, for every $t\in[s']$ and $a \in [b]$, with high probability in $n$, ${\rm nnz}\left( X_{h^{-1}(t),\star} \right) = O\left( \left( {{\rm nnz}\left( X \right)/s'} + n \right) \log n \right)$. 
		By union bounding over $qs'$ events, with high probability in $n$, ${\rm nnz}\left( X_{h^{-1}(t),\star} \right) = O\left( \left( {{\rm nnz}\left( X \right)/s'} + n \right) \log n \right)$, simultaneously for all $t\in[s']$ and $a \in [b]$.

		Therefore, by Lemma~\ref{lem:TensorNormDS}, the distribution $\{ q^{a}_i \}_{i\in h^{-1}(t)}$ in line~\ref{dist-qai-gpk} of the algorithm can be computed in total time $O\left( q^3 s n \log^4 n \log q + \log^4 n \log q\cdot {\rm nnz}\left( X \right) \right)$ for all $a \in [b]$ and all $\ell \in [s]$.

		The total runtime of Algorithm~\ref{alg:rotatedrowsampler-poly-selftensor} is thus $O( m^2n + q^{8} s^2 n\log^2 n \log q + q^3\kappa n\log^3 n + d n\log^4n )$.
		
	\end{proof}
	
	Now we are ready to prove the main result, i.e., Theorem~\ref{main-thrm-gpk}.
	
	\begin{proofof}{Theorem~\ref{main-thrm-gpk}}
		The theorem follows by invoking Lemmas~\ref{resursive-rlss-lem} and \ref{lem:rotatedrowsampler-GPK}. To find the sampling matrix $\Pi$, run Algorithm~\ref{alg:outerloop} on $\Phi$ with $\mu=s_\lambda$ and for the \textsc{RowSampler} primitive, invoke Algorithm~\ref{alg:rotatedrowsampler-poly-selftensor}. By Lemma~\ref{lem:rotatedrowsampler-GPK}, Algorithm~\ref{alg:rotatedrowsampler-GPK} outputs a row norm sampler as per Definition~\ref{def:row-samp}, with probability $1 - \frac{1}{{\rm poly}(n)}$. Therefore, since the total number of times Algorithm~\ref{alg:rotatedrowsampler-GPK} is invoked by Algorithm~\ref{alg:outerloop} is $\log \frac{\| \Phi \|_F^2}{\epsilon\lambda} = O(\log n)$, by a union bound, the preconditions of Lemma~\ref{resursive-rlss-lem} are satisfied with high probability.
		Thus, it follows that $\Pi$ satisfies the following spectral approximation guarantee
		\[ \frac{\Phi^\top \Phi + \lambda I}{1+\epsilon} \preceq \Phi^\top \Pi^\top \Pi \Phi + \lambda I \preceq \frac{\Phi^\top \Phi + \lambda I}{1-\epsilon}.\]
		
		The only thing that remains is to bound the runtime. In the proof of Lemma~\ref{resursive-rlss-lem} in \cite{woodruff2020near}, it is shown that with high probability at any iteration $t \in [T]$ of Algorithm~\ref{alg:outerloop},
		the following holds,
		\[ \frac{\Phi^\top \Phi + \lambda_t I}{1+\epsilon} \preceq \Phi^\top \Pi_t^\top \Pi_t \Phi + \lambda_t I \preceq \frac{\Phi^\top \Phi + \lambda_t I}{1-\epsilon}.\]
		Therefore, $\left\|  \Phi^\top \Pi_t^\top \Pi_t \Phi \right\| = O(\| \Phi^\top \Phi \|)$. Now note that
		Algorithm~\ref{alg:outerloop} invokes the \textsc{RowSampler} primitive $T = \log \frac{\| \Phi \|_F^2}{\epsilon \lambda} = O(\log n)$ times. Thus, by Lemma~\ref{lem:rotatedrowsampler-GPK}, the runtime of finding $\Pi$ is 
		the sum of $O\left( \frac{q^{8} s_\lambda^2 n \log^4 n }{\epsilon^4} + \sqrt{ \frac{\left\|  \Phi^\top \Pi_t^\top \Pi_t \Phi \right\|}{\lambda_t} } q^3 n \log^3 n + n d \log^4 n \right)$ for all $t \in [T]$.
		Since $\lambda_t = 2^{T - t}\lambda$ has a geometric decay, the total time complexity is $O\left( \frac{q^{8} s_\lambda^2 n \log^5 n}{\epsilon^4} + \sqrt{ \frac{ \left\| \Phi^\top \Phi \right\|}{\lambda} } q^3 n \log^3 n + n d \log^5 n \right)$.
		
	\end{proofof}
	
	\subsection{Application to Gaussian Kernel}\label{appndx-applications-Gauss}
	In this section we show how to use Theorem~\ref{main-thrm-gpk} to spectrally approximate the Gaussian kernel matrix on a dataset with bounded radius. Specifically, we prove Corollary~\ref{main-theorem-Gaussian}:
	
	\begin{proofof}{Corollary~\ref{main-theorem-Gaussian}}
		Our approach is to show that there exists a GPK that tightly approximates the Gaussian kernel matrix and then invoke Theorem~\ref{main-thrm-gpk}.
		We start by letting $X \in \RR^{d \times n}$ be the matrix whose columns are data-points $x_1, \ldots x_n$. Also, let $q = \Theta \left(r + \log \frac{n}{\epsilon\lambda} \right)$ and define $\alpha \in \RR^{q+1}$ as $\alpha_j := 1/\sqrt{j!}$ for every $j=0,1, \ldots q$. Additionally, let $v \in \RR^n$ be defined as $v_i := e^{-\|x_i\|_2^2/2}$ for $i \in [n]$. 
		Now we define the GPK kernel matrix $\widetilde{K} \in \RR^{n \times n}$ corresponding to the above mentioned $q$, $X$, $\alpha$, and $v$, i.e., $\widetilde{K} := {\rm diag}(v) \left( \sum_{j =0}^q \alpha_j^2 \cdot X^{\otimes j \top} X^{\otimes j} \right) {\rm diag}(v)$. Also let $\widetilde{\Phi}$ be the feature matrix corresponding to $\widetilde{K}$ defined as per \eqref{feature-matrix-GPK}. Then by invoking Theorem~\ref{main-thrm-gpk} we can find a sampling matrix $\Pi$ in time $O\left( \frac{q^{8} s_\lambda^2 n \log^5 n}{\epsilon^4} + \sqrt{ \frac{\| K \|}{\lambda} } q^3 n\log^3n + nd \log^5 n \right) = \widetilde{O} \left( \frac{r^{8} s_\lambda^2 n}{\epsilon^4} + \sqrt{ \frac{\| K \|}{\lambda} } r^3 n + nd \right)$ such that with high probability in $n$,
		\[ \frac{\widetilde{K} + \lambda I}{1+\epsilon/3} \preceq \widetilde{\Phi}^\top \Pi^\top \Pi \widetilde{\Phi} + \lambda I \preceq \frac{\widetilde{K} + \lambda I}{1-\epsilon/3}.\]
		Now all that is left to do is to show that 
		\[ \frac{K + \lambda I}{1+\epsilon/3} \preceq \widetilde{K} + \lambda I \preceq \frac{K + \lambda I}{1-\epsilon/3}.\]
		To prove the above we note that since $K$ and $\widetilde{K}$ are PSD matrices, it suffices to prove $\left\| \widetilde{K} - K \right\| \le \frac{\epsilon \lambda}{4}$. The reason we have this bound is,
		\begin{align*}
			\left\| \widetilde{K} - K \right\|^2 &\le \left\| \widetilde{K} - K \right\|_F^2\\
			&= \sum_{i,j \in [n]} \left| \widetilde{K}_{i,j} - K_{i,j} \right|^2\\
			&= \sum_{i,j \in [n]} \left| \sum_{\ell =0}^q \langle x_i , x_j \rangle^\ell / \ell!  - e^{\langle x_i , x_j \rangle} \right|^2 \cdot e^{-\|x_i\|_2^2} \cdot e^{-\|x_j\|_2^2}\\
			&\le \sum_{i,j \in [n]} \left| \sum_{\ell =q+1}^\infty \langle x_i , x_j \rangle^\ell / \ell! \right|^2 \\
			&\le \sum_{i,j \in [n]} \left| \sum_{\ell =q+1}^\infty r^\ell / \ell! \right|^2 \\
			&\le \sum_{i,j \in [n]} \left| \frac{\epsilon \lambda}{4n} \right|^2 \\
			&= \frac{\epsilon^2 \lambda^2}{16}.
		\end{align*}
		This completes the proof and shows that,
		\[ \frac{{K} + \lambda I}{1+\epsilon} \preceq \widetilde{\Phi}^\top \Pi^\top \Pi \widetilde{\Phi} + \lambda I \preceq \frac{{K} + \lambda I}{1-\epsilon}.\]
		
	\end{proofof}

	\subsection{Application to Neural Tangent Kernel}\label{appndx-applications-ntk}
	In this section we show how to use Theorem~\ref{main-thrm-gpk} to spectrally approximate the kernel matrix corresponding to the NTK defined in \eqref{eq:def-ntk} on a dataset with bounded radius. Specifically, we prove Corollary~\ref{main-theorem-ntk}:
	
	\begin{proofof}{Corollary~\ref{main-theorem-ntk}}
		Our approach is to show that there exists a GPK that tightly approximates the NTK and then invoke Theorem~\ref{main-thrm-gpk}.
		We start by letting $X \in \RR^{d \times n}$ be the matrix whose columns are normalized data points $\frac{x_1}{\| x_1\|_2}, \ldots \frac{x_n}{\| x_n\|_2}$.
		Also let $v \in \RR^n$ be defined as the vector of norms $v_i := \| x_i \|_2$ for $i \in [n]$. Additionally, let $q = \Theta\left( \frac{n^2r^2}{\epsilon^2 \lambda^2} \right)$ and define the vector of coefficients $\alpha \in \RR^{2q+3}$ as follows for every $j=0,1, \ldots 2q+2$:
		\[ \alpha_j :=\begin{cases}
			\frac{1}{\pi} & \text{ if }j=0\\
			1 & \text{ if }j=1\\
			0 & \text{ if }j>1\text{ is odd}\\
			\frac{1}{\pi} \cdot \frac{(j+1) \cdot (j-2)!}{2^{j-2} ((j/2-1)!)^2 \cdot (j-1)\cdot j} & \text{ if }j>1\text{ is even}
		\end{cases}. \]
		
		Now we define the GPK kernel matrix $\widetilde{K} \in \RR^{n \times n}$ corresponding to the abovementioned $q$, $X$, $\alpha$, and $v$, i.e., $\widetilde{K} := {\rm diag}(v) \left( \sum_{j =0}^q \alpha_j^2 \cdot X^{\otimes j \top} X^{\otimes j} \right) {\rm diag}(v)$. Also let $\widetilde{\Phi}$ be the feature matrix corresponding to $\widetilde{K}$ defined as per \eqref{feature-matrix-GPK}. Then by invoking Theorem~\ref{main-thrm-gpk} and also noting that the definition of NTK in \eqref{eq:def-ntk} implies $\| K \| \le {\rm tr}(K) = 2n$, we can find a sampling matrix $\Pi$ in time $O\left( \frac{q^{8} s_\lambda^2 n \log^5 n}{\epsilon^4} + \sqrt{ \frac{\| K \|}{\lambda} } q^3 n\log^3n + nd \log^5 n \right) = \widetilde{O} \left(\left( \frac{nr}{\epsilon\lambda} \right)^{16} \frac{s_\lambda^2 n}{\epsilon^4} + nd \right)$ such that with high probability in $n$,
		\[ \frac{\widetilde{K} + \lambda I}{1+\epsilon/3} \preceq \widetilde{\Phi}^\top \Pi^\top \Pi \widetilde{\Phi} + \lambda I \preceq \frac{\widetilde{K} + \lambda I}{1-\epsilon/3}.\]
		Now all that is left to do is to show that 
		\[ \frac{K + \lambda I}{1+\epsilon/3} \preceq \widetilde{K} + \lambda I \preceq \frac{K + \lambda I}{1-\epsilon/3}.\]
		To prove the above we note that since $K$ and $\widetilde{K}$ are PSD matrices, it suffices to prove $\left\| \widetilde{K} - K \right\| \le \frac{\epsilon \lambda}{4}$. To prove this bound note that the Taylor series expansion of function $k_{{\tt ntk}}(\beta)$ defined in \eqref{eq:def-ntk} is the following,
		\[ k_{{\tt ntk}}(\beta) \equiv \frac{1}{\pi} + \beta + \frac{1}{\pi} \sum_{\ell=0}^\infty \frac{(2\ell+3) \cdot (2\ell)!}{2^{2\ell} (\ell!)^2 \cdot (2\ell +1)(2\ell+2)} \cdot \beta^{2\ell + 2}. \]
		Therefore, we can write
		\small
		\begin{align*}
			\left\| \widetilde{K} - K \right\|^2 &\le \left\| \widetilde{K} - K \right\|_F^2\\
			&= \sum_{i,j \in [n]} \left| \widetilde{K}_{i,j} - K_{i,j} \right|^2\\
			&= \sum_{i,j \in [n]} \left| \frac{1}{\pi} + \frac{\langle x_i , x_j \rangle}{\|x_i\|\|x_j\|} + \frac{1}{\pi} \sum_{\ell=0}^{q} \frac{(2\ell+3) \cdot (2\ell)!}{2^{2\ell} (\ell!)^2  (2\ell +1)(2\ell+2)}  \left( \frac{\langle x_i , x_j \rangle}{\|x_i\|\|x_j\|} \right)^{2\ell + 2} -  k_{{\tt ntk}}\left( \frac{\langle x_i , x_j \rangle}{\|x_i\|\|x_j\|} \right) \right|^2 \cdot \|x_i\|_2^2\|x_j\|_2^2 \\
			&= \sum_{i,j \in [n]} \left| \frac{1}{\pi} \sum_{\ell=q+1}^{\infty} \frac{(2\ell+3) \cdot (2\ell)!}{2^{2\ell} (\ell!)^2  (2\ell +1)(2\ell+2)}  \left( \frac{\langle x_i , x_j \rangle}{\|x_i\|\|x_j\|} \right)^{2\ell + 2} \right|^2 \cdot \|x_i\|_2^2\|x_j\|_2^2 \\
			&\le \sum_{i,j \in [n]} \left| \frac{1}{\pi} \sum_{\ell=q+1}^{\infty} \frac{(2\ell+3) \cdot (2\ell)!}{2^{2\ell} (\ell!)^2  (2\ell +1)(2\ell+2)} \right|^2 \cdot r^2 \\
			&= \frac{n^2 r^2}{\pi^2} \cdot \left| \sum_{\ell=q+1}^{\infty} \frac{(2\ell+3) \cdot (2\ell)!}{2^{2\ell} (\ell!)^2  (2\ell +1)(2\ell+2)} \right|^2 \\
			&\le \frac{n^2 r^2}{\pi^2} \cdot \left| \sum_{\ell=q+1}^{\infty} \frac{1}{2\ell^{3/2}} \right|^2 \\
			&\le \frac{n^2 r^2}{4\pi^2 q} \le \frac{\epsilon^2\lambda^2}{16}.
		\end{align*}
		\normalsize
		This completes the proof and shows that,
		\[ \frac{{K} + \lambda I}{1+\epsilon} \preceq \widetilde{\Phi}^\top \Pi^\top \Pi \widetilde{\Phi} + \lambda I \preceq \frac{{K} + \lambda I}{1-\epsilon}.\]

	\end{proofof}


\end{document}